\documentclass{article}

\PassOptionsToPackage{numbers, compress}{natbib}
\usepackage[preprint]{neurips_2026}
\usepackage[textsize=tiny]{todonotes}
\usepackage[utf8]{inputenc} 
\usepackage[T1]{fontenc}    
\usepackage{hyperref}       
\usepackage{url}            
\usepackage{booktabs}       
\usepackage{amsfonts}       
\usepackage{nicefrac}       
\usepackage{microtype}      
\usepackage{xcolor}         
\usepackage{hyperref}
\usepackage{amsmath}
\usepackage{amssymb}
\usepackage{mathtools}
\usepackage{amsthm}
\usepackage{bm}
\usepackage{graphicx}
\usepackage{thmtools, thm-restate}
\usepackage{multirow}
\usepackage[normalem]{ulem}
\usepackage[capitalize,noabbrev]{cleveref}

\theoremstyle{plain}
\newtheorem{theorem}{Theorem}[section]
\newtheorem{proposition}[theorem]{Proposition}
\newtheorem{lemma}[theorem]{Lemma}

\theoremstyle{definition}
\newtheorem{definition}[theorem]{Definition}
\newtheorem{assumption}[theorem]{Assumption}
\theoremstyle{remark}
\newtheorem{remark}[theorem]{Remark}

\title{On the Superlinear Relationship between SGD Noise Covariance and Loss Landscape Curvature}

\author{%
  Yikuan Zhang\thanks{Equal contribution.} \\
  School of Physics\\
  Peking University\\
  Beijing, China \\
  \texttt{yk.zhang@stu.pku.edu.cn} \\
  \And
  Ning Yang\footnotemark[1] \\
  Peking University Chengdu Academy for Advanced Interdisciplinary Biotechnologies \\
  Chengdu, China \\
  \texttt{yn\_biophy@pku.edu.cn} \\
  \AND
  Yuhai Tu \\
  Flatiron Institute \\
  New York, USA \\
  \texttt{ytu@flatironinstitute.org} \\
}

\begin{document}

\maketitle

\begin{abstract}
Stochastic Gradient Descent (SGD) introduces anisotropic noise that is correlated with the local curvature of the loss landscape, thereby biasing optimization toward flat minima.
Prior work often assumes an equivalence between the Fisher Information Matrix and the Hessian for negative log-likelihood losses, leading to the claim that the SGD noise covariance $\mathbf{C}$ is proportional to the Hessian $\mathbf{H}$.
We show that this assumption holds only under restrictive conditions that are typically violated in deep neural networks.
Using the recently discovered Activity--Weight Duality, we find a more general relationship agnostic to the specific loss formulation, showing that $\mathbf{C} \propto \mathbb{E}_p[\mathbf{h}_p^2]$, where $\mathbf{h}_p$ denotes the per-sample Hessian with $\mathbf{H} = \mathbb{E}_p[\mathbf{h}_p]$.
As a consequence, $\mathbf{C}$ and $\mathbf{H}$ commute approximately rather than coincide exactly. 
 We further find that, within the analyzed fully connected layers, their diagonal elements follow per-layer empirical power laws $C_{ii} \propto H_{ii}^{\gamma}$, with layer-dependent fitted exponents bounded by $1 \leq \gamma \leq 2$.
Experiments across datasets, architectures, and loss functions support the resulting layerwise bounds, providing a unified characterization of the noise-curvature relationship in deep learning.
\end{abstract}

\section{Introduction}
\label{sec:intro}

The remarkable generalizability of Artificial Neural Networks (ANNs) remains one of the central mysteries in deep learning theory \cite{Yann2015}. The training process, typically driven by Stochastic Gradient Descent (SGD), introduces noise that is not merely an artifact but a crucial feature \cite{bottou91c}. It is widely believed that this stochasticity facilitates the escape from sharp minima toward flatter, more generalizable regions \cite{hardttrainfaster2016,keskarLargebatchTrainingDeep2017,wuUnderstandingGeneralizationDeep2017,fengInverseVarianceFlatness2021,yangStochasticGradientDescent2023}. Crucially, the effectiveness of this process relies on the noise's anisotropy, which stems from alignment between the noise covariance matrix $\mathbf{C}$ and the loss Hessian $\mathbf{H}$ \cite{sagunEigenvaluesHessianDeep2017,sagunEmpiricalAnalysisHessian2018,xingWalkSGD2018,nguyenFirstExitTime2019,chaudhariEntropySGDBiasingGradient2019}. Unlike isotropic noise, which is often ineffective in highly degenerate landscapes, this specific geometric coupling acts as the decisive mechanism for driving this geometry-aware exploration \cite{Neelakantan2015AddingGN,zhouTowar2019,zhuAnisotropicNoiseStochastic2019,xieArtificialNeuralVariability2021,haochenShapeMattersUnderstanding2021,yang2026transientlearningdynamicsdrive}. A broader review of relevant studies can be found in Appendix\ref{app:RelatedWorks}.

Previous studies~\cite{jastrzebskiThreeFactorsInfluencing2018,zhuAnisotropicNoiseStochastic2019,xieDiffusionTheoryDeep2021} have argued that $\mathbf{C}$ is proportional to $\mathbf{H}$ near convergence, based on the assumption that the Hessian coincides with the Fisher Information Matrix for cross-entropy (CE). This equivalence, however, requires strong conditions that are frequently violated in modern deep networks~\cite{li2019hessianbasedanalysissgd}. As a result, the Fisher approximation fails to capture the empirically observed anisotropic alignment: it can break down even for cross-entropy loss and is fundamentally inapplicable to mean squared error (MSE). These limitations highlight the need for a more robust framework to accurately characterize the $\mathbf{C}$–$\mathbf{H}$ relationship.

In this work, we introduce a new theoretical framework for analyzing the $\mathbf{C}$–$\mathbf{H}$ relationship, grounded in the recently discovered Activity–Weight Duality (AWD)~\cite{fengActivityWeightDuality2023}. AWD maps minibatch-induced activity fluctuations— which underlie $\mathbf{C}$—onto equivalent weight perturbations in the per-sample loss $\ell$ that determine $\mathbf{H}$. This duality enables us to establish the general $\mathbf{C}$–$\mathbf{H}$ relationship regardless of the specific forms of the loss function.
 
Our main contributions can be summarized as follows:

\begin{itemize}

\item We develop a new method to compute the SGD noise covariance $\mathbf{C}$ based on the Activity–Weight Duality (AWD), which is independent of the specific loss construction, showing that $\mathbf{C}$ is determined by the second moment of the per-sample Hessian,
$\mathbf{C} \propto \mathbb{E}_{p}\!\left[\mathbf{h}_p^2\right]$. Comparing with the definition of the Hessian definition, $\mathbf{H} = \mathbb{E}_{p} [\mathbf{h}_p]$, we claim a superlinear relationship between $\mathbf{C}$ and $\mathbf{H}$. 
\item We empirically show that, within each parameter space of  single fully connected layer, the diagonal entries of $\mathbf{C}$ follow a power-law fit with those of $\mathbf{H}$, i.e., $C_{ii} \propto H_{ii}^\gamma$. In these layerwise fits, MSE is approximately linear whereas CE is superlinear, challenging the classical Fisher approximation. Building on the AWD framework, we further prove that,  whenever such a per-layer empirical power-law fit exists, the fitted exponent is bounded by $1 \le \gamma \le 2$.



\item We uncover the geometric origins of the distinction in $\gamma$ between CE ($\gamma>1$) and MSE ($\gamma \approx 1$) via a ``suppression experiment".  While both losses exhibit per-sample Hessians dominated by a single leading eigenvalue, CE displays a strong positive correlation between the magnitude of this eigenvalue and its alignment with global Hessian directions, leading to $\gamma > 1$, whereas these quantities are effectively independent for MSE, yielding $\gamma \to 1$. This eigenvalue–eigendirection entanglement provides a geometric explanation for the observed superlinear scaling in CE.


\end{itemize}

\section{Background and Motivation}

\subsection{Limitation of the Fisher Approximation}

SGD updates model parameters using stochastic gradients estimated from randomly sampled mini-batches \cite{robbinsStochasticApproximationMethod1951,bottou91c}. Let $\mathcal{D} = \{(\bm{x}_i, y_i)\}_{i=1}^N$ denote the training set of size $N$. At iteration $t$, a mini-batch $\mathcal{B}_{\mu_t}$ of size $B$ is sampled uniformly from $\mathcal{D}$. The parameters $\bm{w}_t$ are updated with a learning rate $\eta$ as:
\begin{equation}
    \bm{w}_{t+1} = \bm{w}_{t} - \eta \nabla \mathcal{L}^{\mu_t}(\bm{w}_t),
\end{equation}
where $\mathcal{L}^{\mu_t}(\bm{w}_t)$ represents the mini-batch loss. Conceptually, the SGD noise arises from the deviation of the local mini-batch gradient from the full-batch gradient $\bm{g}(\bm{w}_t) \triangleq \nabla \mathcal{L}(\bm{w}_t)$, where $\mathcal{L}(\bm{w}_t)$ denotes the empirical risk over the entire dataset. 

Following previous work \cite{MandtStochatic2017,smithBayesianPerspectiveGeneralization2018,chaudhariStochasticGradientDescent2018,xieDiffusionTheoryDeep2021}, the SGD covariance is given by:
\begin{equation} \label{eq:cov}
    \mathbf{C}(\bm{w}) = \frac{1}{B} \bigg[ \frac{1}{N} \sum_{i=1}^N \nabla \ell_i(\bm{w}) \nabla \ell_i(\bm{w})^\top  - \bm{g}(\bm{w}) \bm{g}(\bm{w})^\top \bigg] 
    \approx \frac{1}{BN} \sum_{i=1}^N \nabla \ell_i (\bm{w}) \nabla \ell_i (\bm{w})^{\top},
\end{equation}
where $\ell_i(\bm{w}) \triangleq \ell(\bm{x}_i,y_i;\bm{w})$ denotes the per-sample loss. Eq.~\eqref{eq:cov} is valid near a generic minimum, where the magnitude of the average gradient is negligible, i.e., $\|\bm{g}(\bm{w})\| \approx 0$.

In the Fisher-based approach, a negative log-likelihood (cross-entropy) loss is used, i.e., $\ell(\bm{x}, y; \bm{w}) = -\log p_{\bm{w}}(y|\bm{x})$. From Eq.~\eqref{eq:cov}, the noise covariance is approximated by the empirical Fisher Information Matrix $\mathbf{C} \approx \frac{1}{B} \mathbf{F}$. Differentiating the overall cross-entropy loss function $\mathcal{L}(\bm{w})$, it is easy to show that the Hessian $\mathbf{H}(\bm{w}) = \nabla^2 \mathcal{L}(\bm{w})=\mathbf{F} - \mathbf{R}$ is  different from $\mathbf{F}$ by a residual term $\mathbf{R}$, which can be written as:
\begin{equation}\label{eq:vanishing_hessian}
    \mathbf{R} =  \int p(\bm{x}) \left[ \int p_{\text{data}}(y|\bm{x}) \frac{1}{p_{\bm{w}^*}(y|\bm{x})} \nabla^2 p_{\bm{w}^*}(y|\bm{x}) \, \mathrm{d}y \right] \mathrm{d}\bm{x}, 
\end{equation}
where $p(\bm{x})$ is the input distribution; $p_{\text{data}}(y|\bm{x})$ and $p_{\bm{w}^*}(y|\bm{x})$ are the probabilities of label $y$ from data and from model with parameters $\bm{w}^*$, respectively.

The key approximation made in the Fisher-based approach is the ``realizability" condition, namely $p_{\bm w^*}(y|\bm x)=p_{\mathrm{data}}(y|\bm x)$, under which it is easy to show that $R=\int p(\bm{x}) \nabla^2 \left[ \int p_{\bm{w}^*}(y|\bm{x}) \, \mathrm{d}y \right] \mathrm{d}\bm{x}=0$ yielding $\mathbf H \approx\mathbf F$, and thus the linear relationship between $\mathbf{C}$ and $\mathbf{H}$~\cite{jastrzebskiThreeFactorsInfluencing2018,zhuAnisotropicNoiseStochastic2019,xieDiffusionTheoryDeep2021}.


However, labels of training data are fixed, i.e., $p_{\text{data}}(y|\bm{x})=\delta(y-y^*(\bm{x}))$. The realizability condition would demand that $ p_{\bm w^*}(y^*(\bm x)\mid \bm x)=1$, where $y^*(\bm{x})$ is the ground-truth label. For a softmax classifier trained with cross-entropy, this requires zero gradient, i.e., $\nabla_{\bm w}\ell_i(\bm w^*)=\bm 0$ for all samples $i$, which can only be achieved in the limit $\|\bm w\|\to\infty$. This means that the Fisher approximation is invalid for any finite-weight solution obtained in practice.

Does the failure of the Fisher approximation imply that $\mathbf{C}$ and $\mathbf{H}$ are unrelated? The answer is no. In the rest of our paper, we will show a robust positive but nonlinear correlation between $\mathbf{C}$ and $\mathbf{H}$ that is agnostic to the specific loss function used. 

\vskip -0.2in
\subsection{Commutativity between $\mathbf{C}$ and $\mathbf{H}$ }
\label{sec:commute}

Beyond the failure of Fisher Approximation, the violation of a direct linear relation between $\mathbf{C}$ and $\mathbf{H}$ can be seen by dimensional analysis. The loss function $\mathcal{L}$ carries an intrinsic physical dimension (e.g., squared error units or bits). Consequently, the elements of the Hessian $\mathbf{H}$ possess dimensions of $[\mathcal{L}]/[\bm{w}]^2$, whereas the covariance matrix $\mathbf{C}$ scales as $[\mathcal{L}]^2/[\bm{w}]^2$.
Thus, the hypothesized proportionality $\mathbf{C} \propto \mathbf{H}$ is dimensionally inconsistent. Any strict linear relationship would require a proportionality constant with dimensions of $[\mathcal{L}]$, implying that the relationship depends on the arbitrary scale of the loss function.

{Given the dimensional mismatch that rules out strict linear dependence, we investigate a weaker condition: do $\mathbf{C}$ and $\mathbf{H}$ share the same eigensystem? This is equivalent to checking matrix commutativity. We analyze the structure of $\mathbf{C}$ in the eigenbasis of $\mathbf{H}$. As visualized in Figure~\ref{fig:CvsH}(a), $\mathbf{C}$ exhibits a characteristic ``arrowhead'' structure: the diagonal is dominant, but non-negligible off-diagonal elements persist in the top rows and columns due to the head-heavy nature of the spectrum.}

This pattern suggests approximate, but not exact, simultaneous diagonalization of $\mathbf C$ and $\mathbf H$. To quantify this property, we first use the standard commutator error:
\begin{equation}
    \epsilon_F=\frac{\|[\mathbf C,\mathbf H]\|_F}{\|\mathbf C\|_F\|\mathbf H\|_F},
\end{equation}
where the subscript $F$ denotes the Frobenius norm.
However, for such an ``arrowhead'' structure, it must be interpreted with care because the commutator weights the observed off-diagonal covariance by Hessian eigengaps. If $\mathbf H=V_H\Lambda_HV_H^\top$ and $\widetilde{\mathbf C}=V_H^\top\mathbf C V_H$, then
$[\widetilde{\mathbf C},\Lambda_H]_{ij}=(\lambda_j-\lambda_i)\widetilde C_{ij}.$
Thus the head-row and head-column leakage in the arrowhead can contribute disproportionately and may make the error commutator even exceed a baseline (random rotated $\mathbf{C}$ while preserving the spectrum) as shown in \Cref{fig:metrics_cifar_cnn_cse,fig:metrics_mnist_fc_mse,fig:metrics_mnist_fc_cse}~(b). We therefore also inspect $\widetilde{\mathbf C}$ directly through the entrywise ratio
$\rho_{\rm off/diag}={\frac{1}{d(d-1)}\sum_{i\ne j}|\widetilde C_{ij}|}/{\frac{1}{d}\sum_i|\widetilde C_{ii}|}.$
This compares a typical off-diagonal entry with the typical diagonal scale, without Hessian eigengap weighting. \Cref{fig:metrics_cifar_cnn_cse,fig:metrics_mnist_fc_mse,fig:metrics_mnist_fc_cse}~(c) shows that $\rho_{\rm off/diag}$ is significantly smaller than the random baseline (random rotated $\mathbf{C}$ while preserving the spectrum). These results are consistent with our analytical results in Theorem~\ref{thm:spectral_noise} that:
$|\widetilde C_{ij}|=o(\widetilde C_{ii}+\widetilde C_{jj})$ for $i\ne j$.

\begin{figure*}[ht!]
\centering
{\includegraphics[width=1\columnwidth]{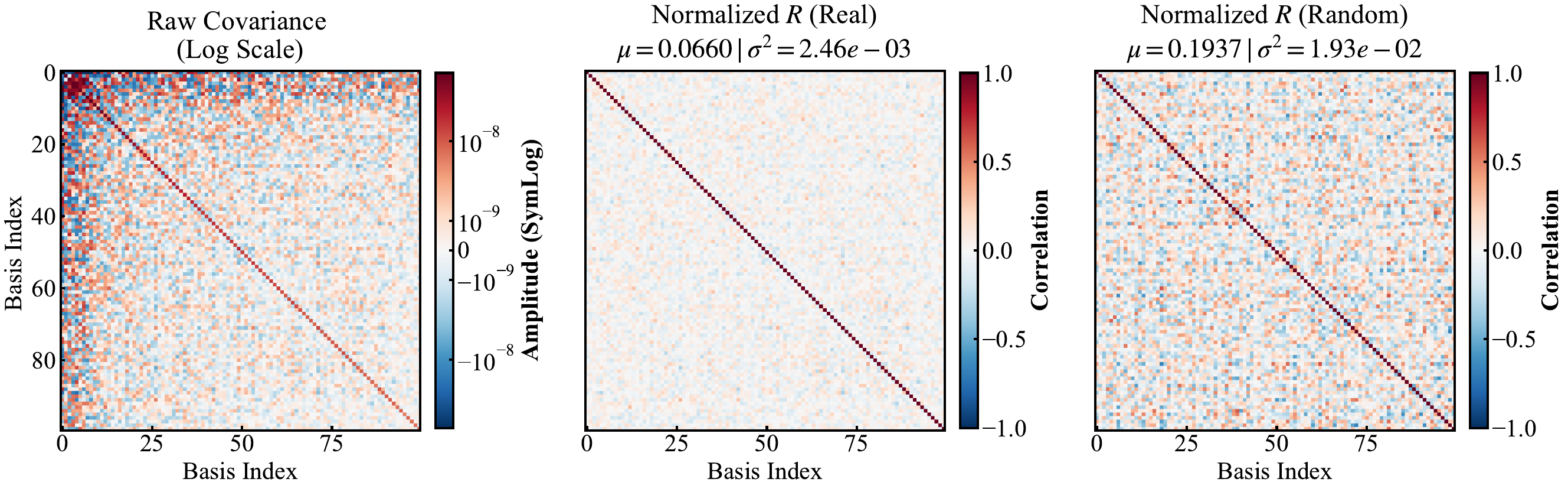}}
\caption{Noise–curvature alignment for a CNN trained on CIFAR-10 with cross-entropy loss (top 100 eigen-directions).
\textbf{(a)} The empirical covariance matrix $\mathbf{C}$ represented in the Hessian eigenbasis.
\textbf{(b)} The scale-invariant correlation matrix $\mathbf{R}_{\text{real}}$, normalized by diagonal elements.
\textbf{(c)} The randomized Baseline $\mathbf{R}_{\text{rand}}$, constructed by randomly rotating $\mathbf{C}$ while preserving its eigenvalue spectrum. See Appendix \Cref{fig:cmt_cifar_cnn_cse,fig:cmt_mnist_fc_cse,fig:cmt_mnist_fc_mse} for more details and results on additional architectures.}
\label{fig:CvsH}
\end{figure*}

To better separate geometric alignment from the magnitude effects, we use a new scale-invariant correlation matrix $\mathbf{R}$, defined as:
$R_{ij}={\widetilde C_{ij}}/{\sqrt{|\widetilde C_{ii}\widetilde C_{jj}|}}$, where $\mu= \frac{1}{d(d-1)}\sum_{i\ne j} \left[|R_{ij}|\right]$.
Here, the diagonal elements $R_{ii}$ are unity, while off-diagonal terms $R_{ij}$ quantify the relative coupling strength between eigen-directions $i$ and $j$.
{For a rigorously benchmark, we construct arandomized baseline} $\mathbf{R}_{\text{rand}}$, generated by rotating $\mathbf{C}$ with a random orthogonal matrix $\mathbf{Q}$ before projecting it onto the Hessian basis ($\mathbf{C}_{\text{rand}} = \mathbf{Q} \mathbf{C} \mathbf{Q}^\top$). This procedure preserves the spectrum of $\mathbf{C}$ while destroying any specific geometric alignment with $\mathbf{H}$.

The results reveal a strong structural alignment between the noise covariance and the Hessian. Figure~\ref{fig:CvsH}(b,c) shows that $\mathbf R_{\rm real}$ is far more diagonal than $\mathbf R_{\rm rand}$: $\mu_{\rm rand}\approx0.8/\sqrt{M}\approx0.2$ (for $M \approx 20$ dominant eigenvalues indicated by the effective ``arrowhead'' bandwidth (Figure~\ref{fig:CvsH} (a)), more details in Appendix~\ref{app:random_baseline}), whereas $\mu_{\rm real}\approx0.066$, which provides statistical evidence that the noise covariance matrix $\mathbf{C}$ is highly aligned with the Hessian.

\section{C-H Relation via Activity-Weight Duality}
\label{sec:awd}

To overcome the inherent limitations of the Fisher approximation, we utilize the Activity-Weight Duality (AWD) framework proposed by \cite{fengActivityWeightDuality2023}. AWD establishes an exact correspondence between changes in the activity space due to variations in data and equivalent perturbations in the weight space, independent of the specific loss formulation, bypassing negative-log-likelihood and realizability restrictions while treating CE and MSE uniformly.

\subsection{Matched Sample Pairs and Activity Perturbations}

For a single Fully Connected Layer (FCL) parameterized by weights $\mathbf{W} \in \mathbb{R}^{d_\text{out} \times d_\text{in}}$. Let $\bm{a} \in \mathbb{R}^{d_\text{in}}$ be the input activity.
To isolate the SGD noise arising from intra-class variability, we instead formulate the matched-pair construction as a class-wise discrete optimal-transport problem between the two sequential mini-batches. For each class $c$, let $\mathcal I_\mu^c=\{p\in\mathcal B_\mu:y_p=c\}$ and $\mathcal I_\nu^c=\{q\in\mathcal B_\nu:y_q=c\}$. Under the balanced-batch setting, these sets have equal cardinality, and the matching is the minimum-cost bijection
\begin{equation}
\pi_c^*
\in \mathop{\arg\min}_{\pi\in\Pi_c}
\sum_{p\in\mathcal I_\mu^c}
\left\|\bm a_{\pi(p)}-\bm a_p\right\|_2^2,
\Pi_c
=
\left\{
\pi:\mathcal I_\mu^c\to\mathcal I_\nu^c
\;\bigm|\;
\pi \text{ is a bijection}
\right\}.
\label{eq:classwise_ot_matching}
\end{equation}
This is the uniform-mass optimal-transport, equivalently linear-assignment, problem restricted within each class. We then set $p'=\pi_{y_p}^*(p)$ and define $\Delta \bm a_p=\bm a_{p'}-\bm a_p$. When the mini-batch size $B$ is large, the same-class candidate pool becomes dense enough that the assignment can select pairs with small activity displacement; hence $\Delta \bm a_p$ can be regarded as a small activity perturbation in the subsequent AWD expansion.

\subsection{Minimal Activity-Weight Duality in FCL}

\begin{definition}[The Minimal Activity-Weight Duality \cite{fengActivityWeightDuality2023}]
\label{def:awd}
Given an input activity perturbation $\Delta \bm{a}$ derived from a matched sample pair, the {Minimal Activity-Weight Duality} defines the unique weight perturbation $\Delta \mathbf{W}^*$ that satisfies the condition of invariant pre-activation $\mathbf{W}(\bm{a} + \Delta \bm{a}) = (\mathbf{W} + \Delta \mathbf{W})\bm{a}$, while minimizing the Frobenius norm of the weight change:
\begin{equation}\label{AWD-minimal}
    \Delta \mathbf{W}^* = \mathop{\arg\min}_{\Delta \mathbf{W}} \|\Delta \mathbf{W}\|_F^2 \quad 
    \text{s.t.} \quad \Delta \mathbf{W} \bm{a} = \mathbf{W} \Delta \bm{a}.
\end{equation}
\end{definition}
\vskip -0.1in
As a consequence, the per-sample loss remains invariant, i.e., $\ell(\bm{x}_{p'};\mathbf{W}) = \ell(\bm{x}_{p};\mathbf{W} + \Delta \mathbf{W})$. The constraint implies that the shift in the weight space must reproduce the exact shift in the pre-activation space caused by the activity noise. 

\begin{lemma}[Explicit Solution for AWD \cite{fengActivityWeightDuality2023}]
\label{lem:awd_solution}
The closed-form solution to the optimization problem in Definition \ref{def:awd} is a rank-1 outer product given by:
$\Delta \mathbf{W}^* = {(\mathbf{W} \Delta \bm{a}) \bm{a}^\top}/{\|\bm{a}\|^2}.$
\end{lemma}

We provide the proof in {Appendix \ref{proof:AWD}}. Intuitively, this solution demonstrates that the equivalent weight noise $\Delta \mathbf{W}^*$ is aligned with the outer product of the propagated error $\mathbf{W} \Delta \bm{a}$ and the input activation $\bm{a}$. Note that for the subsequent gradient analysis, we denote the vectorized weights as $\bm{w} = \operatorname{vec}(\mathbf{W}) \in \mathbb{R}^{D}$.

\subsection{AWD-Based Gradient Approximation}
\label{sec:grad_approx}

Let $\mathcal{L}^\mu$ and $\mathcal{L}^\nu$ denote the empirical losses on two sequential mini-batches $\mu$ and $\nu$, and define the mini-batch gradient difference as
$\bm g_{\mu\nu}\triangleq\nabla \mathcal{L}^\nu(\bm w)-\nabla \mathcal{L}^\mu(\bm w).$
Equivalently to Eq.~\eqref{eq:cov}, the same SGD noise covariance can be written in the two-batch difference form $\mathbf C=\frac12 \mathbb E_{\mu,\nu}\!\left[\bm g_{\mu\nu}\bm g_{\mu\nu}^\top\right]$.


To connect this two-batch covariance form to local curvature, we now map each matched mini-batch sample difference to an equivalent weight-space perturbation, which allows the subsequent lemma to decompose $\bm g_{\mu\nu}$ into a curvature-driven leading term and controlled remainders.
For each matched pair $(p,p')$ with $p\in\mathcal B_\mu$ and $p'\in\mathcal B_\nu$, let
$\Delta \bm a_p=\bm a_{p'}-\bm a_p$ be the corresponding activity perturbation. The minimal AWD perturbation is
$\Delta \bm w_p^{\mu\nu}=\operatorname{vec} \left(\mathbf W_p^{\mu\nu}\right)=\mathbf J_{\Delta,p}^{\mu\nu}\bm w,$
where $\mathbf J_{\Delta,p}^{\mu\nu}\triangleq\nabla_{\bm w}\Delta \bm w_p^{\mu\nu}={\bm a_p \Delta \bm a_p^\top \otimes\mathbf I_{d_{\rm out}}}/{\|\bm a_p\|^2}.$
By construction, $\Delta \bm w_p^{\mu\nu}$ is the minimum-norm weight perturbation that reproduces the pre-activation change induced by replacing $\bm a_p$ with $\bm a_{p'}$.

We first give a local decomposition of the gradient difference.

\begin{lemma}[AWD gradient decomposition]
\label{lem:grad_approx}
Under the AWD mapping, the mini-batch gradient difference admits the decomposition
\begin{equation}
\label{eq:grad_diff_decomposition}
    \bm g_{\mu\nu}
    =
    \frac1B
    \sum_{p\in\mathcal B_\mu}
    \left(
        \mathbf h_p(\bm w)
        \Delta \bm w_p^{\mu\nu}
        +
        \bm \rho_p^{\mu\nu}
    \right),
\end{equation}
\vskip -0.1in
where $\mathbf h_p(\bm w)=\nabla^2\ell_p(\bm w)$ is the per-sample Hessian and the AWD remainder is
\begin{equation}
\label{eq:grad_remainder}
    \bm \rho_p^{\mu\nu}
    =
    (\mathbf J_{\Delta,p}^{\mu\nu})^\top
    \nabla \ell_p(\bm w)
    +
    (\mathbf J_{\Delta,p}^{\mu\nu})^\top
    \mathbf h_p(\bm w)
    \Delta \bm w_p^{\mu\nu}
    +
    (\mathbf I+\mathbf J_{\Delta,p}^{\mu\nu})^\top
    \bm r_{T,p}^{\mu\nu}.
\end{equation}
Here $\bm r_{T,p}^{\mu\nu}$ is the Taylor remainder satisfying $\|\bm r_{T,p}^{\mu\nu}\|=\mathcal O\left(\|\Delta \bm w_p^{\mu\nu}\|^2\right),$
and the AWD Jacobian satisfies
$\|\mathbf J_{\Delta,p}^{\mu\nu}\|=\mathcal O\left({\|\Delta \bm a_p\|}/{\|\bm a_p\|}\right).$
\vskip -0.5in
\end{lemma}

\begin{proof}[Proof sketch]
By the AWD construction for the analyzed FCL block, the matched sample $p'$ is locally represented by evaluating sample $p$ at the perturbed weight $\bm w+\Delta \bm w_p^{\mu\nu}(\bm w)$. Thus,
$\ell_{p'}(\bm w) = \ell_p \left(\bm w+\Delta \bm w_p^{\mu\nu}(\bm w)\right).$
Taking the total derivative with respect to $\bm w$ and subtracting $\nabla\ell_p(\bm w)$ followed by averaging over the matched pairs in the mini-batch yields Eq.~\eqref{eq:grad_diff_decomposition}.
\end{proof}

\paragraph{Leading-order regime.}
Lemma~\ref{lem:grad_approx} separates the gradient difference into a Hessian-driven term and an explicit AWD remainder. We now express the leading-order regime through two local-perturbation small quantities.

\begin{assumption}[Local-perturbation small-quantity regime]
\label{assump:awd_small}
{For each matched pair $(p,p')$, the activity-induced AWD perturbation is small; and in the late-training regime near a generic minimum, the gradient-dependent AWD term is relatively small compared with the curvature-induced response.}
\begin{equation}
    \label{eq:asum1}
    \epsilon_{\rm act}
    \triangleq
    \frac{\|\Delta \bm a_p\|}{\|\bm a_p\|}
    \ll 1,
    \qquad
    \epsilon_{\rm min}
    \triangleq
    \frac{\|\nabla \ell_p(\bm w)\|}
    {\|\mathbf h_p(\bm w)\|\,\|\bm w\|}
    \ll 1.
\end{equation}
The first quantity $\epsilon_{\rm act}$ measures the local activity displacement selected by dense same-class matching; through the AWD map $\Delta\bm w_p^{\mu\nu}=\mathbf J_{\Delta,p}^{\mu\nu}\bm w$, it induces a small equivalent weight perturbation. 


\end{assumption}

Under Assumption~\ref{assump:awd_small}, the AWD remainder in Lemma~\ref{lem:grad_approx} is controlled by these two quantities and gives
\vskip -0.2in
\begin{equation}
\label{eq:grad_approx_final}
    \bm g_{\mu\nu}
    =
    \frac1B
    \sum_{p\in\mathcal B_\mu}
    \mathbf h_p(\bm w)
    \Delta \bm w_p^{\mu\nu} + \mathcal O(\epsilon_{\rm act}+\epsilon_{\rm min}).
\end{equation}
\vskip -0.08in
Detailed norm-level and projected controls are deferred to Appendix~\ref{app:awd_approx}.
We also empirically verify the smallness of the AWD remainder in \Cref{fig:metrics_cifar_cnn_cse,fig:metrics_mnist_fc_mse,fig:metrics_mnist_fc_cse}~(f\&g): panel~(f) reports the norm-level ratio $\|\bm \rho_p^{\mu\nu}\|/\|\mathbf h_p(\bm w)\Delta \bm w_p^{\mu\nu}\|$, while panel~(g) shows that the corresponding projected remainder-to-leading ratios remain small along the ordered Hessian eigendirections.
\vskip -0.2in
\paragraph{Remark.}
\vskip -0.08in
The generic-minimum condition above clarifies why AWD is essential. Starting from Eq.~\eqref{eq:cov}, one could obtain a Hessian second-moment form such as $\mathbb E_p[\mathbf h_p\bm\Sigma\mathbf h_p]$ by expanding $\nabla\ell_p(\bm w)$ around an interpolating minimum $\bm w^*$ with $\nabla\ell_p(\bm w^*)=\bm 0$ for every sample. This argument, however, is tied to interpolation and an unknown iterate covariance $\bm\Sigma$. AWD instead acts on gradient differences at the current iterate: canceling the shared zeroth-order per-sample gradient and leaves the curvature response plus the controlled remainder in Assumption~\ref{assump:awd_small} which is a weaker condition than interpolation. Appendix~\ref{app:awd_approx} further clarifies the relation between these two approaches.
\vskip -0.08in

\subsection{Spectral Structure of the Noise Covariance}
\label{sec:spectral_analysis}

We next derive the covariance-level consequence of the AWD gradient decomposition. Let $\{\bm v_i\}$ be the global Hessian eigenbasis, $\mathbf H\bm v_i=H_{ii}\bm v_i$, and let $\{\bm u_m^{(p)}\}$ be the local eigenbasis of the per-sample Hessian, $\mathbf h_p(\bm w)\bm u_m^{(p)}=\kappa_m^{(p)}\bm u_m^{(p)}$. The following lemma states the exact covariance decomposition; the asymptotic estimates for the associated remainders are given in Appendix~\ref{app:awd_approx}.
\begin{lemma}[AWD covariance decomposition]
\label{lem:cov_decomp}
Substituting Eq.~\eqref{eq:grad_approx_final} into
$\mathbf C=\frac12\mathbb E_{\mu,\nu}\!\left[\bm g_{\mu\nu}\bm g_{\mu\nu}^\top\right]$
,  $C_{ij}\triangleq\bm v_i^\top\mathbf C\bm v_j$ gives
\begin{equation}
\label{eq:cov_entry_decomp}
\begin{aligned}
    C_{ij}
    &=
    \frac{\sigma_w^2}{2B}
    \mathbb E_p
    \left[
        \sum_m
        (\kappa_m^{(p)})^2
        (\bm u_m^{(p)}\!\cdot\!\bm v_i)
        (\bm u_m^{(p)}\!\cdot\!\bm v_j)
    \right]
    +
    R^{\rm sample}_{ij}
    +
    R^{\rm iso}_{ij}
    +
    R^{\rm rem}_{ij},
\end{aligned}
\end{equation}

where $R^{\rm sample}$ collects cross-sample covariance terms, $R^{\rm iso}$ captures deviations from local isotropy of the AWD perturbation covariance, and $R^{\rm rem}$ collects AWD-remainder terms. 
Appendix~\ref{app:proof_thm1} gives the three remainder terms explicitly. The first two are statistical covariance errors due to the high dimensionality, while $R^{\rm rem}$ is controlled by the local-perturbation small quantities $\epsilon_{\rm act}$ and $\epsilon_{\rm min}$.

\end{lemma}


\begin{assumption}[Statistical covariance conditions]
\label{assump:spectral_cov_conditions}
We denote the explicit leading term in Eq.~\eqref{eq:cov_entry_decomp} by $D_{ij}$. Let $I_K$ denote the retained top K global Hessian eigendirections used in Section~\ref{sec:power-law}. In addition to the Assumption~\ref{assump:awd_small} which produces $|R^{\rm rem}_{ij}|=\mathcal O(\epsilon_{\rm act}+\epsilon_{\rm min})$ , assume that
\[
    \epsilon_{\rm stat}
    \triangleq
    \frac{|R^{\rm sample}_{ij}|+|R^{\rm iso}_{ij}|}
    {|D_{ii}| + |D_{jj}|}
    \ll 1.
\]
\end{assumption}
Detailed sufficient conditions showing how these small quantities $\epsilon$ control $R^{\rm rem}$, $R^{\rm sample}$, and $R^{\rm iso}$ are given in Appendix~\ref{app:proof_thm1}.

\begin{remark}
For $i\neq j$, the term $(\bm u_m^{(p)}\!\cdot\!\bm v_i)(\bm u_m^{(p)}\!\cdot\!\bm v_j)$ in $D_{ij}$ is the signed product of projections onto two orthogonal global eigendirections. In the high-dimensional regime here, these centered, weakly correlated products average out relative to the diagonal scale, $|D_{ij}|/\sqrt{|D_{ii}||D_{jj}|} \lesssim 1/\sqrt{N_{sample}}$ with $N_{sample} \sim 20,000$ denoting the number of samples, as formalized in Appendix~\ref{app:proof_thm1} and observed in \cref{fig:cmt_cifar_cnn_cse,fig:cmt_mnist_fc_cse,fig:cmt_mnist_fc_mse} (last row).
\end{remark}

\begin{theorem}[Spectral Decomposition of SGD Noise]
\label{thm:spectral_noise}
Under Assumption~\ref{assump:spectral_cov_conditions}, for every $i\in I_K$,
\begin{equation}
\label{eq:thm_result}
    C_{ii}=
    \frac{\sigma_w^2}{2B}
    \mathbb E_p
    \left[
        \sum_m
        (\kappa_m^{(p)})^2
        (\bm u_m^{(p)}\!\cdot\!\bm v_i)^2
    \right]
    \left[
    1+\mathcal O(\epsilon_{\rm stat}+\epsilon_{\rm act}+\epsilon_{\rm min})
    \right].
\end{equation}
\end{theorem}
(Proof provided in \textbf{Appendix \ref{app:proof_thm1}}.)

\begin{remark}[Dimensional Consistency]
    This result addresses the dimensional inconsistency discussed in Section \ref{sec:commute}. The noise variance $C_{ij}$ (dimension $[\mathcal{L}]^2/[\bm{w}]^2$) is expressed as the product of squared curvature $(\kappa_m^{(p)})^2$ (dimension $[\mathcal{L}]^2/[\bm{w}]^4$) and weight variance $\sigma_w^2$ (dimension $[\bm{w}]^2$). The resulting dimensions are consistent without introducing arbitrary constants.
\end{remark}

\textbf{Origin of Superlinear $\mathbf{C}$--$\mathbf{H}$ relationship.}
Eq.~\eqref{eq:thm_result} shows that $C_{ii}$ is governed by squared local curvatures $(\kappa_m^{(p)})^2$ of stiff sample modes and their alignment $(\bm u_m^{(p)}\!\cdot\!\bm v_i)^2$ with the global Hessian direction $\bm v_i$ \cite{sagunEigenvaluesHessianDeep2017,sagunEmpiricalAnalysisHessian2018,wuUnderstandingGeneralizationDeep2017,papyanFullSpectrumDeepnet2019,papyan2019measurementsthreelevelhierarchicalstructure}. By contrast, $H_{ii}=\mathbb E_p[\sum_m\kappa_m^{(p)}(\bm u_m^{(p)}\!\cdot\!\bm v_i)^2]$ is the corresponding first moment. Thus $\mathbf C$ emphasizes the second moment of local curvature whereas $\mathbf H$ tracks the first; in heavy-tailed bulk-and-outlier spectra \cite{xieOverlookedStructureStochastic2023,tang2025investigating}, this stronger weighting of stiff directions yields the observed superlinear $C_{ii}$--$H_{ii}$ scaling.

\begin{figure}[ht!]
\begin{center}
\includegraphics[width=0.9\columnwidth]{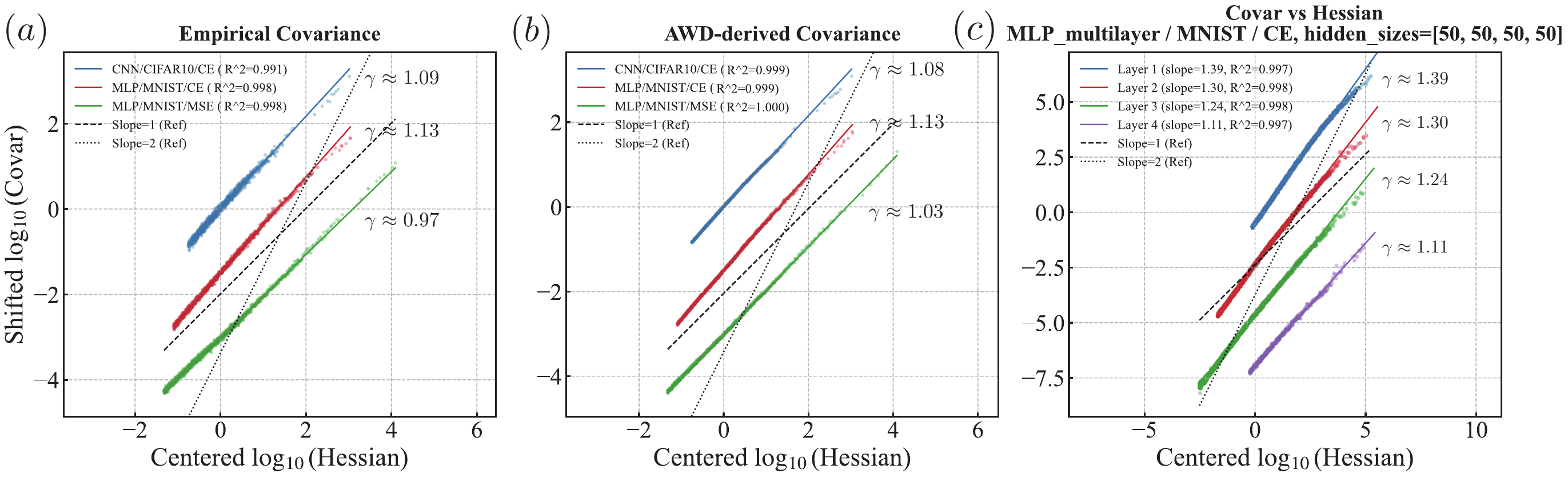}
\caption{Log-log plot of diagonal elements using the top $1000$ eigenvalues in single fully connected layers for models trained to convergence ($100\%$ training accuracy for CE and $>95\%$ for MSE). Data points are mean-centered and vertically shifted for visualization; solid lines denote linear fits.
\textbf{(a)} Empirical noise covariance versus the Hessian.
\textbf{(b)} AWD-derived noise covariance (Eq.~\ref{eq:thm_result}) versus the Hessian. The dotted and dashed lines correspond to slope $1$ and $2$, respectively.
\textbf{(c)} Layer-wise empirical covariance versus Hessian for MLP\_multilayer/MNIST/CE with hidden sizes $[50,50,50,50]$ and layer indices $1$--$4$.}
\label{fig:loglog}
\end{center}
\vskip -0.2in
\end{figure}

\section{Empirical Power-Law}
\label{sec:power-law}

Building on the spectral decomposition above, we now test whether the predicted covariance--curvature dependence appears as a measurable empirical law. Within the analyzed single fully connected layer, the diagonal entries of the SGD noise covariance and Hessian exhibit a clear power-law relation, and the fitted exponent provides a compact measure of the nonlinearity of the $\mathbf C$--$\mathbf H$ dependence.

We quantify this empirical relation by plotting $C_{ii}$ against $H_{ii}$ in log-log scale:
\begin{equation}
    C_{ii} \propto H_{ii}^{\gamma}.
\end{equation}
As shown in \Cref{fig:loglog}(a), the top 1000 dominant Hessian eigendirections are well fit by a line with slope $\gamma$ and fit quality $R^2>0.99$, consistent with prior observations \cite{xieOverlookedStructureStochastic2023,tang2025investigating}; excluding the degenerate Hessian tail improves numerical stability \cite{Meurant2006,yaopyhessian,tang2025investigating}. As shown in \Cref{fig:loglog}(b) (with additional results in \Cref{fig:lglg_cifar_cnn_cse,fig:lglg_mnist_fc_cse,fig:lglg_mnist_fc_mse}),
the same empirical fit within the analyzed fully connected layer
is reproduced by the AWD-derived covariance with similar fitted exponents.

Table~\ref{tab:gamma_comparison} compares the empirical exponent $\gamma_{\textnormal{emp}}$, fitted from the full covariance $\mathbf C$ in Eq.~\ref{eq:cov}, with the AWD-predicted exponent $\gamma_{\textnormal{AWD}}$ from the leading order term in Eq.~\ref{eq:thm_result}. The comparison covers multiple datasets, architectures, class counts, and CE/MSE losses. CE models are trained to $100\%$ training accuracy, while MSE models reach $>95\%$ within the same training budget. For CE, $\gamma_{\textnormal{AWD}}$ closely matches $\gamma_{\textnormal{emp}}$ and is consistently superlinear, reaching $\gamma\approx1.4$; for MSE, $\gamma_{\textnormal{emp}}$ remains much closer to $1$.

{The empirical analysis reveals a critical divergence: while MSE yields} $\gamma \approx 1$, {CE consistently exhibits superlinear scaling} $\gamma > 1$. This behavior stands in direct contradiction to the Fisher approximation, which predicts strict linearity for negative log-likelihood losses \cite{jastrzebskiThreeFactorsInfluencing2018,zhuAnisotropicNoiseStochastic2019,xieDiffusionTheoryDeep2021}.

Throughout this section, $C_{ii}\propto H_{ii}^{\gamma}$ denotes an empirical fit within the analyzed single fully connected layer, not a unique full-space exponent. The layer-wise multi-layer visualization in \Cref{fig:loglog}(c) illustrates this point for an MLP model on MNIST trained with CE and hidden sizes $[50,50,50,50]$. Specifically, the fitted $\gamma$ decreases as the layer approaches the output side. Appendix~\ref{app:multilayer_validation} further confirms this distinction: given layer-dependent fitted exponents, the joint multi-layer spectrum need not collapse to one log-log line. But still, the full model follows the AWD second-moment covariance structure, so the covariance--curvature relation remains superlinear in this second-moment sense.

Additional sweeps indicate that batch size and learning rate can change the fitted fully connected-layer exponent $\gamma$, while the observed values remain within $[1,2]$; a more systematic analysis is left for future work.

\section{Bounds on the Scaling Exponent}
\label{sec:bounds}

\subsection{Derivation of the Bounds}

To provide a theoretical foundation for the observed superlinear scaling within single fully connected layers, we establish bounds $1 \le \gamma \le 2$ derived from Theorem~\ref{thm:spectral_noise}. Our analysis focuses on the vicinity of generic minima, where these bounds can be established rigorously. While this approach relies on local convexity in the late stage of training, it avoids the restrictive ``true parameter'' (interpolating minima) assumption inherent in the Fisher approximation. This distinction is essential: in the highly non-convex landscapes of deep neural networks, convergence to a generic local minimum is a realistic outcome of training, whereas exact realizability is rarely guaranteed \cite{liuBadGlobalMinima2020,sunGlobalLandscapeNeural2020,liuLossLandscapesOptimization2022,zhangEmbeddingPrincipleHierarchical2022}.

\begin{theorem}[Bounds on $\gamma$]
\label{thm:bounds}
Consider the loss landscape in the vicinity of a generic minimum $\bm{w}^*$ where {local convexity} holds, i.e., the per-sample Hessians $\mathbf{h}_p(\bm{w}^*)$ are positive semi-definite (PSD). Let $\kappa_{\max}$ be the uniform upper bound of the per-sample eigenvalues. Then, the following properties hold:
\begin{enumerate}
    \item The diagonal elements of the noise covariance are bounded by the linear and quadratic forms of the Hessian:
    \begin{equation}
    \label{eq:inequality_chain}
        \frac{\sigma_w^2}{2B} H_{ii}^2 \leq C_{ii} \leq \frac{\sigma_w^2 \kappa_{\max}}{2B} H_{ii}.
    \end{equation}
    \item Consequently, if an empirical power-law relationship $C_{ii} \propto H_{ii}^{\gamma}$ hold, the fitted exponent $\gamma$ is theoretically bounded by:
    \begin{equation}
        1 \le \gamma \le 2.
    \end{equation}
\end{enumerate}
\end{theorem}

We defer the detailed proof to {Appendix \ref{app:bounds}}. 



\begin{remark}[Empirical Robustness Beyond Local Convexity]
\label{rem:early_stage_robustness}
Although Theorem~\ref{thm:bounds} is derived under the assumption of positive curvature at generic minima, empirical results (\Cref{fig:metrics_mnist_fc_mse,fig:metrics_cifar_cnn_cse,fig:metrics_mnist_fc_cse}) show that the fitted layerwise or single-subspace exponents remain in $[1,2]$ much earlier, well before convergence. This suggests that noise--curvature alignment is an intrinsic property of SGD, persisting even in training regimes where the Fisher approximation—relying on asymptotic realizability—fails to hold.
\end{remark}

\subsection{Physical Interpretation}
\label{sec:mechanisms}


To provide an intuitive interpretation of the scaling exponent $\gamma$, we examine the conditions under which the theoretical bounds $1 \le \gamma \le 2$ can be realized. Recall from Theorem \ref{thm:spectral_noise} that $\gamma$ {essentially quantifies the relationship between the first moment} ($\mathbf{H}$) {and the second moment} ($\mathbf{C}$) {of the local curvature spectrum.} This relationship is governed by the interplay between the local curvature magnitudes $\kappa_m^{(p)}$ and their projections onto the global eigenbasis $(\bm{u}_m^{(p)} \cdot \bm{v}_i)^2$.
We identify two distinct limiting cases:

\begin{proposition}[Perfect Alignment, $\gamma \to 2$]
\label{prop:gamma2}
Consider an idealized scenario of perfect geometric alignment, where local eigenbases align strictly with the global eigenbasis ($( \bm{u}_m^{(p)} \cdot \bm{v}_i)^2 = \delta_{mi}$), but the curvature magnitudes $\kappa_i^{(p)}$ fluctuate stochastically.
If $\text{Var}_p(\kappa_i^{(p)}) \propto (\mathbb{E}_p[\kappa_i^{(p)}])^2$), then the second moment scales quadratically with the Hessian:
\begin{equation}
    C_{ii} \propto H_{ii}^2 \implies \gamma = 2.
\end{equation}
\end{proposition}

\begin{proposition}[Spectral Degeneracy, $\gamma \to 1$]
\label{prop:gamma1}
Assume the per-sample Hessian spectrum is degenerate and dominated by the top $n$ eigenvalues with approximately the same magnitude independent of $m$, $\kappa_{m}^{(p)} \approx \bar{\kappa}^{(p)} \gg \kappa_\mathrm{rest}, (m \leq n)$. Instead of perfect global alignment, their eigenvectors $\bm{u}_m^{(p)}$ ($m \leq n$) are approximately aligned across different samples $p$ such that $(\bm{u}_m^{(p)} \cdot \bm{v}_i)^2 \approx \delta_{mi}\ $ for $i \leq n$, while $(\bm{u}_m^{(p)} \cdot \bm{v}_i)^2$ remains finite for $i > n$. Then the diagonal elements of Covariance and Hessian can be approximated as:
\vskip -0.2in
\begin{equation}   
    C_{ii} \propto \mu_{\kappa^2} \sum_{m=1}^n \mathbb{E}_p [(\bm{u}_m^{(p)} \cdot \bm{v}_i)^2]=\frac{\mu_{\kappa^2}}{\mu_\kappa}H_{ii},
\end{equation}
where $\mu_{\kappa^2} = \mathbb{E}_p[(\kappa_{m}^{(p)})^2]$ and $\mu_{\kappa} = {\mathbb{E}_p[\kappa_{m}^{(p)}]}$,
implying $\gamma \to 1$.

\end{proposition}

We defer the proofs to Appendix \ref{app:scaling_proofs}. In general, the exponent $\gamma$ captures how strongly stochastic curvature fluctuations are aligned with the global geometry of the loss landscape. 



\subsection{Distinction between CE and MSE}

We now try to explain the significant difference in $\gamma$ values between CE and MSE losses as observed in examples shown in Table 1. Empirically, we find that the per-sample hessian spectra are dominated by the first eigenvalue (i.e., $m=n=1$), as illustrated in \Cref{fig:suppres,fig:abl_cifar_mlp_cse,fig:abl_mnist_fc_mse,fig:abl_mnist_fc_cse} (see columns 2 and 3) for both MSE and CE losses. This empirical observation is verified by a ``suppression experiment" as shown in (\cref{fig:abl_mnist_fc_cse,fig:abl_cifar_mlp_cse,fig:abl_mnist_fc_mse}) with details described in Appendix \ref{app:ablation_details}. Briefly, to isolate the impact of the largest first eigenvalues of per-sample hessians, we suppress the magnitude of the rest of the eigenvalues by several orders of magnitude. 
We find that the global Hessian spectrum $H_{ii}$ and Covariance diagonal $C_{ii}$ from the resulting suppressed per-sample hessians hardly change for both CE and MSE losses, as shown in Figure \ref{fig:suppres}  (a,b).

The dominance of the first eigenvalue leads to a simple formula for $C_{ii} \propto \mathbb{E}_p \left[  (\kappa_1^{(p)})^2 (\bm{u}_1^{(p)} \cdot \bm{v}_i)^2 \right]=\mathbb{E}_p \left[  XY \right]$ expressed as the correlation between two stochastic variables of the system: $X \equiv (\kappa_1^{(p)})^2$ -- the magnitude of the first eigenvalue and $Y \equiv (\bm{u}_1^{(p)} \cdot \bm{v}_i)^2$ -- the alignment of the corresponding eigenvector with the eigenvector of $\mathbf{H}$. When $X$ and $Y$ are independent of each other, we have $C_{ii}=\mathbb{E}_p \left[  X\right]\mathbb{E}_p \left[  Y\right] \propto H_{ii}$, which means $\gamma = 1$. If there is a positive correlation between $X$ and $Y$, it can lead to a superlinear dependence with $\gamma > 1$. 

We posit this $X-Y$ correlation depends on the geometry of loss function, thereby driving the distinction between CE with $\gamma > 1$ and MSE with $\gamma \to 1$. To test this hypothesis, we replaced these first eigenvalues with their mean value and found that the resulting $H_{ii}$ and $C_{ii}$ hardly change for models with MSE loss (Figure \ref{fig:suppres} (d)) but undergo significant distortion for CE loss, especially for the case $\gamma \sim 1.4$ (Figure \ref{fig:suppres} (c)), which leads to lowering of the exponent $\gamma \to 1$. We emphasize that while our experiments isolate this correlation as the reason for the distinction between MSE and CE losses, the underlying theoretical mechanism driving this behavior remains open for future work \cite{soudry2018the,lu2021neuralcollapsecrossentropyloss}.

\section{Conclusion and Discussion } \label{sec:concls}

In this paper, we explain the robust positive correlation between the SGD noise covariance ($\mathbf{C}$) and the Hessian ($\mathbf{H}$). Using the recently discovered Activity-Weight Duality (AWD), we show that $\mathbf{C}$ and $\mathbf{H}$ are the second and first moments of the same per-sample Hessian distribution, naturally giving rise to their strong correlation independent of the loss function. In particular, we find that $\mathbf{C}$ and $\mathbf{H}$ approximately commute, implying aligned eigenvectors, while their eigenvalues follow an empirical power law, $C_{ii} \propto H_{ii}^{\gamma}$, with $1 \le \gamma \le 2$ in fully connected layers.

Although our derivation uses the explicit weight perturbation $\Delta \bm{w}$ for fully connected layers, the underlying AWD framework is model-agnostic. As implied by Theorem~\ref{thm:spectral_noise}, $C_{ii}$ depends on $\Delta \bm{w}$ only through an overall prefactor, $\sigma_w^2 \delta_{mn}$, which is insensitive to architectural details. We therefore expect the $\mathbf{C}$--$\mathbf{H}$ relation to extend to other architectures, including convolutional networks and Transformers, though with layer- and model-dependent exponents. Consistent with this expectation, additional multi-layer experiments in Appendix~\ref{app:multilayer_validation} demonstrate that the layerwise power laws remain remarkably robust, with $R^2>0.99$. Moreover, the joint multi-layer covariance is still accurately captured by the AWD second-moment prediction, even though it no longer follows a single clean power law with the joint Hessian.

Unlike the Fisher approximation, our approach directly probes the microscopic structure of per-sample Hessians rather than relying on likelihood-based surrogates. Suppression experiments trace the observed superlinear scaling ($\gamma \sim 1.4$) to the entanglement between curvature eigenvalues and projection directions, clarifying differences between MSE and cross-entropy losses. A key implication is that $\mathbf{C}$ provides a superlinear proxy for Hessian geometry that is more informative than the linear Fisher approximation. In particular, $\gamma>1$ reflects enhanced anisotropic noise along stiff directions, suggesting a possible mechanism for biasing optimization toward flatter regions of the loss landscape.


\bibliographystyle{unsrtnat}
\bibliography{deeplearning}


\appendix

\section{Related Works}
\label{app:RelatedWorks}
\textbf{Flat Minima and Generalization.}
It is widely accepted that SGD tend to converge to flat minima closely related with good generalization, despite the non-convex loss landscape \cite{hochreiterFlatMinima1997,hardttrainfaster2016,keskarLargebatchTrainingDeep2017,zhangUnderstandingDeepLearning2017,arpitAcloser2017,hofferTrainLongerGeneralize2017,Neyshaburexploring2017,yaohessianbased2018}. Specifically, flatness has been shown to be a reliable predictor of generalization \cite{jiangFantasticGeneralizationMeasures2020}, with low-complexity solutions often exhibiting small Hessian norms \cite{wuUnderstandingGeneralizationDeep2017}. While the definition of flatness is debated due to its sensitivity to parameter rescaling \cite{dinhsharp2017}, recent approaches effectively address this via scale-invariant measures or PAC-Bayesian frameworks \cite{DziugComputingNonvacuous2017,MandtStochatic2017,chaudhariStochasticGradientDescent2018,smithBayesianPerspectiveGeneralization2018,tsuzukuNormalizedFlatMinima2020}. Most recently, \cite{fengActivityWeightDuality2023} reveals that two key factors, sharpness and weight norm, act together to determine generalization by using the Activity–Weight Duality relation.

\textbf{SGD Dynamics and its Diffusion Approximation.} SGD optimization behaviors have been extensively studied from various perspectives. Early works investigated the convergence of simple one-hidden-layer neural networks \cite{Liconvergence2017,brutzkussgdlearnsoverparameterizednetworks2017}, then in non-convex setting \cite{daneshmand2018,jin2017,hu2018}. On the other hand, modeling SGD as a continuous-time stochastic process is a powerful theoretical tool \cite{jastrzebskiThreeFactorsInfluencing2018,liStochasticModifiedEquations2017,MandtStochatic2017,wuHowSGDSelects2018,xuglobal2018,hu2018,nguyenFirstExitTime2019,zhuAnisotropicNoiseStochastic2019,xieDiffusionTheoryDeep2021,livaliditymodelingsgdstochastic2021}.
By modeling the evolution of parameter probability densities, research on Stochastic Gradient Langevin Dynamics (SGLD) has analyzed density diffusion under injected isotropic noise \cite{satoSGLD2014,RaginskySGLD2017,zhanghittingtimeanalysisstochastic2018}. Meanwhile, \cite{jastrzebskiThreeFactorsInfluencing2018} studied the minima selection probability of SGD.
\cite{zhuAnisotropicNoiseStochastic2019,xieDiffusionTheoryDeep2021} suggests that the anisotropic diffusion inherent in SGD leads to flatter minima compared to isotropic diffusion. 
Related stability-based views connect the late-stage behavior of SGD to curvature and gradient-noise structure, including dynamical-stability analyses near minima and studies of the maximum Hessian eigenvalue along training trajectories \cite{wuHowSGDSelects2018,Ma2021OnLS,lee2023a,andreyev2025edgestochasticstabilityrevisiting}. Our results complement this line of work by resolving how the first and second moments of per-sample Hessians determine the anisotropic covariance structure in the Hessian basis.

\textbf{SGD Noise Relates to the Hessian.} The Hessian of the Loss Landscape is pivotal in optimization and generalization \cite{ghorbaniInvestigationNeuralNet2019,li2019hessianbasedanalysissgd,Jacot2020The,yaohessianbased2018,dauphinIdentifyingAttackingSaddle2014,kaurMaximumHessianEigenvalue2023}. It is found that the Hessian relates to the gradient covariance matrix, and are both highly anisotropic \cite{sagunEigenvaluesHessianDeep2017,sagunEmpiricalAnalysisHessian2018,xingWalkSGD2018}. Assuming an
equivalence between the Fisher Information Matrix and the Hessian for negative log-likelihood
losses, \cite{jastrzebskiThreeFactorsInfluencing2018,zhuAnisotropicNoiseStochastic2019,xieDiffusionTheoryDeep2021} argued that $\mathbf{C}$ is proportional to $\mathbf{H}$ at the ``true parameter".
Meanwhile, empirical studies consistently reveal that the Covariance and Hessian spectrum follows a ``bulk-and-outlier" structure: a massive bulk of near-zero eigenvalues and a small number of large outliers \cite{sagunEigenvaluesHessianDeep2017,sagunEmpiricalAnalysisHessian2018,wuUnderstandingGeneralizationDeep2017,papyanFullSpectrumDeepnet2019,papyan2019measurementsthreelevelhierarchicalstructure}. Recent studies analyzed the heavy-tailed structure of the SGD noise covariance and the local Hessian and discovered a Power law structure\cite{xieOverlookedStructureStochastic2023,tang2025investigating}.

\begin{table}[hbtp] 
\caption{Comparison of the scaling exponent $\gamma$ between empirical measurements ($\gamma_{\textnormal{emp}}$) and AWD-based theoretical predictions ($\gamma_{\textnormal{AWD}}$). Experiments were conducted on MNIST and CIFAR-10 datasets ($\mathcal{C}$ denotes the number of classes used) using MLP and CNN architectures. Values represent the mean $\pm$ standard deviation over 4 independent runs.}
\label{tab:gamma_comparison}
  \begin{center}
    \begin{small}
      \begin{sc}
\resizebox{\linewidth}{!}{
\begin{tabular}{llc cc cc}
\toprule
& & & \multicolumn{2}{c}{\textbf{MSE}} & \multicolumn{2}{c}{\textbf{CE}} \\ 
\cmidrule(lr){4-5} \cmidrule(lr){6-7}
\textbf{Data} & \textbf{Model} & $\mathcal{C}$ & $\gamma_{\textnormal{emp}}$ & $\gamma_{\textnormal{AWD}}$ & $\gamma_{\textnormal{emp}}$ & $\gamma_{\textnormal{AWD}}$ \\
\midrule
\multirow{3}{*}{MNIST} & \multirow{3}{*}{MLP} 
  & 3  & 1.07$\pm$0.02 & 1.14 $\pm$0.02 & 1.22$\pm$0.01 & 1.24$\pm$0.01\\
& & 6  & 1.01$\pm$0.01 & 1.08 $\pm$0.01 & 1.19$\pm$0.01 & 1.19$\pm$0.01\\
& & 10 & 0.96$\pm$0.01 & 1.03 $\pm$0.01 & 1.14$\pm$0.01 & 1.14$\pm$0.01 \\
\cmidrule{2-7}
& \multirow{3}{*}{CNN} 
  & 3  & 0.99$\pm$0.01 & 1.05 $\pm$0.02 & 1.38$\pm$0.05 & 1.33$\pm$0.03 \\
& & 6  & 1.03$\pm$0.01 & 1.06 $\pm$0.01 & 1.33$\pm$0.01 & 1.31$\pm$0.01 \\
& & 10 & 0.98$\pm$0.01 & 1.05 $\pm$0.01 & 1.31$\pm$0.02 & 1.29$\pm$0.02 \\
\midrule
\multirow{3}{*}{CIFAR-10} & \multirow{3}{*}{MLP} 
  & 3  & 1.12$\pm$0.04 & 1.17 $\pm$0.03 & 1.27$\pm$0.04 & 1.27$\pm$0.04 \\
& & 6  & 1.03$\pm$0.01 & 1.08 $\pm$0.02 & 1.43$\pm$0.01 & 1.43$\pm$0.01 \\
& & 10 & 1.00$\pm$0.02 & 1.03 $\pm$0.02 & 1.51$\pm$0.03 & 1.47$\pm$0.03 \\
\cmidrule{2-7}
& \multirow{3}{*}{CNN} 
  & 3  & 1.04$\pm$0.04 & 1.09 $\pm$0.03 & 1.10$\pm$0.01& 1.11$\pm$0.01 \\
& & 6  & 1.00$\pm$0.01 & 1.01 $\pm$0.01 & 1.08$\pm$0.02& 1.08$\pm$0.01 \\
& & 10 & 1.01$\pm$0.01 & 1.02 $\pm$0.01 & 1.08$\pm$0.01& 1.08$\pm$0.01\\
\bottomrule
\end{tabular}} 
      \end{sc}
    \end{small}
  \end{center}
  \vskip -0.1in
\end{table}

\section{Experimental Setup}
\label{sec:expset}

To ensure the reproducibility of our results and facilitate further research, the code is open sourced at https://anonymous.4open.science/r/AWCH-1A6A.

\subsection{Architecture Details}

Unless otherwise noted, our mainline models were trained on the MNIST dataset (or its fixed subset) \cite{lecun1998mnist} and the CIFAR-10 dataset (or its fixed subset) \cite{krizhevsky2009learning}. We used vanilla Stochastic Gradient Descent without extra Regularization. All models were trained to convergence with $100\%$ training accuracy for cross-entropy (CE) or $>95\%$ training accuracy for mean squared error (MSE) loss. To obtain stable hessian spectra, we used a $\textit{Softmax}$ layer before the mean squared error.

\subsubsection{MLP}

The MLP models on MNSIT dataset consist of two hidden layers of width 50 ($784 \rightarrow 50 \rightarrow 50 \rightarrow 10$), employing ReLU activations without bias. We focus on the specific weight matrix connecting the two hidden layers, treating it as a flattened parameter vector of dimension $D = 50 \times 50 =2500$ and calculate the Covariance $\textbf{C}$ and the Hessian $\textbf{H}$ in this subspace at each checkpoint.

For the multi-layer MNIST MLP used in \Cref{fig:loglog}(c) and Appendix~\ref{app:multilayer_validation}, we use a deeper bias-free ReLU network with hidden widths $[50,50,50,50]$, i.e.,
$784\rightarrow 50\rightarrow 50\rightarrow 50\rightarrow 50\rightarrow 10$.
The model is trained on MNIST with CE loss to $100\%$ training accuracy using the same vanilla-SGD protocol as the other MNIST MLP experiments. For the layer-wise analysis, we compute $\mathbf C$ and $\mathbf H$ separately on the four hidden weight matrices indexed from the input side to the output side. For the joint multi-layer analysis, these analyzed weight matrices are concatenated into a single parameter vector before computing the joint covariance, Hessian, and AWD-derived covariance.

The MLP models on CIFAR-10 dataset consist of three hidden layers ($3072 \rightarrow 1000 \rightarrow 50 \rightarrow 50 \rightarrow 10$), employing ReLU activations without bias. We focus on the specific weight matrix connecting the last two hidden layers, treating it as a flattened parameter vector of dimension $D = 50 \times 50 =2500$ and calculate the Covariance $\textbf{C}$ and the Hessian $\textbf{H}$ in this subspace at each checkpoint.

\subsubsection{CNN}

The CNN models on MNSIT and CIFAR-10 datasets share the same architecture (with adaptive input channel), of which the config is given by $ \left [32, M, 64, M, 128, 128, M\right ]$, where the number means the channel depth (with kernel size 3 and padding 1) and $M$ represents the $\textit{Max pooling}$ (with kernel size 2 and stride 2), followed by the adaptive average pooling with out put dimension $128 \times 1 \times 1$. These feature maps are flattened into a vector of dimension 128 and fed into a classifier consisting of a hidden fully connected layer and an out put layer ($ 128 \rightarrow 20 \rightarrow 10$). We use the ReLu activation without bias. To isolated the noise from mini batch sampling, we do not use the Batch Normalization. We focus on the specific weight matrix connecting the features and the hidden layer, treating it as a flattened parameter vector of dimension $D = 128 \times 20 =2560$ and calculate the Covariance $\textbf{C}$ and the Hessian $\textbf{H}$ in this subspace at each checkpoint.

\subsection{Visualizations Details}
In this section, we detail the experimental setup used to generate the presented visualizations including the figures and table.
\subsubsection{Figures Details}

Unless otherwise noted, the CNN model presented in the figures was trained on a fixed, balanced subset of the CIFAR-10 dataset. This subset comprises 2,000 examples per class, totaling 20,000 training samples. The model was trained using the Cross-Entropy loss function for 100 epochs, achieving convergence with $100\%$ training accuracy. Optimization was performed using vanilla SGD with a batch size of $B=128$ and a learning rate of $\eta = 0.1$.

The MLP models presented in the figures were trained on a fixed, balanced subset of the MNIST dataset, consisting of 2,000 examples per class for a total of 20,000 training samples. Models were trained for 100 epochs using vanilla SGD with a batch size of $B=50$ and a learning rate of $\eta = 0.1$. Under these settings, model minimized with Cross-Entropy (CE) loss converged to $100\%$ training accuracy, while that minimizing Mean Squared Error (MSE) loss achieved $>95\%$ training accuracy.

\subsubsection{Table Details}
\label{app:tableD}
Detailed training hyperparameters for the results in Table \ref{tab:gamma_comparison} are provided below. All models were optimized using vanilla SGD with a learning rate of $\eta = 0.1$. For each hyperparameter configuration, we performed 4 independent runs with distinct random seeds. We computed $\gamma$ after a fixed number of training epochs. The number of top eigenvalues ($\mathcal{N}$) selected for computation varied based on the number of classes $\mathcal{C} \in \{3, 6, 10\}$ included in the training subset. In Table \ref{tab:gamma'_comparison}, we show the $\gamma'_{\text{AWD}}$ (via Eq.\ref{eq:awd_appro}), which explicitly accounts for the gradient term $\mathbf{C}^{hg}$ and $\mathbf{C}^{gg}$, a slightly better approximation of $\gamma_{\text{emp}}$.

\begin{table}[h]
\centering
\caption{Hyperparameters for models presented in Table \ref{tab:gamma_comparison}. $N_{data}$ denotes total samples per class.}
\label{tab:hyperparams}
\begin{small}
\begin{sc}
\begin{tabular}{lcccccccc}
\toprule
 & & \multicolumn{3}{c}{\textbf{Training Setup}} & \multicolumn{3}{c}{\textbf{Eigenvalues ($\mathcal{N}$)}} \\
\cmidrule(lr){3-5} \cmidrule(lr){6-8}
\textbf{Model} & \textbf{Loss} & $N_{data}$ & Batch ($B$) & Epochs & $\mathcal{C}=3$ & $\mathcal{C}=6$ & $\mathcal{C}=10$ \\
\midrule
\textbf{MNIST} & & & & & & & \\
MLP & CE & 2,000 & 50 & 100 & 300 & 1,000 & 1,000 \\
MLP & MSE & 2,000 & 50 & 100 & 300 & 1,000 & 1,000 \\
CNN & CE & 2,000 & 50 & 100 & 200 & 500 & 1,000 \\
CNN & MSE & 5,000 & 128 & 100 & 200 & 300 & 800 \\
\midrule
\textbf{CIFAR-10} & & & & & & & \\
MLP & CE & 2,000 & 100 & 150 & 800 & 1,500 & 1,500 \\
MLP & MSE & 5,000 & 100 & 100 & 500 & 1,000 & 1,000 \\
CNN & CE & 2,000\textsuperscript{*} & 128 & 100 & 500 & 1,000 & 1,500 \\
CNN & MSE & 5,000 & 128 & 500 & 500 & 500 & 1,000 \\
\bottomrule
\end{tabular}
\end{sc}
\end{small}
\raggedright \footnotesize{\textsuperscript{*}For the 3-class case, $N_{data}=5,000$.}
\end{table}

\begin{table}[t] 
\caption{\textbf{Verification of Activity-Weight Duality (AWD) across architectures and complexities near the global minimum for MSE loss.} The comparison covers different model architectures (MLP, CNN), datasets, and task complexities (number of classes $\mathcal{C}$). For each hyperparameter configuration, we repeated the experiment 4 times and report the mean and standard deviation.}
\label{tab:gamma'_comparison}

\centering
\small 
\setlength{\tabcolsep}{4pt} 
\begin{tabular}{llc ccc}
\toprule
& & & \multicolumn{2}{c}{\textbf{MSE}} \\ 
\cmidrule(lr){4-6}
\textbf{Data} & \textbf{Model} & $\mathcal{C}$ & $\bm{\gamma_{\text{emp}}}$ & $\bm{\gamma_{\text{AWD}}}$ & $\bm{\gamma'_{\text{AWD}}}$ \\
\midrule
\multirow{3}{*}{MNIST} & \multirow{3}{*}{MLP}  
  & 3    & 1.07$\pm$0.02 & 1.14 $\pm$0.02 &  1.07$\pm$0.02\\
& & 6    & 1.01$\pm$0.01 & 1.08 $\pm$0.01 &  1.03$\pm$0.01\\
& & 10   & 0.96$\pm$0.01 & 1.03 $\pm$0.01 &  1.00$\pm$0.01\\
\cmidrule{2-6}
& \multirow{3}{*}{CNN}  
  & 3    & 0.99$\pm$0.01 & 1.05 $\pm$0.02 &  1.02$\pm$0.02\\
& & 6    & 1.03$\pm$0.01 & 1.06 $\pm$0.01 &  1.05$\pm$0.01\\
& & 10   & 0.98$\pm$0.01 & 1.05 $\pm$0.01 &  1.03$\pm$0.01\\
\midrule
\multirow{3}{*}{CIFAR-10} & \multirow{3}{*}{MLP}  
  & 3    & 1.12$\pm$0.04 & 1.17 $\pm$0.03 &  1.09$\pm$0.03\\
& & 6    & 1.03$\pm$0.01 & 1.08 $\pm$0.02 &  1.06$\pm$0.02\\
& & 10   & 1.00$\pm$0.02 & 1.03 $\pm$0.02 &  1.03$\pm$0.02\\
\cmidrule{2-6}
& \multirow{3}{*}{CNN}  
  & 3    & 1.04$\pm$0.04 & 1.09 $\pm$0.03 &  1.06$\pm$0.03\\
& & 6    & 1.00$\pm$0.01 & 1.01 $\pm$0.01 &  0.97$\pm$0.01\\
& & 10   & 1.01$\pm$0.01 & 1.02 $\pm$0.01 &  0.97$\pm$0.01\\
\bottomrule
\end{tabular}

\vskip -0.1in
\end{table}
\subsection{Ablation Experiments Details}
\label{app:ablation_details}

In this section, we provide the precise algorithmic details for the ``suppression experiment" ablation study presented in Section \ref{sec:mechanisms}. 

\subsubsection{Experimental Procedure}

The procedure consists of three stages: per-sample spectral decomposition, intervention (filtering and homogenization), and global reconstruction.

\textbf{1. Per-Sample Decomposition.}
For a given mini-batch of size $B$, we compute the per-sample Hessian $\mathbf{h}_p$ for each sample $p \in \{1, \dots, B\}$. We perform a full eigen-decomposition for each sample:
\begin{equation}
    \mathbf{h}_p = \sum_{m=1}^D \kappa_m^{(p)} \bm{u}_m^{(p)} (\bm{u}_m^{(p)})^\top,
\end{equation}
where eigenvalues are sorted in descending order $\kappa_1^{(p)} \ge \kappa_2^{(p)} \ge \dots$. Since the Hessian spectrum is heavily dominated by the top mode, we focus our intervention on the principal eigenvalue $\kappa_1^{(p)}$.

\textbf{2. Identification of Stiff Samples.}
We set a relative threshold $\tau$ based on a percentage $\theta$ (e.g., $10\%$) of the maximum principal eigenvalue observed in the batch:
\begin{equation}
    \tau = \theta \cdot \max_{p \in \{1 \dots B\}} \kappa_1^{(p)}.
\end{equation}
The set of stiff samples is defined as $\mathcal{S}_\text{stiff} = \{ p \mid \kappa_1^{(p)} > \tau \}$.

\textbf{3. Homogenization and Suppression.}
We construct a modified curvature profile $\tilde{\kappa}_m^{(p)}$ for each sample:
\begin{itemize}
    \item \textbf{Stiff Sample Homogenization:} For stiff samples ($p \in \mathcal{S}_\text{stiff}$), we replace their dominant eigenvalue $\kappa_1^{(p)}$ with the group mean $\bar{\kappa}_\text{stiff}$ to remove magnitude variance. To ensure rank-1 dominance, we suppress their tail modes ($m > 1$) by a factor of $\epsilon_{tail} = 10^{-5}$.
    \begin{equation}
        \bar{\kappa}_\text{stiff} = \frac{1}{|\mathcal{S}_\text{stiff}|} \sum_{p \in \mathcal{S}_\text{stiff}} \kappa_1^{(p)}.
    \end{equation}
    \item \textbf{Background Suppression:} For non-stiff samples ($p \notin \mathcal{S}_\text{stiff}$), we suppress their rank-1 still modes by a factor of $\epsilon_{bg} = 10^{-3}$ and their tail modes ($m > 1$) by a factor of $\epsilon_{tail} = 10^{-5},$ treating them as background noise relative to the stiff samples.
\end{itemize}
Formally, the intervened eigenvalues $\tilde{\kappa}_m^{(p)}$ are set as:
\begin{equation}
    \tilde{\kappa}_m^{(p)} = 
    \begin{cases} 
        \bar{\kappa}_\text{stiff} & \text{if } p \in \mathcal{S}_\text{stiff} \text{ and } m=1, \\
        10^{-5} \cdot \kappa_m^{(p)} & \text{if }  m > 1, \\
        10^{-3} \cdot \kappa_m^{(p)} & \text{if } p \notin \mathcal{S}_\text{stiff} \text{ and } m=1 
    \end{cases}
\end{equation}

\textbf{4. Global Reconstruction.}
Finally, we reconstruct the global Hessian $\tilde{\mathbf{H}}$ using the modified spectral components:
\begin{equation}
    \tilde{H}_{ii} = \frac{1}{B} \sum_{p=1}^B \sum_{m=1}^D \tilde{\kappa}_m^{(p)} \left( \bm{u}_m^{(p)} \cdot \bm{v}_i \right)^2.
\end{equation}
We then compute the eigenvalues of this reconstructed global matrix $\tilde{\mathbf{H}}$ and compare its spectrum to the original $\mathbf{H}$. We also reconstruct the covariance $2\tilde{\mathbf{C}}_{ii}/\sigma_w^2 =  \frac{1}{B} \sum_{p=1}^B \sum_{m=1}^D (\tilde{\kappa}_m^{(p)})^2 \left( \bm{u}_m^{(p)} \cdot \bm{v}_i \right)^2 $ similarly.

\section{Derivation of the Random Baseline for Scale-Invariant Alignment}
\label{app:random_baseline}

In this section, we derive the theoretical expected value of the scale-invariant off-diagonal coupling under the null hypothesis of a random basis. This serves as a baseline to evaluate the statistical significance of the observed alignment between the noise covariance matrix $\mathbf{C}$ and the Hessian $\mathbf{H}$.

\subsection{Problem Setup}

Let $\mathbf{C} \in \mathbb{R}^{D \times D}$ be a covariance matrix. We analyze its structure in a specific basis (e.g., the Hessian eigenbasis). The scale-invariant correlation matrix $\mathbf{R}$ is defined as:
\begin{equation}
    R_{ij} = \frac{C_{ij}}{\sqrt{|C_{ii} C_{jj}|}}.
\end{equation}
We define the \textit{mean off-diagonal coupling} metric as $\mu_{\text{coupling}} = \mathbb{E}[|R_{ij}|]$ for $i \neq j$.

To establish a random baseline, we assume the \textbf{Null Hypothesis ($H_0$)}: The eigenvectors of $\mathbf{C}$ are random and uniformly distributed on the orthogonal group $O(D)$ (Haar measure), having no structural alignment with the basis of observation.

\subsection{The Spiked Spectrum Model}

Deep learning optimization landscapes often exhibit a ``heavy-tailed'' or ``spiked'' spectrum. We model the spectrum of $\mathbf{C}$ using a simplified \textit{Spiked Covariance Model}:
\begin{itemize}
    \item \textbf{Dominant Subspace:} The top $M$ eigenvalues ($\lambda_i\equiv H_{ii}$) contain the majority of the energy: $\lambda_1 \approx \dots \approx \lambda_M \approx \sigma^2$.
    \item \textbf{Bulk/Tail:} The remaining $D-M$ eigenvalues are negligible: $\lambda_{M+1} \approx \dots \approx \lambda_D \approx 0$.
    \item \textbf{Sparsity:} $M \ll D$.
\end{itemize}

Under the random basis assumption, $\mathbf{C}$ can be expressed as:
\begin{equation}
    \mathbf{C} = \mathbf{Q} \mathbf{\Lambda} \mathbf{Q}^\top,
\end{equation}
where $\mathbf{\Lambda} = \text{diag}(\lambda_k)$ and $\mathbf{Q}$ is a random orthogonal matrix. The elements of $\mathbf{C}$ are given by:
\begin{equation}
    C_{ij} = \sum_{k=1}^D \lambda_k Q_{ik} Q_{jk} \approx \sigma^2 \sum_{k=1}^M Q_{ik} Q_{jk}.
\end{equation}

\subsection{Statistical Approximations in High Dimensions}

For large $D$, the entries of a random orthogonal matrix $\mathbf{Q}$ can be approximated as independent Gaussian variables:
\begin{equation}
    Q_{ik} \sim \mathcal{N}\left(0, \frac{1}{D}\right).
\end{equation}

By using the approximation $\mathbb{E}[|C_{ij}| / \sqrt{C_{ii}C_{jj}}] \approx \mathbb{E}[|C_{ij}|] / \mathbb{E}[\sqrt{C_{ii}C_{jj}}]$, justified by the concentration of measure phenomenon in high-dimensional spaces. The denominator, representing the diagonal energy, is a sum of $M$ independent squared variables, which concentrates sharply around its mean. Thus, it can be treated as a constant scalar relative to the zero-mean fluctuations of the numerator.

We now analyze the numerator and denominator of $|R_{ij}|$ separately. 

\paragraph{1. The Denominator (Diagonal Elements).}
The diagonal elements represent the energy projected onto each dimension:
\begin{equation}
    C_{ii} \approx \sigma^2 \sum_{k=1}^M Q_{ik}^2.
\end{equation}
Since $Q_{ik}^2$ follows a scaled Chi-squared distribution with 1 degree of freedom, its expectation is $\mathbb{E}[Q_{ik}^2] = \frac{1}{D}$. By the Law of Large Numbers (since $M$ is sufficiently large but $M \ll D$), the sum concentrates around its mean:
\begin{equation}
    C_{ii} \approx \mathbb{E}[C_{ii}] = \sigma^2 \cdot M \cdot \frac{1}{D} = \frac{M \sigma^2}{D}.
\end{equation}
Thus, the normalization factor in the denominator is approximately:
\begin{equation}
    \text{Denom} = \sqrt{C_{ii} C_{jj}} \approx \frac{M \sigma^2}{D}.
\end{equation}

\paragraph{2. The Numerator (Off-Diagonal Elements).}
For $i \neq j$, the off-diagonal element is a sum of products:
\begin{equation}
    C_{ij} \approx \sigma^2 \sum_{k=1}^M Q_{ik} Q_{jk}.
\end{equation}
Let $X_k = Q_{ik} Q_{jk}$. Since $Q_{ik}$ and $Q_{jk}$ are independent (for $i \neq j$), $\mathbb{E}[X_k] = 0$. The variance of each term is:
\begin{equation}
    \text{Var}(X_k) = \mathbb{E}[Q_{ik}^2] \mathbb{E}[Q_{jk}^2] = \frac{1}{D} \cdot \frac{1}{D} = \frac{1}{D^2}.
\end{equation}
By the Central Limit Theorem (CLT), the sum $C_{ij}$ follows a Gaussian distribution:
\begin{equation}
    C_{ij} \sim \mathcal{N}\left(0, \sigma^4 \cdot M \cdot \frac{1}{D^2}\right).
\end{equation}
We are interested in the expected absolute value $\mathbb{E}[|C_{ij}|]$. For a generic Gaussian variable $Z \sim \mathcal{N}(0, s^2)$, the mean absolute deviation is $\mathbb{E}[|Z|] = \sqrt{\frac{2}{\pi}} s$. Therefore:
\begin{equation}
    \mathbb{E}[|C_{ij}|] \approx \sqrt{\frac{2}{\pi}} \cdot \sqrt{\frac{M \sigma^4}{D^2}} = \sqrt{\frac{2}{\pi}} \frac{\sqrt{M} \sigma^2}{D}.
\end{equation}

\subsection{Theoretical Baseline Result}

Combining the approximations for the numerator and denominator, the scaling factors $\sigma^2$ and $D$ cancel out, demonstrating the scale-invariant nature of the metric:

\begin{align}
    \mathbb{E}[|R_{ij}|] &\approx \frac{\mathbb{E}[|C_{ij}|]}{\mathbb{E}[\sqrt{C_{ii} C_{jj}}]} \nonumber \\
    &\approx \frac{\sqrt{\frac{2}{\pi}} \frac{\sqrt{M} \sigma^2}{D}}{\frac{M \sigma^2}{D}}.
\end{align}

Simplifying the expression yields the final theoretical baseline:

\begin{equation}
    \label{eq:random_baseline_final}
    \mathbb{E}[|R_{ij}|] \approx \sqrt{\frac{2}{\pi}} \frac{1}{\sqrt{M}} \approx \frac{0.8}{\sqrt{M}}.
\end{equation}

\subsection{Interpretation}

This derivation leads to two key insights:
\begin{enumerate}
    \item \textbf{Effective Dimension Collapse:} In a spiked model, the ``random noise floor'' is determined not by the total dimension $N$, but by the effective dimension (number of spikes) $M$.
    \item \textbf{High Background Noise:} For a typical deep learning setting where the spectrum is sharp (e.g., $M \approx 20-50$), the random baseline is relatively high ($0.11 - 0.18$).
\end{enumerate}

Therefore, an observed coupling value significantly lower than $\frac{0.8}{\sqrt{M}}$ indicates a \textbf{statistically significant structural alignment} that effectively suppresses the geometric  ``leakage'' expected from random projections.

\section{Derivatives and Proofs}
\subsection{Proof of Lemma \ref{lem:awd_solution} \cite{fengActivityWeightDuality2023}}
\label{proof:AWD}

\begin{proof}
The constraint $(W + \Delta W)\bm{a} = W(\bm{a} + \Delta \bm{a})$ simplifies to the linear equation $\Delta W \bm{a} = W \Delta \bm{a}$.
Since the Frobenius norm $\|\Delta W\|_F^2 = \sum_{j=1}^{H_2} \|\Delta \bm{w}_j\|^2$ is separable across rows, we can solve the optimization problem independently for each row $j$. Let $\Delta \bm{w}_j^\top$ and $\bm{w}_j^\top$ denote the $j$-th rows of $\Delta W$ and $W$, respectively.
The constraint for the $j$-th output neuron becomes:
\begin{equation}
    \Delta \bm{w}_j^\top \bm{a} = \bm{w}_j^\top \Delta \bm{a}.
\end{equation}
To find the minimum norm solution for $\Delta \bm{w}_j$, we construct the Lagrangian:
\begin{equation}
    L_j = \|\Delta \bm{w}_j\|^2 - \lambda_j (\Delta \bm{w}_j^\top \bm{a} - C_j),
\end{equation}
where $C_j = \bm{w}_j^\top \Delta \bm{a}$ is a scalar constant determined by the specific row weights and activity perturbation.
Solving $\nabla_{\Delta \bm{w}_j} L_j = 0$ yields $\Delta \bm{w}_j = \frac{\lambda_j}{2} \bm{a}$. Substituting this back into the constraint gives $\frac{\lambda_j}{2} \|\bm{a}\|^2 = C_j$, which implies $\frac{\lambda_j}{2} = \frac{C_j}{\|\bm{a}\|^2}$.
Thus, the optimal row vector is:
\begin{equation}
    \Delta \bm{w}_j = \frac{\bm{w}_j^\top \Delta \bm{a}}{\|\bm{a}\|^2} \bm{a}.
\end{equation}
In element-wise notation, the $i$-th component of the $j$-th row is $\Delta W_{ji} = \frac{a_i}{\|\bm{a}\|^2} \sum_k W_{jk} \Delta a_k$. Stacking these rows back into a matrix yields the outer product form $\Delta W = \frac{(W \Delta \bm{a}) \bm{a}^\top}{\|\bm{a}\|^2}$.
\end{proof}

{
\subsection{AWD Remainder Analysis and Spectral Covariance Decomposition}
\label{app:derivations}

In this appendix, we provide the mathematical derivation of the relationship between the SGD noise covariance and the spectral properties of the per-sample Hessian. We proceed by decomposing the gradient noise via Activity-Weight Duality (AWD), controlling the explicit remainders, transforming the leading term into the eigenbasis of the global Hessian, and using the spectral decomposition of individual samples to obtain the final result.

\label{app:awd_approx}

\subsubsection{Proof of Lemma~\ref{lem:grad_approx} and Per-Sample Remainder Control}

This subsection is organized to mirror Section~\ref{sec:grad_approx}. We first prove the exact decomposition in Lemma~\ref{lem:grad_approx}, then isolate the raw AWD approximation used in the empirical diagnostics, and finally provide the sufficient norm-level and eigendirection-level remainder controls that justify the main-text leading-order discussion.

Consider a matched pair $(p,p')$ with $p \in \mathcal{B}_{\mu}$ and $p' \in \mathcal{B}_{\nu}$. Under the AWD mapping, the loss of $p'$ is represented by evaluating the loss of $p$ at the perturbed weight:
\begin{equation}
    \ell_{p'}(\bm w)
    =
    \ell_p
    \left(
        \bm w+\Delta \bm w_p^{\mu\nu}(\bm w)
    \right).
\end{equation}
Taking the total derivative with respect to $\bm w$ gives
\begin{equation}
\begin{aligned}
    \nabla \ell_{p'}(\bm w)
    &=
    \left(
        \mathbf I
        +
        \mathbf J_{\Delta,p}^{\mu\nu}
    \right)^\top
    \nabla \ell_p
    \left(
        \bm w+\Delta \bm w_p^{\mu\nu}
    \right).
\end{aligned}
\end{equation}
Applying the first-order Taylor expansion of the gradient around $\bm w$ yields
\begin{equation}
    \nabla \ell_p
    \left(
        \bm w+\Delta \bm w_p^{\mu\nu}
    \right)
    =
    \nabla \ell_p(\bm w)
    +
    \mathbf h_p(\bm w)\Delta \bm w_p^{\mu\nu}
    +
    \bm r_{T,p}^{\mu\nu},
\end{equation}
where
\begin{equation}
    \|
    \bm r_{T,p}^{\mu\nu}
    \|
    =
    \mathcal O
    \left(
        \|
        \Delta \bm w_p^{\mu\nu}
        \|^2
    \right).
\end{equation}
Substituting this expansion into the total derivative formula and subtracting $\nabla \ell_p(\bm w)$ from both sides, we obtain
\begin{equation}
\begin{aligned}
    \nabla \ell_{p'}(\bm w)
    -
    \nabla \ell_p(\bm w)
    &=
    \mathbf h_p(\bm w)\Delta \bm w_p^{\mu\nu}
    +
    (\mathbf J_{\Delta,p}^{\mu\nu})^\top
    \nabla \ell_p(\bm w)
    \\
    &\quad
    +
    (\mathbf J_{\Delta,p}^{\mu\nu})^\top
    \mathbf h_p(\bm w)\Delta \bm w_p^{\mu\nu}
    +
    (\mathbf I+\mathbf J_{\Delta,p}^{\mu\nu})^\top
    \bm r_{T,p}^{\mu\nu}.
\end{aligned}
\end{equation}
Therefore, averaging over the matched pairs in $\mathcal{B}_{\mu}$ gives
\begin{equation}
    \bm g_{\mu\nu}
    =
    \frac1B
    \sum_{p\in\mathcal B_\mu}
    \left(
        \mathbf h_p(\bm w)\Delta \bm w_p^{\mu\nu}
        +
        \bm \rho_p^{\mu\nu}
    \right),
\end{equation}
with $\bm \rho_p^{\mu\nu}$ exactly as defined in Eq.~\eqref{eq:grad_remainder}. This proves Eq.~\eqref{eq:grad_diff_decomposition}.

The Jacobian estimate in Lemma~\ref{lem:grad_approx} follows directly from:
\begin{equation}
\begin{aligned}
    \|
    \mathbf J_{\Delta,p}^{\mu\nu}
    \|
    &=
    \left\|
        \frac{
            \bm a_p \Delta \bm a_p^\top
            \otimes
            \mathbf I_{d_{\rm out}}
        }{
            \|\bm a_p\|^2
        }
    \right\|
    \\
    &=
    \frac{
        \|
        \bm a_p \Delta \bm a_p^\top
        \|
    }{
        \|\bm a_p\|^2
    }
    =
    \frac{
        \|\Delta \bm a_p\|
    }{
        \|\bm a_p\|
    },
\end{aligned}
\end{equation}
where we used $\|A \otimes I\| = \|A\|$ and $\|\bm a \bm b^\top\| = \|\bm a\|\,\|\bm b\|$ for the operator norm.

\paragraph{Auxiliary gradient decomposition for empirical diagnostics.}
For the remainder analysis below, it is convenient to separate the exact remainder into the explicit gradient-coupling correction
\begin{equation}
    \bm \rho_{p,\mathrm{gc}}^{\mu\nu}
    \triangleq
    (\mathbf J_{\Delta,p}^{\mu\nu})^\top
    \nabla \ell_p(\bm w)
\end{equation}
and the higher-order part
\begin{equation}
    \bm \rho_{p,\mathrm{ho}}^{\mu\nu}
    \triangleq
    (\mathbf J_{\Delta,p}^{\mu\nu})^\top
    \mathbf h_p(\bm w)\Delta \bm w_p^{\mu\nu}
    +
    (\mathbf I+\mathbf J_{\Delta,p}^{\mu\nu})^\top
    \bm r_{T,p}^{\mu\nu}.
\end{equation}
Retaining the explicit gradient-coupling term while isolating the higher-order correction defines the raw AWD proxy
\begin{equation}
\label{eq:grad_expansionapp}
    \bm g_{\mu\nu}^{\rm raw}
    \triangleq
    \frac1B
    \sum_{p\in\mathcal B_\mu}
    \left[
        \underbrace{
            \mathbf h_p(\bm w)\Delta \bm w_p^{\mu\nu}
        }_{\text{Term I}}
        +
        \underbrace{
            (\mathbf J_{\Delta,p}^{\mu\nu})^\top
            \nabla \ell_p(\bm w)
        }_{\text{Term II}}
    \right].
\end{equation}
The exact decomposition is therefore
\begin{equation}
    \bm g_{\mu\nu}
    =
    \bm g_{\mu\nu}^{\rm raw}
    +
    \frac1B
    \sum_{p\in\mathcal B_\mu}
    \bm \rho_{p,\mathrm{ho}}^{\mu\nu}.
\end{equation}
The corresponding raw AWD covariance proxy is defined by
\begin{equation}
\label{eq:awd_appro}
   \mathbf{C}_{\text{AWD,raw}}
   \triangleq
   \frac{1}{2}
   \mathbb{E}_{\mu, \nu}
   \left[
       \bm g_{\mu\nu}^{\rm raw}
       (\bm g_{\mu\nu}^{\rm raw})^\top
   \right]
   =
   \mathbf{C}^{hh} + \mathbf{C}^{hg} + \mathbf{C}^{gg},
\end{equation}
where $\mathbf{C}^{hh}$ is the pure Hessian-weight contribution, $\mathbf{C}^{hg}$ contains the cross terms, and $\mathbf{C}^{gg}$ is the pure gradient-coupling contribution.
These quantities play two distinct roles below. Mathematically, $\bm \rho_{p,\mathrm{gc}}^{\mu\nu}$ and $\bm \rho_{p,\mathrm{ho}}^{\mu\nu}$ separate the exact correction into a gradient-coupling part and a higher-order part. Empirically, $\mathbf{C}_{\text{AWD,raw}}$, $\mathbf{C}^{hh}$, $\mathbf{C}^{hg}$, and $\mathbf{C}^{gg}$ define the diagnostic covariance proxies used to assess the quality of the approximation.

\paragraph{Sufficient conditions for per-sample remainder bounds.}
{ The exact decomposition in Lemma~\ref{lem:grad_approx} does not require any further assumptions beyond the AWD mapping and the local Taylor expansion. The additional conditions below only justify the two local-perturbation small quantities used in Assumption~\ref{assump:awd_small} and the eigendirection-level refinement employed later in this appendix. For the analyzed FCL block, define}
\[
{
\epsilon_{\rm act}
\triangleq
\frac{\|\Delta \bm a_p\|}{\|\bm a_p\|},
\qquad
\epsilon_{\rm min}
\triangleq
\frac{\|\nabla\ell_p(\bm w)\|}
{\|\mathbf h_p(\bm w)\|\,\|\bm w\|}.}
\]
Assume the following lower-bound relations on the fitted matched pairs:
\begin{equation}
\label{eq:app_deltaw_scaling}
    \|
    \Delta \bm w_p^{\mu\nu}
    \|
    \ge
    c_{\Delta}
    \|
    \mathbf J_{\Delta,p}^{\mu\nu}
    \|
    \,
    \|
    \bm w
    \|,
\end{equation}
and
\begin{equation}
\label{eq:app_curvature_scaling}
    \|
    \mathbf h_p(\bm w)\Delta \bm w_p^{\mu\nu}
    \|
    \ge
    c_h
    \|
    \mathbf h_p(\bm w)
    \|
    \,
    \|
    \Delta \bm w_p^{\mu\nu}
    \|,
    \qquad
    c_{\Delta},c_h>0,
    \qquad
    \|
    \mathbf h_p(\bm w)
    \|
    \ge c>0.
\end{equation}
The first relation gives a lower bound on the realized AWD perturbation relative to $\|\mathbf J_{\Delta,p}^{\mu\nu}\|\,\|\bm w\|$, while the second gives a non-degenerate lower bound on the Hessian response along the AWD direction. Together they yield the estimate
\begin{equation}
\label{eq:app_eta_scaling}
{
    \eta_p
    =
    \frac{
        \|
        \bm \rho_{p,\mathrm{gc}}^{\mu\nu}
        \|
    }{
        \|
        \mathbf h_p(\bm w)\Delta \bm w_p^{\mu\nu}
        \|
    }
    \lesssim
    \frac{
        \|
        \nabla \ell_p(\bm w)
        \|
    }{
        \|
        \mathbf h_p(\bm w)
        \|
        \,
        \|
        \bm w
        \|
    }
    =
    \mathcal O(\epsilon_{\rm min}),}
\end{equation}
{Hence, in the local generic-minimum regime $\epsilon_{\rm min}\ll1$, Eq.~\eqref{eq:app_eta_scaling} gives $\eta_p\lesssim\epsilon_{\rm min}\ll1$.}
Thus Eqs.~\eqref{eq:app_deltaw_scaling}--\eqref{eq:app_eta_scaling} provide a sufficient justification for the compact near-generic-minimum assumption used in the main text.

\paragraph{Norm-level per-sample remainder control.}
{ The remainder can now be controlled term by term in the local-perturbation regime $\epsilon_{\rm act},\epsilon_{\rm min}\to0$. Under Assumption~\ref{assump:awd_small},}
\begin{equation}
{
    \frac{
        \|
        \bm \rho_{p,\mathrm{gc}}^{\mu\nu}
        \|
    }{
        \|
        \mathbf h_p(\bm w)\Delta \bm w_p^{\mu\nu}
        \|
    }
    =
    \eta_p
    =
    \mathcal O(\epsilon_{\rm min}).}
\end{equation}
For the higher-order part,
\begin{equation}
\begin{aligned}
    \frac{
        \|
        \bm \rho_{p,\mathrm{ho}}^{\mu\nu}
        \|
    }{
        \|
        \mathbf h_p(\bm w)\Delta \bm w_p^{\mu\nu}
        \|
    }
    &\le
    \|
    \mathbf J_{\Delta,p}^{\mu\nu}
    \|
    \\
    &\quad
    +
    (1+\|
    \mathbf J_{\Delta,p}^{\mu\nu}
    \|)
    \frac{
        \|
        \bm r_{T,p}^{\mu\nu}
        \|
    }{
        \|
        \mathbf h_p(\bm w)\Delta \bm w_p^{\mu\nu}
        \|
    }.
\end{aligned}
\end{equation}
{ Since $\|\mathbf J_{\Delta,p}^{\mu\nu}\| = \mathcal O(\epsilon_{\rm act})$ and}
\begin{equation}
    \frac{
        \|
        \bm r_{T,p}^{\mu\nu}
        \|
    }{
        \|
        \mathbf h_p(\bm w)\Delta \bm w_p^{\mu\nu}
        \|
    }
    =
    \mathcal O
    \left(
        \frac{
            \|\Delta \bm w_p^{\mu\nu}\|^2
        }{
            \|\mathbf h_p(\bm w)\Delta \bm w_p^{\mu\nu}\|
        }
    \right)
    =
    {\mathcal O(\epsilon_{\rm act})}
\end{equation}
by the lower-bound relation in Eq.~\eqref{eq:app_curvature_scaling}. Indeed,
\[
\|\mathbf h_p(\bm w)\Delta \bm w_p^{\mu\nu}\|
\ge
c_h\|\mathbf h_p(\bm w)\|\,\|\Delta \bm w_p^{\mu\nu}\|,
\qquad
\|\mathbf h_p(\bm w)\|\ge c>0,
\]
so
\[
\frac{\|\Delta \bm w_p^{\mu\nu}\|^2}{\|\mathbf h_p(\bm w)\Delta \bm w_p^{\mu\nu}\|}
\lesssim
\frac{\|\Delta \bm w_p^{\mu\nu}\|}{c_h\|\mathbf h_p(\bm w)\|}
\le
\frac{1}{c_h c}\|\Delta \bm w_p^{\mu\nu}\|
{
=
\mathcal O(\epsilon_{\rm act}).}
\]
Therefore we obtain
\begin{equation}
{
    \frac{
        \|
        \bm \rho_{p,\mathrm{ho}}^{\mu\nu}
        \|
    }{
        \|
        \mathbf h_p(\bm w)\Delta \bm w_p^{\mu\nu}
        \|
    }
    =
    \mathcal O(\epsilon_{\rm act}).}
\end{equation}
Combining the two parts yields
\begin{equation}
{
    \frac{
        \|
        \bm \rho_p^{\mu\nu}
        \|
    }{
        \|
        \mathbf h_p(\bm w)\Delta \bm w_p^{\mu\nu}
        \|
    }
    =
    \mathcal O(\epsilon_{\rm act}+\epsilon_{\rm min}).}
\end{equation}
{ This is the mathematical norm-level statement used in Section~\ref{sec:grad_approx}. Panel~(f) of \Cref{fig:metrics_cifar_cnn_cse,fig:metrics_mnist_fc_cse,fig:metrics_mnist_fc_mse} provides the corresponding empirical check: once the higher-order discrepancy is already negligible so that $\mathbf{C}_{\text{AWD,raw}}$ is a faithful proxy for the empirical covariance, the observed convergence of $\mathbf{C}_{\text{AWD,raw}}$ toward $\mathbf{C}^{hh}$ together with the decay of $\mathbf{C}^{hg}$ and $\mathbf{C}^{gg}$ is exactly the norm-level signature of small $\epsilon_{\rm act}$ and $\epsilon_{\rm min}$. However, panel~(f) only shows that this separation holds after aggregation in Frobenius norm. For the spectral covariance analysis in the main text, we also need to know whether it persists after resolving the covariance along the ordered global Hessian eigendirections.}

{
\paragraph{Why AWD does not require an interpolating minimum.}
With the $\epsilon_{\rm min}$ control in Eq.~\eqref{eq:app_eta_scaling} in place, we can compare AWD with a conventional Taylor route to a Hessian second-moment expression. Let $\bm w_t$ denote the SGD iterate at iteration $t$. If $\bm w^*$ is an interpolating minimum, then $\nabla\ell_p(\bm w^*)=\bm 0$ for every sample, and a local expansion gives
\begin{equation}
    \nabla\ell_p(\bm w_t)
    =
    \mathbf h_p(\bm w^*)(\bm w_t-\bm w^*)
    +
    \mathcal O(\|\bm w_t-\bm w^*\|^2).
\end{equation}
Substituting this expression into Eq.~\eqref{eq:cov} yields, schematically,
\begin{equation}
    \mathbf C
    \approx
    \frac{1}{B}
    \mathbb E_p
    \left[
        \mathbf h_p(\bm w^*)\,
        \bm\Sigma_t\,
        \mathbf h_p(\bm w^*)
    \right],
    \qquad
    \bm\Sigma_t
    \triangleq
    \mathbb E[
        (\bm w_t-\bm w^*)
        (\bm w_t-\bm w^*)^\top
    ].
\end{equation}
This route is not sufficient for generic SGD solutions. Around a generic candidate minimum $\tilde{\bm w}$, one instead has
\begin{equation}
    \nabla\ell_p(\bm w_t)
    \approx
    \nabla\ell_p(\tilde{\bm w})
    +
    \mathbf h_p(\tilde{\bm w})(\bm w_t-\tilde{\bm w}),
\end{equation}
so neglecting the zeroth-order term requires
\begin{equation}
    A_{\rm lin}
    \triangleq
    \frac{
        \|\nabla\ell_p(\tilde{\bm w})\|
    }{
        \|\mathbf h_p(\tilde{\bm w})\|\,
        \|\bm w_t-\tilde{\bm w}\|
    }
    \ll 1.
\end{equation}
This condition is restrictive because a generic minimum only satisfies $\mathbb E_p[\nabla\ell_p(\tilde{\bm w})]=\bm 0$, not $\nabla\ell_p(\tilde{\bm w})=\bm 0$ samplewise; moreover, $\|\bm w_t-\tilde{\bm w}\|$ may be very small near the minimum.
AWD avoids this obstruction by expanding the matched-sample gradient difference at the current iterate:
\begin{equation}
    \nabla\ell_{p'}(\bm w)
    -
    \nabla\ell_p(\bm w)
    =
    \mathbf h_p(\bm w)\Delta\bm w_p^{\mu\nu}
    +
    (\mathbf J_{\Delta,p}^{\mu\nu})^\top
    \nabla\ell_p(\bm w)
    +
    \mathcal O(\|\Delta\bm w_p^{\mu\nu}\|^2).
\end{equation}
The shared zeroth-order gradient cancels before any interpolation assumption is imposed. The remaining gradient-coupling correction is controlled by $\epsilon_{\rm min}$ through
\begin{equation}
    A_{\rm AWD}
    \triangleq
    \frac{
        \|(\mathbf J_{\Delta,p}^{\mu\nu})^\top\nabla\ell_p(\bm w)\|
    }{
        \|\mathbf h_p(\bm w)\Delta\bm w_p^{\mu\nu}\|
    }
    \lesssim
    \frac{
        \|\nabla\ell_p(\bm w)\|
    }{
        \|\mathbf h_p(\bm w)\|\,\|\bm w\|
    },
\end{equation}
as formalized in Eq.~\eqref{eq:app_eta_scaling}. Together with the activity perturbation control $\epsilon_{\rm act}$, this gives the $\mathcal O(\epsilon_{\rm act}+\epsilon_{\rm min})$ remainder used in Section~\ref{sec:grad_approx}. Since $\|\bm w\|$ remains finite and typically large while $\|\bm w_t-\tilde{\bm w}\|$ can be arbitrarily small near a generic minimum, $A_{\rm AWD}\ll1$ is substantially weaker than $A_{\rm lin}\ll1$. Thus AWD gives a second-moment covariance route without requiring an interpolating minimum or an a priori model for the iterate covariance $\bm\Sigma_t$.
}

\paragraph{Global Hessian eigendirection-level per-sample control.}
We now refine the remainder analysis after projection onto the global Hessian eigenbasis. Let $I_K$ denote the retained set of global Hessian eigendirections used in the power-law fit, let $\bm v_i$ be the corresponding global Hessian eigenvectors for $i\in I_K$, and define
\begin{equation}
    A_{p,i}^{\mu\nu}
    \triangleq
    \bm v_i^\top
    \mathbf h_p(\bm w)\Delta \bm w_p^{\mu\nu},
    \qquad
    B_{p,i}^{\mu\nu}
    \triangleq
    \bm v_i^\top
    \bm \rho_p^{\mu\nu}.
\end{equation}
If the projected leading response remains non-degenerate on the retained eigendirections, namely for each $i\in I_K$ and some constant $c_i>0$,
\begin{equation}
\label{eq:app_proj_nondeg}
    |
    A_{p,i}^{\mu\nu}
    |
    \ge
    c_i
    \|
    \mathbf h_p(\bm w)\Delta \bm w_p^{\mu\nu}
    \|,
\end{equation}
then projecting the exact decomposition from Eq.~\eqref{eq:grad_diff_decomposition} onto $\bm v_i$ gives
\begin{equation}
\label{eq:app_proj_exact}
    \bm v_i^\top
    \left(
        \nabla \ell_{p'}(\bm w)
        -
        \nabla \ell_p(\bm w)
    \right)
    =
    A_{p,i}^{\mu\nu}
    +
    B_{p,i}^{\mu\nu}.
\end{equation}
Moreover, by Cauchy--Schwarz,
\begin{equation}
\label{eq:app_proj_B_bound}
    |
    B_{p,i}^{\mu\nu}
    |
    =
    |
    \bm v_i^\top \bm \rho_p^{\mu\nu}
    |
    \le
    \|
    \bm v_i
    \|
    \,
    \|
    \bm \rho_p^{\mu\nu}
    \|
    =
    \|
    \bm \rho_p^{\mu\nu}
    \|,
\end{equation}
since $\|\bm v_i\|=1$. Combining Eqs.~\eqref{eq:app_proj_nondeg} and \eqref{eq:app_proj_B_bound} with the norm-level estimate yields
\begin{equation}
\label{eq:app_proj_ratio}
    \frac{
        |
        B_{p,i}^{\mu\nu}
        |
    }{
        |
        A_{p,i}^{\mu\nu}
        |
    }
    \le
    \frac{
        \|
        \bm \rho_p^{\mu\nu}
        \|
    }{
        |
        A_{p,i}^{\mu\nu}
        |
    }
    \le
    \frac{1}{c_i}
    \frac{
        \|
        \bm \rho_p^{\mu\nu}
        \|
    }{
        \|
        \mathbf h_p(\bm w)\Delta \bm w_p^{\mu\nu}
        \|
    }
    =
    {\mathcal O(\epsilon_{\rm act}+\epsilon_{\rm min}).}
\end{equation}
Equivalently,
\begin{equation}
\begin{aligned}
\label{eq:app_proj_refined}
    \bm v_i^\top
    \left(
        \nabla \ell_{p'}(\bm w)
        -
        \nabla \ell_p(\bm w)
    \right)
    &=
    A_{p,i}^{\mu\nu}
    +
    {\mathcal O}
    \left(
        {(\epsilon_{\rm act}+\epsilon_{\rm min})}
        |
        A_{p,i}^{\mu\nu}
        |
    \right) \\
    &=
    \bm v_i^\top
    \mathbf h_p(\bm w)\Delta \bm w_p^{\mu\nu}
    +
    {\mathcal O}
    \left(
        {(\epsilon_{\rm act}+\epsilon_{\rm min})}
        |
        \bm v_i^\top
        \mathbf h_p(\bm w)\Delta \bm w_p^{\mu\nu}
        |
    \right).
\end{aligned}
\end{equation}
This is the component-level remainder control promised in Section~\ref{sec:grad_approx}. It is stronger than the norm-level sketch and is the version used implicitly when the covariance is analyzed entrywise in the global Hessian eigenbasis below.
Panel~(g) of \Cref{fig:metrics_cifar_cnn_cse,fig:metrics_mnist_fc_cse,fig:metrics_mnist_fc_mse} gives the corresponding empirical refinement: instead of comparing only global norms, it resolves the covariance along the ordered global Hessian eigendirections and compares the diagonal magnitudes of $\mathbf{C}^{hh}$, $\mathbf{C}^{hg}$, and $\mathbf{C}^{gg}$ relative to $\mathbf{C}_{\text{AWD,raw}}$. This is consistent with the eigendirection-level separation observed in \Cref{fig:metrics_cifar_cnn_cse,fig:metrics_mnist_fc_cse,fig:metrics_mnist_fc_mse}~(g).

\paragraph{Batch-level decomposition for covariance analysis.}
Define the batch-level leading and remainder parts
\begin{equation}
\label{eq:app_s_r_def}
    \bm s_{\mu\nu}
    \triangleq
    \sum_{p\in\mathcal B_\mu}
    \mathbf h_p(\bm w)\Delta \bm w_p^{\mu\nu},
    \qquad
    \bm r_{\mu\nu}
    \triangleq
    \sum_{p\in\mathcal B_\mu}
    \bm\rho_p^{\mu\nu}.
\end{equation}
Then Eq.~\eqref{eq:grad_diff_decomposition} gives the exact batch decomposition
\begin{equation}
\label{eq:app_g_exact}
    \bm g_{\mu\nu}
    =
    \frac{1}{B}
    \left(
        \bm s_{\mu\nu}
        +
        \bm r_{\mu\nu}
    \right).
\end{equation}
{ For the covariance analysis, the key question is whether the accumulated remainder $\bm r_{\mu\nu}$ remains controlled after the batch sum. This does not follow automatically from the per-sample estimate on $\bm \rho_p^{\mu\nu}$, because cancellation in $\bm s_{\mu\nu}$ can change the relative size after aggregation. We therefore impose the batch-level control stated later in Assumption~\ref{assump:batch_rem_control}, namely $\|\bm r_{\mu\nu}\|=\mathcal O((\epsilon_{\rm act}+\epsilon_{\rm min})\|\bm s_{\mu\nu}\|)$ together with its component-level counterpart on the retained eigendirections.}
Accordingly, the covariance admits the exact split
\begin{equation}
\label{eq:app_C_exact}
    \mathbf C
    =
    \mathbf C^{hh}
    +
    \mathbf R^{\rm rem},
    \qquad
    \mathbf C^{hh}
    \triangleq
    \frac{1}{2B^2}
    \mathbb E_{\mu,\nu}
    \left[
        \bm s_{\mu\nu}\bm s_{\mu\nu}^\top
    \right],
\end{equation}
with
\begin{equation}
\label{eq:app_Rrem_exact}
    \mathbf R^{\rm rem}
    \triangleq
    \frac{1}{2B^2}
    \mathbb E_{\mu,\nu}
    \left[
        \bm s_{\mu\nu}\bm r_{\mu\nu}^\top
        +
        \bm r_{\mu\nu}\bm s_{\mu\nu}^\top
        +
        \bm r_{\mu\nu}\bm r_{\mu\nu}^\top
    \right].
\end{equation}
The next two subsections analyze $\mathbf C^{hh}$ in the global Hessian eigenbasis and then control $\mathbf R^{\rm rem}$ at both the norm and component levels.

\subsubsection{Projection onto the Global Hessian Eigenbasis}
\label{app:coord_trans}

To analyze the spectral properties of the noise, we move to the eigenbasis of the full-batch Hessian $\mathbf{H} = \mathbb{E}_p [\mathbf{h}_p]$. Let $\mathbf{H} = \mathbf{V} \mathbf{\Lambda} \mathbf{V}^\top$ be the eigendecomposition in the descending order, where $\mathbf{V} = [\bm{v}_1, \dots, \bm{v}_D]$ is the orthogonal matrix of eigenvectors. We  seek to calculate the diagonal elements of the covariance matrix in this basis, defined as $C_{ii}  = \bm{v}_i^\top \mathbf{C} \bm{v}_i$.

\subsubsection{Covariance Decomposition and Sufficient Conditions for Assumption~\ref{assump:spectral_cov_conditions}}
\label{app:proof_thm1}

We first derive the exact covariance decomposition stated in Lemma~\ref{lem:cov_decomp}. Isolating the leading covariance term $\mathbf C^{hh}$ from the exact split in Eq.~\eqref{eq:app_C_exact}, for any global Hessian eigenvectors $\bm v_i$ and $\bm v_j$ we have
\begin{equation}
\label{eq:app_Chh_ij_start}
    C^{hh}_{ij}
    =
    \frac{1}{2B^2}
    \mathbb E_{\mu,\nu}
    \left[
        \left(
            \bm v_i^\top \bm s_{\mu\nu}
        \right)
        \left(
            \bm v_j^\top \bm s_{\mu\nu}
        \right)
    \right].
\end{equation}
Expanding $\bm s_{\mu\nu}=\sum_{p\in\mathcal B_\mu}\mathbf h_p(\bm w)\Delta \bm w_p^{\mu\nu}$ gives
\begin{equation}
\label{eq:app_Chh_ij_double}
\begin{aligned}
    C^{hh}_{ij}
    &=
    \frac{1}{2B^2}
    \mathbb E_{\mu}
    \left[
        \sum_{p,q\in\mathcal B_\mu}
        \mathbb E_{\nu}
        \left[
            \left(
                \bm v_i^\top
                \mathbf h_p(\bm w)\Delta \bm w_p^{\mu\nu}
            \right)
            \left(
                (\Delta \bm w_q^{\mu\nu})^\top
                \mathbf h_q(\bm w)\bm v_j
            \right)
        \right]
    \right].
\end{aligned}
\end{equation}
Separate the same-sample and cross-sample pieces:
\begin{equation}
\label{eq:app_Chh_sample_split}
    C^{hh}_{ij}
    =
    \widetilde D_{ij}
    +
    R^{\rm sample}_{ij},
\end{equation}
where
\begin{equation}
\label{eq:app_Rsample_def}
\begin{aligned}
    R^{\rm sample}_{ij}
    \triangleq
    \frac{1}{2B^2}
    \mathbb E_{\mu}
    \left[
        \sum_{\substack{p,q\in\mathcal B_\mu\\ p\neq q}}
        \mathbb E_{\nu}
        \left[
            \left(
                \bm v_i^\top
                \mathbf h_p(\bm w)\Delta \bm w_p^{\mu\nu}
            \right)
            \left(
                (\Delta \bm w_q^{\mu\nu})^\top
                \mathbf h_q(\bm w)\bm v_j
            \right)
        \right]
    \right]
\end{aligned}
\end{equation}
collects the cross-sample covariance contribution, and
\begin{equation}
\label{eq:app_Dtilde_def}
\begin{aligned}
    \widetilde D_{ij}
    \triangleq
    \frac{1}{2B^2}
    \mathbb E_{\mu}
    \left[
        \sum_{p\in\mathcal B_\mu}
        \mathbb E_{\nu}
        \left[
            \bm v_i^\top
            \mathbf h_p(\bm w)
            \Delta \bm w_p^{\mu\nu}
            (\Delta \bm w_p^{\mu\nu})^\top
            \mathbf h_p(\bm w)
            \bm v_j
        \right]
    \right].
\end{aligned}
\end{equation}
To analyze $\widetilde D_{ij}$, expand the per-sample Hessian in its local eigenbasis,
\[
    \mathbf h_p(\bm w)
    =
    \sum_m
    \kappa_m^{(p)}
    \bm u_m^{(p)}
    (\bm u_m^{(p)})^\top,
\]
and define the local covariance entries
\begin{equation}
\label{eq:app_Mpmn_def}
    \mathcal M_{p,mn}
    \triangleq
    \mathbb E_{\nu}
    \left[
        (\bm u_m^{(p)})^\top
        \Delta \bm w_p^{\mu\nu}
        (\Delta \bm w_p^{\mu\nu})^\top
        \bm u_n^{(p)}
    \right].
\end{equation}
Then Eq.~\eqref{eq:app_Dtilde_def} becomes
\begin{equation}
\label{eq:app_Dtilde_local}
\begin{aligned}
    \widetilde D_{ij}
    &=
    \frac{1}{2B^2}
    \mathbb E_{\mu}
    \left[
        \sum_{p\in\mathcal B_\mu}
        \sum_{m,n}
        \kappa_m^{(p)}
        \kappa_n^{(p)}
        (\bm u_m^{(p)}\!\cdot\!\bm v_i)
        (\bm u_n^{(p)}\!\cdot\!\bm v_j)
        \mathcal M_{p,mn}
    \right].
\end{aligned}
\end{equation}
Now write the local covariance as
\begin{equation}
\label{eq:app_M_split}
    \mathcal M_{p,mn}
    =
    \sigma_w^2\delta_{mn}
    +
    \Delta \mathcal M_{p,mn}.
\end{equation}
Substituting Eq.~\eqref{eq:app_M_split} into Eq.~\eqref{eq:app_Dtilde_local} yields
\begin{equation}
\label{eq:app_Dtilde_iso_split}
    \widetilde D_{ij}
    =
    D_{ij}
    +
    R^{\rm iso}_{ij},
\end{equation}
where
\begin{equation}
\label{eq:app_Dij_def}
    D_{ij}
    \triangleq
    \frac{\sigma_w^2}{2B}
    \mathbb E_{p}
    \left[
        \sum_m
        (\kappa_m^{(p)})^2
        (\bm u_m^{(p)}\!\cdot\!\bm v_i)
        (\bm u_m^{(p)}\!\cdot\!\bm v_j)
    \right]
\end{equation}
is the leading spectral term in Theorem~\ref{thm:spectral_noise}, and
\begin{equation}
\label{eq:app_Riso_def}
\begin{aligned}
    R^{\rm iso}_{ij}
    \triangleq
    \frac{1}{2B^2}
    \mathbb E_{\mu}
    \left[
        \sum_{p\in\mathcal B_\mu}
        \sum_{m,n}
        \kappa_m^{(p)}
        \kappa_n^{(p)}
        (\bm u_m^{(p)}\!\cdot\!\bm v_i)
        (\bm u_n^{(p)}\!\cdot\!\bm v_j)
        \Delta \mathcal M_{p,mn}
    \right]
\end{aligned}
\end{equation}
measures the deviation from local isotropy. Hence the leading covariance term admits the exact decomposition
\begin{equation}
\label{eq:app_Chh_exact_decomp}
    C^{hh}_{ij}
    =
    D_{ij}
    +
    R^{\rm sample}_{ij}
    +
    R^{\rm iso}_{ij}.
\end{equation}
Combining Eq.~\eqref{eq:app_Chh_exact_decomp} with Eq.~\eqref{eq:app_C_exact} gives the exact component-level covariance decomposition
\begin{equation}
\label{eq:app_Cij_final_decomp}
    C_{ij}
    =
    D_{ij}
    +
    R^{\rm sample}_{ij}
    +
    R^{\rm iso}_{ij}
    +
    R^{\rm rem}_{ij},
\end{equation}
where $R^{\rm rem}_{ij}\triangleq \bm v_i^\top \mathbf R^{\rm rem}\bm v_j$. This proves the decomposition used in Lemma~\ref{lem:cov_decomp}.
{ The next three assumptions provide appendix-level sufficient conditions with two distinct roles. The batch-level AWD remainder control propagates the local-perturbation small quantities $\epsilon_{\rm act}$ and $\epsilon_{\rm min}$ to $R^{\rm rem}$. The sample-independence and local-isotropy control gives the diagonal statistical contribution summarized by $\epsilon_{\rm stat}$ in Assumption~\ref{assump:spectral_cov_conditions}. The off-diagonal projection cancellation is an auxiliary estimate supporting the approximate-diagonalization discussion in Section~\ref{sec:spectral_analysis}, rather than an additional term in the displayed diagonal theorem.}
\begin{assumption}[Sample-independence and local-isotropy]
\label{assump:sample_iso}
{ On the retained eigendirections in the late-training regime, define the statistical covariance ratio}
\[
{
    \epsilon_{\rm stat}
    \triangleq
    \max\left\{
    \max_{i\in I_K}
    \frac{|R^{\rm sample}_{ii}|+|R^{\rm iso}_{ii}|}{|D_{ii}|},
    \max_{\substack{i,j\in I_K\\ i\neq j}}
    \frac{|R^{\rm sample}_{ij}|+|R^{\rm iso}_{ij}|}{|D_{ii}|+|D_{jj}|}
    \right\}
    \ll 1.}
\]
{ This ratio summarizes the cross-sample covariance contribution and the deviation from local isotropy.}
\end{assumption}

To state the batch-level AWD remainder condition, let $\mathcal R$ denote the matched batch pairs, and let $I_K$ denote the retained set of global Hessian eigendirections. For component-level entries, define
\begin{equation}
\label{eq:app_XY_def}
    X_i^{\mu\nu}
    \triangleq
    \bm v_i^\top \bm s_{\mu\nu},
    \qquad
    Y_i^{\mu\nu}
    \triangleq
    \bm v_i^\top \bm r_{\mu\nu}.
\end{equation}
\begin{assumption}[Batch-level AWD remainder control]
\label{assump:batch_rem_control}
{ For the matched batch pairs $(\mu,\nu)\in\mathcal R$ and the retained global Hessian eigendirections $i\in I_K$, the accumulated AWD remainder is controlled by the local-perturbation small quantities relative to the accumulated leading response:}
\begin{equation}
\label{eq:app_batch_norm_control}
{
    \|
    \bm r_{\mu\nu}
    \|
    \le
    \xi_{\mu\nu}
    \|
    \bm s_{\mu\nu}
    \|,
    \qquad
    \sup_{(\mu,\nu)\in\mathcal R}
    \xi_{\mu\nu}
    =
    \mathcal O(\epsilon_{\rm act}+\epsilon_{\rm min}),}
\end{equation}
{ and, componentwise in the global Hessian eigenbasis,}
\begin{equation}
\label{eq:app_batch_component_control}
{
    |Y_i^{\mu\nu}|
    \le
    \zeta_i^{\mu\nu}|X_i^{\mu\nu}|,
    \qquad
    \sup_{\substack{i\in I_K\\(\mu,\nu)\in\mathcal R}}
    \zeta_i^{\mu\nu}
    =
    \mathcal O(\epsilon_{\rm act}+\epsilon_{\rm min}).}
\end{equation}
\end{assumption}

To state the off-diagonal projection condition, define the local energy projections
\[
    A_m^{(p,i)}
    \triangleq
    \kappa_m^{(p)}
    (\bm u_m^{(p)}\!\cdot\!\bm v_i),
    \qquad
    B_m^{(p,j)}
    \triangleq
    \kappa_m^{(p)}
    (\bm u_m^{(p)}\!\cdot\!\bm v_j),
    \qquad i\neq j,
\]
and the per-sample off-diagonal sum
\[
    S_p^{ij}
    \triangleq
    \sum_m
    A_m^{(p,i)}B_m^{(p,j)}.
\]
Then Eq.~\eqref{eq:app_Dij_def} gives
\begin{equation}
\label{eq:app_Dij_sample_sum}
    D_{ij}
    =
    \frac{\sigma_w^2}{2B}
    \mathbb E_p[S_p^{ij}],
    \qquad
    D_{ii}
    =
    \frac{\sigma_w^2}{2B}
    \mathbb E_p
    \left[
        \sum_m (A_m^{(p,i)})^2
    \right],
    \qquad
    D_{jj}
    =
    \frac{\sigma_w^2}{2B}
    \mathbb E_p
    \left[
        \sum_m (B_m^{(p,j)})^2
    \right].
\end{equation}
\begin{assumption}[Off-diagonal projection decoupling]
\label{assump:offdiag_decouple}
{ For the retained off-diagonal pairs $(i,j)$ with $i\neq j$, let $N_{\rm samp}$ denote the number of samples in the empirical average over $p$, and set $\epsilon_{\rm off}\triangleq N_{\rm samp}^{-1/2}$. Since $S_p^{ij}$ is the product of two projections onto orthogonal global Hessian eigendirections, we assume that its signed sample average is cancellation-dominated in the high-dimensional regime: after averaging over retained samples, its mean is close to zero up to a CLT-type fluctuation of order $\epsilon_{\rm off}$, and is therefore small relative to the corresponding diagonal scale. Concretely, there exists a bounded non-negative quantity $\varepsilon_{ij}^{\rm dec}=\mathcal O(1)$ such that}
\begin{equation}
\label{eq:app_offdiag_mean_bound}
    \left|
        \mathbb E_p[S_p^{ij}]
    \right|
    \le
    \frac{\sqrt{1+\varepsilon_{ij}^{\rm dec}}}{\sqrt{N_{\rm samp}}}
    \sqrt{
        \sum_m
        \mathbb E_p[(A_m^{(p,i)})^2]\,
        \mathbb E_p[(B_m^{(p,j)})^2]
    }.
\end{equation}
Equivalently, the leading off-diagonal term satisfies
\[
    \frac{|D_{ij}|}{\sqrt{D_{ii}D_{jj}}}
    =
    O\!\left(\frac{1}{\sqrt{N_{\rm samp}}}\right).
\]
\end{assumption}

{ For the diagonal theorem in the current main text, the diagonal part of Assumption~\ref{assump:sample_iso} supplies the statistical contribution $\epsilon_{\rm stat}$. Assumption~\ref{assump:offdiag_decouple} instead provides the auxiliary off-diagonal cancellation estimate used to justify approximate diagonalization in the surrounding spectral-structure discussion.}

\begin{proof}[Proof of Theorem~\ref{thm:spectral_noise}]
We first control the covariance error generated by the accumulated AWD remainder $\bm r_{\mu\nu}$. Assumption~\ref{assump:batch_rem_control} is the batch-level version of the preceding per-sample remainder estimates; it explicitly rules out cancellation of the leading batch response that would otherwise hide the leading term in the denominator. From Eq.~\eqref{eq:app_g_exact},
\begin{equation}
\label{eq:app_cov_batch_error}
    \bm g_{\mu\nu}\bm g_{\mu\nu}^\top
    -
    \frac{1}{B^2}
    \bm s_{\mu\nu}\bm s_{\mu\nu}^\top
    =
    \frac{1}{B^2}
    \left[
        \bm s_{\mu\nu}\bm r_{\mu\nu}^\top
        +
        \bm r_{\mu\nu}\bm s_{\mu\nu}^\top
        +
        \bm r_{\mu\nu}\bm r_{\mu\nu}^\top
    \right].
\end{equation}
Then, using $\|xy^\top\|_F=\|x\|\,\|y\|$,
\begin{equation}
\label{eq:app_Rrem_norm_bound}
    \left\|
        \bm g_{\mu\nu}\bm g_{\mu\nu}^\top
        -
        \frac{1}{B^2}
        \bm s_{\mu\nu}\bm s_{\mu\nu}^\top
    \right\|_F
    \le
    \frac{1}{B^2}
    \left(
        2\xi_{\mu\nu}
        +
        \xi_{\mu\nu}^2
    \right)
    \|\bm s_{\mu\nu}\|^2.
\end{equation}
Taking expectation in Eq.~\eqref{eq:app_Rrem_norm_bound} and using Assumption~\ref{assump:batch_rem_control} gives
\begin{equation}
{
    \|
    \mathbf R^{\rm rem}
    \|_F
    \le
    \mathcal O(\epsilon_{\rm act}+\epsilon_{\rm min})
    \frac{1}{2B^2}
    \mathbb E_{\mu,\nu}
    \|
    \bm s_{\mu\nu}
    \|^2
}
\end{equation}
Here
\[
    \frac{1}{2B^2}
    \mathbb E_{\mu,\nu}
    \|
    \bm s_{\mu\nu}
    \|^2
    =
    \operatorname{tr}(\mathbf C^{hh}),
\]
since $\mathbf C^{hh}=\frac{1}{2B^2}\mathbb E_{\mu,\nu}[\bm s_{\mu\nu}\bm s_{\mu\nu}^\top]$ and $\operatorname{tr}(xx^\top)=\|x\|^2$ for any vector $x$. Thus Eq.~\eqref{eq:app_Rrem_norm_bound} is the batchwise norm-level covariance error estimate for the contribution of $\bm r_{\mu\nu}$. For the component-level refinement, Eq.~\eqref{eq:app_XY_def} gives
\begin{equation}
\label{eq:app_Cij_exact_split}
    C_{ij}
    =
    \frac{1}{2B^2}
    \mathbb E_{\mu,\nu}
    \left[
        (X_i^{\mu\nu}+Y_i^{\mu\nu})
        (X_j^{\mu\nu}+Y_j^{\mu\nu})
    \right]
    =
    C^{hh}_{ij}
    +
    R^{\rm rem}_{ij},
\end{equation}
with
\begin{equation}
\label{eq:app_Rrem_ij_def}
    R^{\rm rem}_{ij}
    =
    \frac{1}{2B^2}
    \mathbb E_{\mu,\nu}
    \left[
        X_i^{\mu\nu}Y_j^{\mu\nu}
        +
        Y_i^{\mu\nu}X_j^{\mu\nu}
        +
        Y_i^{\mu\nu}Y_j^{\mu\nu}
    \right].
\end{equation}
By the component part of Assumption~\ref{assump:batch_rem_control},
\begin{equation}
\label{eq:app_Rrem_ij_bound}
    |R^{\rm rem}_{ij}|
    \le
    \frac{1}{2B^2}
    \mathbb E_{\mu,\nu}
    \left[
        \left(
            \zeta_i^{\mu\nu}
            +
            \zeta_j^{\mu\nu}
            +
            \zeta_i^{\mu\nu}\zeta_j^{\mu\nu}
        \right)
        \left|
            X_i^{\mu\nu}X_j^{\mu\nu}
        \right|
    \right].
\end{equation}
In particular, setting $j=i$ gives
\begin{equation}
\label{eq:app_Cii_exact_split}
    C_{ii}
    =
    C^{hh}_{ii}
    +
    R^{\rm rem}_{ii},
\end{equation}
with
\begin{equation}
\label{eq:app_Rrem_ii_bound}
    |R^{\rm rem}_{ii}|
    \le
    \frac{1}{2B^2}
    \mathbb E_{\mu,\nu}
    \left[
        \left(
            2\zeta_i^{\mu\nu}
            +
            (\zeta_i^{\mu\nu})^2
        \right)
        (X_i^{\mu\nu})^2
    \right].
\end{equation}
These estimates control the explicit batch-remainder contribution in Eq.~\eqref{eq:app_Cij_final_decomp}. For general off-diagonal entries, the relevant question is whether they are negligible relative to the diagonal scale on the retained eigendirections. Under Assumption~\ref{assump:sample_iso}, $R^{\rm sample}_{ij}$ and $R^{\rm iso}_{ij}$ are already lower-order, so it remains to control $D_{ij}$ and $R^{\rm rem}_{ij}$. Starting from Eq.~\eqref{eq:app_Rrem_ij_bound}, use the elementary bound
\[
    |X_i^{\mu\nu}X_j^{\mu\nu}|
    \le
    \frac{1}{2}
    \left[
        (X_i^{\mu\nu})^2
        +
        (X_j^{\mu\nu})^2
    \right].
\]
Let
\[
    \zeta_K
    \triangleq
    \sup_{\substack{k\in I_K\\(\mu,\nu)\in\mathcal R}}
    \zeta_k^{\mu\nu}.
\]
{ By Assumption~\ref{assump:batch_rem_control}, $\zeta_K=\mathcal O(\epsilon_{\rm act}+\epsilon_{\rm min})$. Hence Eq.~\eqref{eq:app_Rrem_ij_bound} gives}
\[
    |R^{\rm rem}_{ij}|
    \le
    \left(
        2\zeta_K+\zeta_K^2
    \right)
    \frac{1}{2B^2}
    \mathbb E_{\mu,\nu}
    |X_i^{\mu\nu}X_j^{\mu\nu}|
    \le
    \frac{
        2\zeta_K+\zeta_K^2
    }{2}
    \left(
        C^{hh}_{ii}
        +
        C^{hh}_{jj}
    \right),
\]
and therefore
\[
{
    |R^{\rm rem}_{ij}|
    =
    \mathcal O\!\left(
        (\epsilon_{\rm act}+\epsilon_{\rm min})
        \left(
        C^{hh}_{ii}
        +
        C^{hh}_{jj}
        \right)
    \right).
}
\]
For the leading off-diagonal term, Assumption~\ref{assump:offdiag_decouple} gives the CLT-type estimate
\begin{equation}
\label{eq:app_Dij_mean_bound}
    |D_{ij}|
    \le
    \frac{\sigma_w^2}{2B}
    \cdot
    \frac{\sqrt{1+\varepsilon_{ij}^{\rm dec}}}{\sqrt{N_{\rm samp}}}
    \sqrt{
        \sum_m
        \mathbb E_p[(A_m^{(p,i)})^2]\,
        \mathbb E_p[(B_m^{(p,j)})^2]
    }.
\end{equation}
Dividing by the diagonal scale and using
\[
    \sum_m a_m b_m
    \le
    \left(\sum_m a_m\right)
    \left(\sum_m b_m\right),
    \qquad a_m,b_m\ge 0,
\]
gives
\begin{equation}
\label{eq:app_Dij_ratio_bound}
    \frac{
        |D_{ij}|
    }{
        \sqrt{D_{ii}D_{jj}}
    }
    \lesssim
    \frac{\sqrt{1+\varepsilon_{ij}^{\rm dec}}}{\sqrt{N_{\rm samp}}}.
\end{equation}
Finally, since $\sqrt{D_{ii}D_{jj}}\le \tfrac12(D_{ii}+D_{jj})$, Eq.~\eqref{eq:app_Dij_ratio_bound} implies the off-diagonal mean estimate
\begin{equation}
\label{eq:app_Dij_diag_scale}
    |D_{ij}|
    =
    O\!\left(
        \frac{D_{ii}+D_{jj}}{\sqrt{N_{\rm samp}}}
    \right),
\end{equation}
{ so the leading off-diagonal term is $\mathcal O(\epsilon_{\rm off}(D_{ii}+D_{jj}))$ and is small relative to the diagonal scale when $N_{\rm samp}$ is large. Since the diagonal part of Assumption~\ref{assump:sample_iso} gives $C^{hh}_{kk}=D_{kk}[1+\mathcal O(\epsilon_{\rm stat})]$ for retained $k$, this yields}
\[
{
    |D_{ij}|
    =
    \mathcal O\!\left(\epsilon_{\rm off}(C^{hh}_{ii}+C^{hh}_{jj})\right).}
\]
By Assumption~\ref{assump:sample_iso}, the other two off-diagonal error terms satisfy
\[
{
    |R^{\rm sample}_{ij}|+|R^{\rm iso}_{ij}|
    \le
    \epsilon_{\rm stat}(|D_{ii}|+|D_{jj}|)
    =
    \mathcal O\!\left(\epsilon_{\rm stat}(C^{hh}_{ii}+C^{hh}_{jj})\right).}
\]
Combining these bounds with Eq.~\eqref{eq:app_Cij_final_decomp} yields
\[
{
    |C_{ij}|
    =
    \mathcal O\!\left(
        (\epsilon_{\rm act}+\epsilon_{\rm min}+\epsilon_{\rm stat}+\epsilon_{\rm off})
        (C^{hh}_{ii}+C^{hh}_{jj})
    \right).}
\]
This is the component-level form of approximate commutativity used here: the off-diagonal covariance is negligible relative to the corresponding diagonal scale.
For the diagonal entries used in the power-law fit, Eq.~\eqref{eq:app_Cij_final_decomp} specializes to
\begin{equation}
\label{eq:app_Cii_final_decomp}
    C_{ii}
    =
    D_{ii}
    +
    R^{\rm sample}_{ii}
    +
    R^{\rm iso}_{ii}
    +
    R^{\rm rem}_{ii}.
\end{equation}
For diagonal entries, no additional smallness assumption on the leading term itself is needed. Assumption~\ref{assump:sample_iso} already gives
\[
{
    |R^{\rm sample}_{ii}|+|R^{\rm iso}_{ii}|
    \le
    \epsilon_{\rm stat}|D_{ii}|,}
\]
{ so $C^{hh}_{ii}=D_{ii}[1+\mathcal O(\epsilon_{\rm stat})]$. Meanwhile, Eq.~\eqref{eq:app_Rrem_ii_bound} and Assumption~\ref{assump:batch_rem_control} give}
\[
{ 
    |R^{\rm rem}_{ii}|
    \le
    \left(
        2\zeta_K+\zeta_K^2
    \right)
    C^{hh}_{ii}
    =
    \mathcal O\!\left((\epsilon_{\rm act}+\epsilon_{\rm min})C^{hh}_{ii}\right).}
\]
Consequently,
\[
{ 
    |R^{\rm rem}_{ii}|
    =
    \mathcal O\!\left((\epsilon_{\rm act}+\epsilon_{\rm min})D_{ii}\right).}
\]
Therefore
\[
{ 
    C_{ii}
    =
    D_{ii}\left[1+\mathcal O(\epsilon_{\rm act}+\epsilon_{\rm min}+\epsilon_{\rm stat})\right].}
\]
{  Thus the displayed formula in Theorem~\ref{thm:spectral_noise} is the leading diagonal term $D_{ii}$ obtained after controlling the statistical contributions $R^{\rm sample}_{ii}$ and $R^{\rm iso}_{ii}$ by $\epsilon_{\rm stat}$ and the AWD-remainder contribution $R^{\rm rem}_{ii}$ by the local-perturbation small quantities. The off-diagonal estimates derived above should be read as auxiliary approximate-diagonalization controls supporting the spectral-structure discussion and empirical commutativity diagnostics.}
\end{proof}
\paragraph{Empirical validation of the sufficient conditions.}
Panel~(h) of \Cref{fig:metrics_cifar_cnn_cse,fig:metrics_mnist_fc_cse,fig:metrics_mnist_fc_mse} supports Assumption~\ref{assump:sample_iso}: in the late-training regime, $\mathbf C^{hh}$ closely tracks $\mathbf C^{hh,SD}$, where $\mathbf C^{hh,SD}$ is obtained by dropping the cross-sample terms, and $\mathbf C^{hh,SD,WD}$ also closely tracks $\mathbf C^{hh,SD}$, where $\mathbf C^{hh,SD,WD}$ is obtained by further replacing $\mathcal M_{p,mn}$ with $\sigma_{w,pm}^2\delta_{mn}$. This indicates that both the cross-sample correction and the off-diagonal local-covariance correction are small in this regime. The off-diagonal conclusion is also consistent with the near-diagonality of $C$ in the global Hessian eigenbasis shown in \cref{fig:cmt_cifar_cnn_cse,fig:cmt_mnist_fc_cse,fig:cmt_mnist_fc_mse} (last row).
}

\subsection{Proof of Theorem \ref{thm:bounds}}
\label{app:bounds}

In this appendix, we provide the detailed proof for Theorem \ref{thm:bounds}. We first define the necessary auxiliary variables, derive the moment inequalities, and finally show how these inequalities constrain the power-law exponent $\gamma$ 

\subsubsection{Definitions and Auxiliary Variables}

Consider the projection of the gradient noise and Hessian onto a specific global eigen-direction $\bm{v}_i$ (the $i$-th eigenvector of $H$). We define two auxiliary random variables, $A_i^{(p)}$ and $B_i^{(p)}$, representing the first and second moments of the projected curvature for a specific sample $p$:

\begin{equation}
    A_i^{(p)} := \sum_{m} \kappa_m^{(p)} (\bm{u}_m^{(p)} \cdot \bm{v}_i)^2, \quad 
    B_i^{(p)} := \sum_{m} (\kappa_m^{(p)})^2 (\bm{u}_m^{(p)} \cdot \bm{v}_i)^2.
\end{equation}

Here, $\kappa_m^{(p)}$ and $\bm{u}_m^{(p)}$ are the eigenvalues and eigenvectors of the per-sample Hessian $\mathbf{h}_p$. By definition, the global Hessian eigenvalue $H_{ii}$ and the noise covariance diagonal $C_{ii}$ are expectations over the dataset distribution $\mathcal{D}$ (indexed by $p$):
\begin{equation}
    H_{ii} = \mathbb{E}_{p \sim \mathcal{D}} [A_i^{(p)}], \quad 
    C_{ii} = \frac{\sigma_w^2}{2B}\mathbb{E}_{p \sim \mathcal{D}} [B_i^{(p)}].
\end{equation}
Note that for $C_{ii}$, we can omit the constant factor $\sigma_w^2/2B$ for clarity as it affects the magnitude constant $c_c$ but not the scaling exponent $\gamma$.

\subsubsection{Derivation of Inequalities}
\begin{proof}

\textbf{1. The Lower Bound ($H_{ii}^2 \le \frac{2B}{\sigma_w^2}C_{ii}$):}

We first apply the Cauchy-Schwarz inequality to the sum defining $A_i^{(p)}$. Let $x_m = \kappa_m^{(p)} (\bm{u}_m^{(p)} \cdot \bm{v}_i)$ and $y_m = (\bm{u}_m^{(p)} \cdot \bm{v}_i)$.  Since $\sum_m (\bm{u}_m^{(p)} \cdot \bm{v}_i)^2 = \|\bm{v}_i\|^2 = 1$, the term $A_i^{(p)}$ can be viewed as a convex combination of $\kappa_m^{(p)}$.
Directly comparing $A_i^{(p)}$ and $B_i^{(p)}$:
\begin{equation}\label{CS}
    (A_i^{(p)})^2 = \left( \sum_m \kappa_m^{(p)} (\bm{u}_m^{(p)} \cdot \bm{v}_i)^2 \right)^2 \le \left(\sum_m (\kappa_m^{(p)})^2 (\bm{u}_m^{(p)} \cdot \bm{v}_i)^2 \right) \left( \sum_m (\bm{u}_m^{(p)} \cdot \bm{v}_i)^2 \right).
\end{equation}
Using $\sum_m (\bm{u}_m^{(p)} \cdot \bm{v}_i)^2 = 1$, we obtain the sample-wise inequality:
\begin{equation}
    (A_i^{(p)})^2 \le B_i^{(p)}.
\end{equation}
Next, we take the expectation over samples $p$. Using Jensen's inequality for the convex function $g(x)=x^2$, we have $(\mathbb{E}[X])^2 \le \mathbb{E}[X^2]$. Thus:
\begin{equation}
    H_{ii}^2 = (\mathbb{E}_p [A_i^{(p)}])^2 \le \mathbb{E}_p [(A_i^{(p)})^2] \le \mathbb{E}_p [B_i^{(p)}] = \frac{2B}{\sigma_w^2}C_{ii}.
\end{equation}
This proves the lower bound inequality.

\textbf{2. The Upper Bound ($ \frac{2B}{\sigma_w^2} C_{ii} \le \kappa_{\max} H_{ii}$):}

Assume the per-sample Hessian eigenvalues are bounded by $\kappa_{\max}$, i.e., $0 \le \kappa_m^{(p)} \le \kappa_{\max}$ for all $m, p$.
We can bound the second moment term $B_i^{(p)}$:
\begin{equation}
    B_i^{(p)} = \sum_{m} \kappa_m^{(p)} \cdot \kappa_m^{(p)} (\bm{u}_m^{(p)} \cdot \bm{v}_i)^2 \le \kappa_{\max} \sum_{m} \kappa_m^{(p)} (\bm{u}_m^{(p)} \cdot \bm{v}_i)^2 = \kappa_{\max} A_i^{(p)}.
\end{equation}
Taking the expectation over $p$ preserves the inequality:
\begin{equation}
    \frac{2B}{\sigma_w^2} C_{ii} = \mathbb{E}_p [B_i^{(p)}] \le \kappa_{\max} \mathbb{E}_p [A_i^{(p)}] = \kappa_{\max} H_{ii}.
\end{equation}
This proves the upper bound inequality. Together, we get the chain of inequality:
\begin{equation}\label{ineqCH}
    \frac{\sigma_w^2}{2B} H_{ii}^2 \leq C_{ii} \leq \frac{\sigma_w^2 \kappa_{\max}}{2B} H_{ii}
\end{equation}

\subsubsection{Bounds on the Scaling Exponent $\gamma$}

We now investigate the implications of these inequalities.

Let the eigenvalues be indexed by $i \in [1, D]$ in descending order { within the fixed layer or parameter subspace under analysis}. {Conditional on the empirical per-layer or subspace-level power-law fit, $C_{ii} = c \cdot H_{ii}^{\gamma}$, where $c > 0$ is a constant,} we have 
\begin{equation}\label{logch}
    \log{C_{ii}} = \gamma \log{H_{ii}} + \log{c}.
\end{equation}

Since the logarithm is a monotonically increasing function, taking the logarithm of both sides for the two inequalities in Eq.(\ref{ineqCH}) preserves the inequality, leading to:
\begin{equation}\label{coi}
    \log{\alpha} + 2\log{H_{ii}} \leq \log{C_{ii}} \leq \log{\beta} + \log{H_{ii}},
\end{equation}
where $\alpha = \frac{\sigma_w^2}{2B} > 0$ and $\beta = \frac{\sigma_w^2 \kappa_{\max}}{2B} > 0$.

\textbf{Case 1: The Upper Limit ($\gamma \le 2$)}
Substituting Eq.(\ref{logch}) into the first inequality of Eq.(\ref{coi}) gives:
\begin{equation}
    (2 -\gamma)\log{H_{ii}} \leq \log{\alpha/c},
\end{equation}
which holds for arbitrary small $H_{ii} \ll 1$, requiring $\gamma \leq 2$.

\textbf{Case 2: The Lower Limit ($\gamma \ge 1$)} Similarly, the second inequality of Eq.(\ref{coi}) leads to:
\begin{equation}
    (\gamma -1)\log{H_{ii}} \leq \log{\beta/c},
\end{equation}
which holds for arbitrary small $H_{ii} \ll 1$, requiring $\gamma \geq 1$.

{Combining both cases, the exponent of any fixed per-layer or subspace-level fit satisfying the stated assumptions is bounded by $1 \le \gamma \le 2$.}
\end{proof}

\subsection{Proofs of Scaling Regimes}
\label{app:scaling_proofs}

We provide the derivations for the limiting cases of the scaling exponent $\gamma$, utilizing the notation established in Section \ref{sec:mechanisms}.

\subsubsection{Proof of Proposition \ref{prop:gamma2}}
\label{app:smr_proof}
We first define the phenomenological conditions required to observe $\gamma=2$. We examine the case where the geometric alignment is perfect but the eigenvalue magnitudes are stochastic.

Assume perfect alignment between local and global basis, i.e., $(\bm{u}_m^{(p)} \cdot \bm{v}_i)^2 = \delta_{mi}$, corresponding to the equality condition of the Cauchy-Schwarz inequality Eq. (\ref{CS}). Then the moments simplify to:
\begin{equation}
    H_{ii} = \mathbb{E}_p [\kappa_i^{(p)}], \quad C_{ii} \propto \mathbb{E}_p [(\kappa_i^{(p)})^2].
\end{equation}
Using the bias-variance decomposition $\mathbb{E}[X^2] = (\mathbb{E}[X])^2 + \text{Var}(X)$, we write:
\begin{equation}
    C_{ii} \propto H_{ii}^2 + \text{Var}_p(\kappa_i^{(p)}).
\end{equation}
If $\text{Var}_p(\kappa_i^{(p)}) \propto (\mathbb{E}_p [\kappa_i^{(p)}])^2 = H_{ii}^2$, then:
\begin{equation}
    C_{ii} \propto H_{ii}^2 \implies \gamma = 2.
\end{equation}

Assume the per-sample Hessian spectrum is locally dominated by the top $n$ eigenvalues with magnitude $\kappa_{m}^{(p)} \approx \bar{\kappa}^{(p)} \gg \kappa_\mathrm{rest}, (m = 1,2,...,n)$ and their associated eigenvectors are aligned (with tiny fluctuations) across different samples $p$. Specifically, if the curvature moment ratio is state-independent:

\subsubsection{Proof of Proposition \ref{prop:gamma1}}

\textbf{Assumptions:}
\begin{enumerate}
    \item \textbf{Spectral Dominance:} The per-sample Hessian spectrum is locally dominated by the top $n$ eigenvalues with magnitude $\kappa_{m}^{(p)} \approx \bar{\kappa}^{(p)} \gg \kappa_\mathrm{rest}$ for $m \leq n$.

    \item \textbf{Local Alignment:} The per-sample dominant eigen-directions $(\bm{u}_m^{(p)} \cdot \bm{v}_i)^2$ are aligned across different samples $p$, but with tiny fluctuation so that $(\bm{u}_m^{(p)} \cdot \bm{v}_i)^2 \to \delta_{mi}$ for $i \leq n$ while $(\bm{u}_m^{(p)} \cdot \bm{v}_i)^2$ remains finite for $i > n$. This tiny fluctuations contribute to the bulk eigenvalues of Global Hessian which is the source of heavy-tailed spectrum.

   \item \textbf{Statistical Homogeneous Magnitude:} $\kappa_m^{(p)}$ for $m \leq n$ are identically distributed (i.i.d.) variables drawn from a distribution characteristic of the loss basin. We define the first and second moments as characteristic constants:
    \begin{equation}
        \mathbb{E}_p[\kappa_m^{(p)}] = \mu_{\kappa}, \quad \mathbb{E}_p[(\kappa_m^{(p)})^2] = \mu_{\kappa^2} = \mu_{\kappa}^2 + \sigma_{\kappa}^2.
    \end{equation}
    Crucially, we assume the variance $\sigma_{\kappa}^2$ is an intrinsic property of the data complexity and is \textit{independent} of the projection direction $m$.
    
\end{enumerate}

\textbf{Derivation:}
We evaluate the global Hessian diagonal $H_{ii}$ and the noise covariance $C_{ii}$ by applying the independence property $\mathbb{E}[XY] = \mathbb{E}[X]\mathbb{E}[Y]$ under the Local Alignment assumption:

For the Hessian:
\begin{align}
    H_{ii} &= \mathbb{E}_p \left[ \sum_{m=1}^n \kappa_m^{(p)} (\bm{u}_m^{(p)} \cdot \bm{v}_i)^2 \right] \\
    &= \sum_{m=1}^n \mathbb{E}_p [\kappa_m^{(p)}] \cdot \mathbb{E}_p [(\bm{u}_m^{(p)} \cdot \bm{v}_i)^2] \\
    &= \mu_{\kappa} \underbrace{\sum_{m=1}^n \mathbb{E}_p [(\bm{u}_m^{(p)} \cdot \bm{v}_i)^2]}_{\mathcal{G}_i \text{ (Geometric Factor)}},
\end{align}
where the last identity results from the Statistical Homogeneous Magnitude assumption. For the Noise Covariance:
\begin{align}
    C_{ii} &\propto \mathbb{E}_p \left[ \sum_{m=1}^n (\kappa_m^{(p)})^2 (\bm{u}_m^{(p)} \cdot \bm{v}_i)^2 \right] \\
    &= \sum_{m=1}^n \mathbb{E}_p [(\kappa_m^{(p)})^2] \cdot \mathbb{E}_p [(\bm{u}_m^{(p)} \cdot \bm{v}_i)^2] \\
    &= \mu_{\kappa^2} \sum_{m=1}^n \mathbb{E}_p [(\bm{u}_m^{(p)} \cdot \bm{v}_i)^2] \\
    &= \mu_{\kappa^2} \mathcal{G}_i.
\end{align}

\textbf{Conclusion:}
Combining the two expressions, we find the relationship depends on the ratio of the moments:
\begin{equation}
    C_{ii} \propto \frac{\mu_{\kappa^2}}{\mu_{\kappa}} H_{ii} = \left( \mu_{\kappa} + \frac{\sigma_{\kappa}^2}{\mu_{\kappa}} \right) H_{ii}.
\end{equation}
Since $\mu_{\kappa}$ and $\sigma_{\kappa}^2$ are scalar constants characteristic of the dataset and do not depend on the eigen-index $i$, the pre-factor is a constant. Thus, the scaling relationship is approximately linear with $\gamma \to 1$.

\section{Multi-Layer Validation}
\label{app:multilayer_validation}


We further test the noise--curvature relation in a deeper MNIST MLP trained with CE loss to $100\%$ training accuracy. These experiments separate two questions that are conflated in a single-layer analysis. First, \Cref{fig:multilayer_per_layer} shows that, within each layer, the diagonal elements of $\mathbf C$ and $\mathbf H$ exhibit clean empirical power laws with well-defined, layer-dependent fitted exponents. Second, after all analyzed layers are concatenated into a joint parameter vector, the AWD second-moment prediction still tracks the covariance structure: \Cref{fig:multilayer_commutativity} shows approximate covariance--Hessian diagonalization in the joint Hessian eigenbasis, and \Cref{fig:multilayer_loglog} shows that the AWD-derived covariance closely follows the empirical covariance even though the joint $\mathbf C$--$\mathbf H$ log-log relation does not collapse to one straight line. This distinction is expected because different layers have different fitted exponents and the joint Hessian contains non-negligible inter-layer coupling, as quantified in \Cref{fig:multilayer_coupling}.

\newpage


\begin{figure}[h!]
    \centering
    \includegraphics[width=1.0\linewidth]{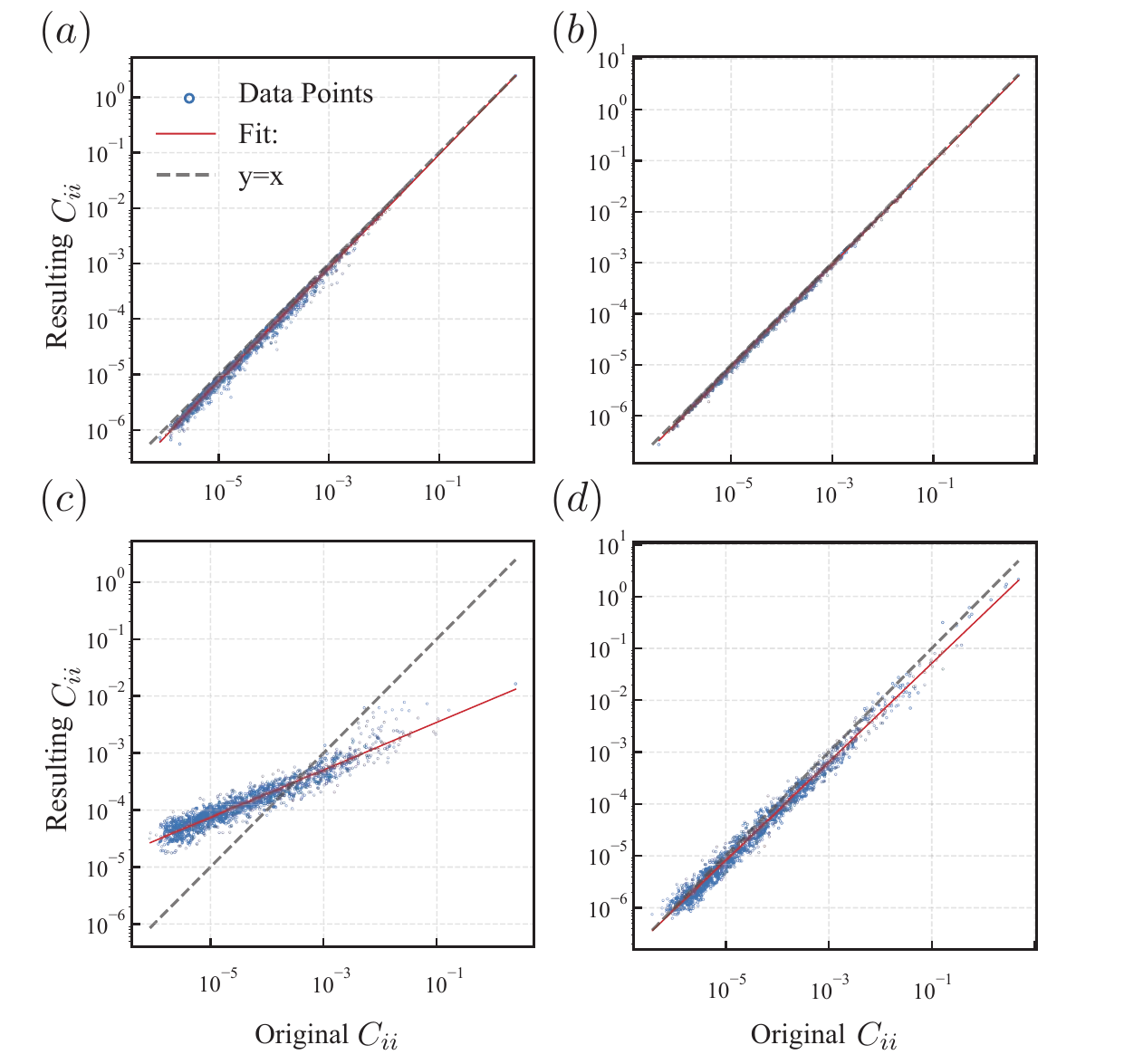}
    \caption{Log-log plots of the diagonal elements of the resulting Covariance by ``suppression experiment" versus the original Covariance.\textbf{(a, b)} Covariance derived from per-sample Hessians retaining only the dominant eigenvalues.
    \textbf{(c, d)} Covariance derived after replacing the dominant eigenvalues with their mean value.
    Columns correspond to distinct models: \textbf{(a, c)} MLP on CIFAR-10 (CE loss, $\gamma \approx 1.4$) and \textbf{(b, d)} MLP on MNIST (MSE loss, $\gamma \approx 1$).}
    \label{fig:suppres}
\end{figure}

\begin{figure}[t]
    \centering
    \includegraphics[width=1\linewidth]{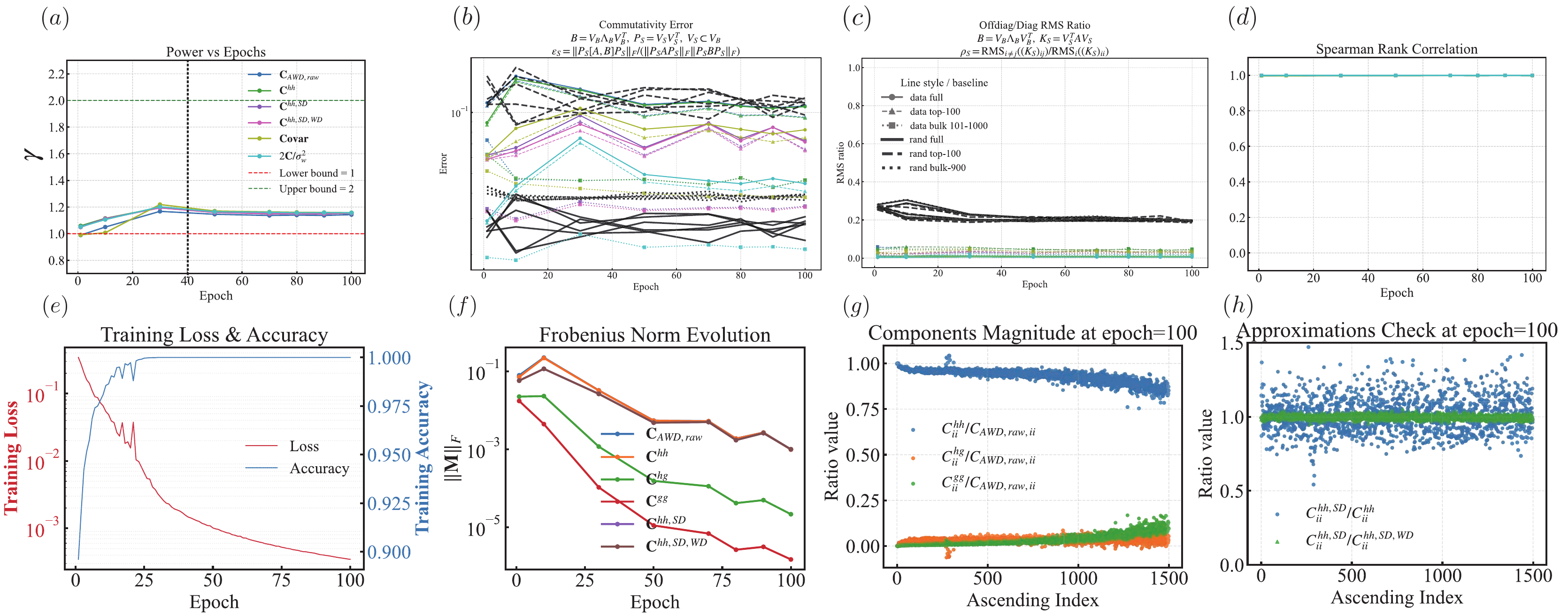}
    \caption{Comprehensive Analysis of SGD Noise Structure and Approximations (MLP on MNIST, CE Loss). 
\textbf{(a)} Evolution of the scaling exponent $\gamma$. The exponent $\gamma$ remains robustly within the interval $[1, 2]$ throughout training, it gradually increases as training progress and tends to deviate from the lower bound 1 when near the global minimum. Notably, in the terminal phase (near the global minimum, indicated by the vertical dashed line), the scaling exponent derived from the raw empirical covariance ($\mathbf{Covar}$) converges to match both the AWD-derived covariance ($\mathbf{C}_{AWD,raw}$) and its fully or partially approximations ($ \mathbf{C}^{hh} , \mathbf{C}^{hh,SD} , \mathbf{C}^{hh,SD,WD} , \mathbf{C}^{hh,SD,WD,LI}$, see \ref{app:awd_approx} for details), showing that the approximations used  in Theorem \ref{thm:spectral_noise} are valid. 
\textbf{(b)} {Normalized commutator error $\epsilon_F$ between each covariance estimate and $\mathbf H$, compared with the spectrum-preserving random baseline obtained by randomly rotating $\mathbf C$ while keeping its eigenvalue spectrum. This panel reports the textbook commutator diagnostic, which can be amplified by Hessian eigengaps in head-heavy spectra. The three reported windows are the full retained off-diagonal block, the top-100 head block $1\le i\ne j\le100$, and the bulk block $101\le i\ne j\le1000$.}
\textbf{(c)} {Entrywise off-diagonal-to-diagonal ratio $\rho_{\rm off/diag}$ for $\widetilde{\mathbf C}=V_H^\top\mathbf C V_H$. Smaller values indicate stronger approximate diagonalization in the Hessian eigenbasis; the random curve uses the same spectrum-preserving baseline as in panel~(b).}
\textbf{(d)} Spearman Rank Correlation between the diagonals of $\mathbf{C}$ and $\mathbf{H}$ in $\mathbf{H}$'s eigenbasis. A value of 1.0 indicates a strict monotonic correspondence between the noise and curvature spectra. 
\textbf{(e)} Training dynamics showing the loss and accuracy; the model converges to 100\% training accuracy around epoch 30. 
\textbf{(f)} Evolution of Frobenius norms validating the gradient noise approximation, see \ref{app:awd_approx} for the definition of these variables. The dominance of the Hessian-weight term ($\mathbf{C}^{hh}$) over the gradient-activity terms ($\mathbf{C}^{hg}, \mathbf{C}^{gg}$) confirms the ``Vanishing Gradients" assumption near the global minimum. The convergence of terms ($\mathbf{C}_{AWD,raw}, \mathbf{C}^{hh} , \mathbf{C}^{hh,SD} , \mathbf{C}^{hh,SD,WD}$) validates the \textit{Independence of Distinct Samples} and \textit{Local Isotropy} assumptions in Theorem \ref{thm:spectral_noise}. 
\textbf{(g)} Diagonals magnitude of ($\mathbf{C}^{hh}, \mathbf{C}^{hg}, \mathbf{C}^{gg}$) compared to $\mathbf{C}_{AWD,raw}$ at epoch 100 (near the global minimum) vs. descending basis index, providing a detailed view of the dominance of the Hessian term ($\mathbf{C}^{hh}$). 
\textbf{(h)}  Diagonal magnitudes of ($\mathbf{C}^{hh}, \mathbf{C}^{hh,SD,WD}$) compared with $\mathbf{C}^{hh,SD}$ at epoch 100 (late-training regime) vs. descending basis index. Here $\mathbf{C}^{hh,SD}$ is obtained from $\mathbf{C}^{hh}$ by dropping cross-sample terms, while $\mathbf{C}^{hh,SD,WD}$ is obtained from $\mathbf{C}^{hh,SD}$ by further replacing $\mathcal{M}_{p,mn}$ with $\sigma_{w,pm}^2\delta_{mn}$. Their close agreement indicates that both corrections are small in this regime.
}
\label{fig:metrics_mnist_fc_cse}
\end{figure}

\begin{figure}[t]
    \centering
    \includegraphics[width=1\linewidth]{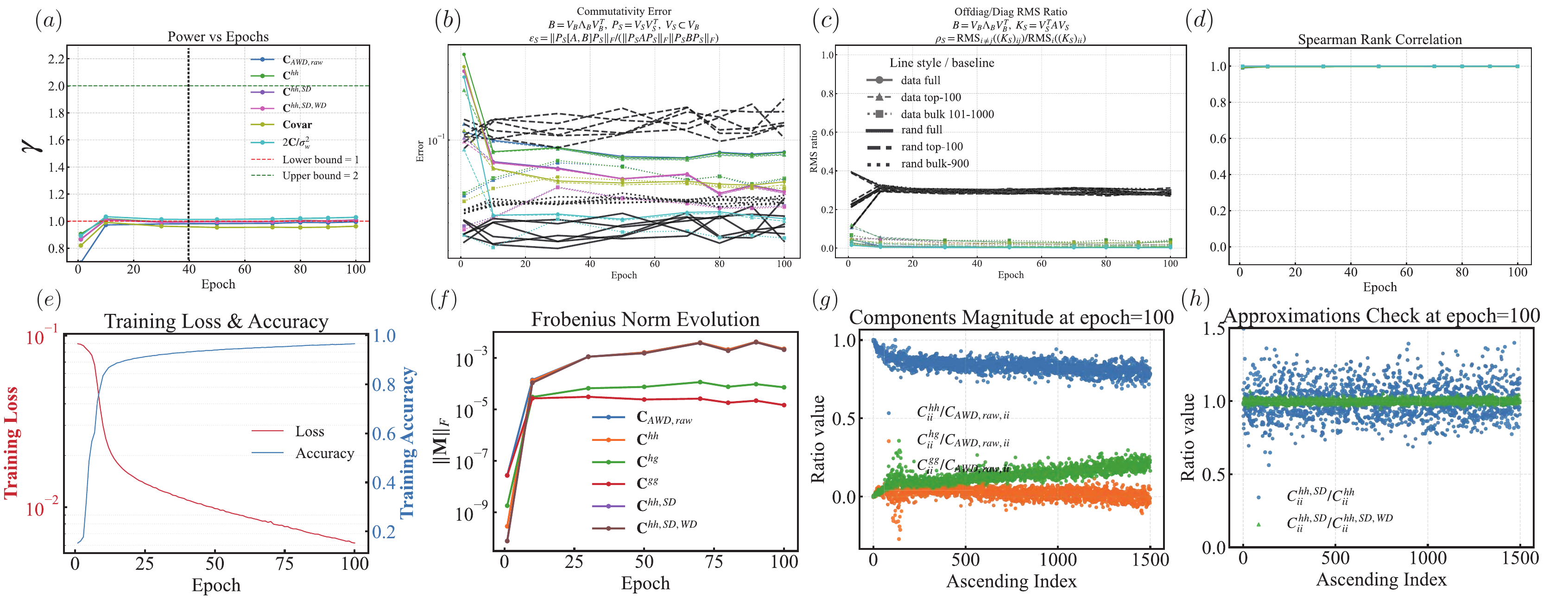}
    \caption{Comprehensive Analysis of SGD Noise Structure and Approximations (MLP on MNIST, MSE Loss)
    \textbf{(a)} Different from the case with CE loss, the exponent $\gamma$ remains 1 when the model is near the global minimum. }
    \label{fig:metrics_mnist_fc_mse}
\end{figure}

\begin{figure}[t]
    \centering
    \includegraphics[width=1\linewidth]{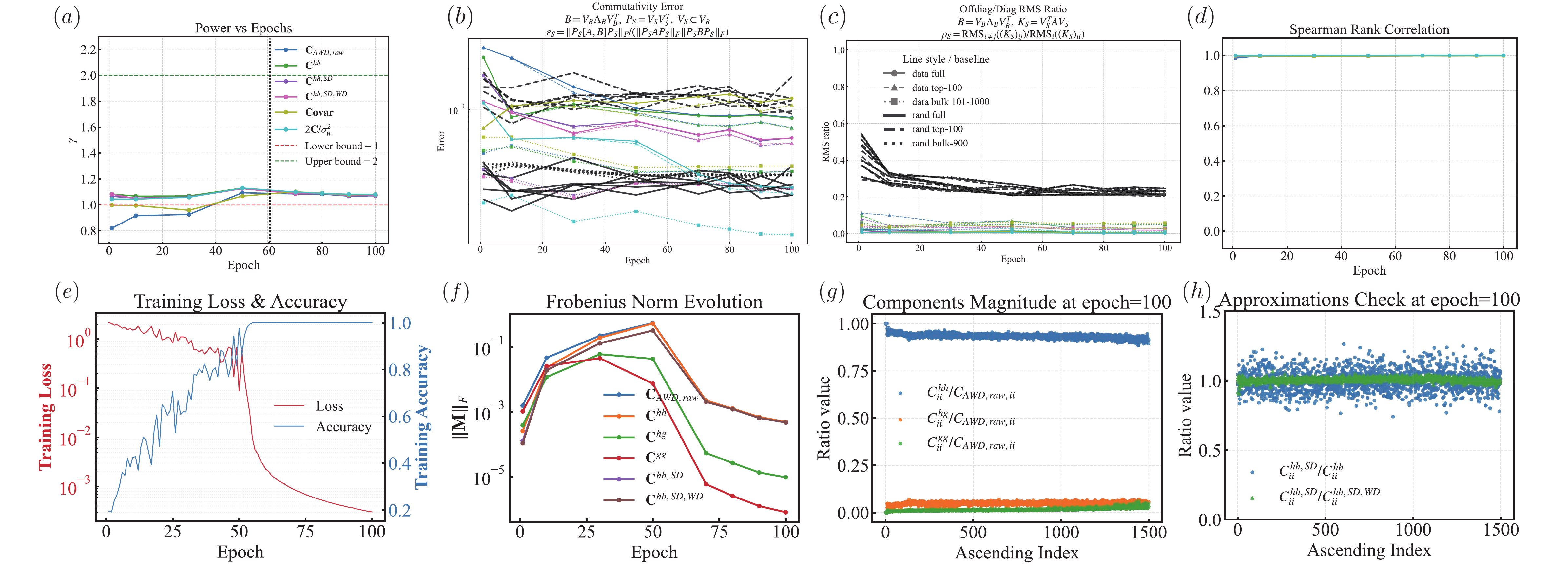}
    \caption{Comprehensive Analysis of SGD Noise Structure and Approximations (CNN on CIFAR-10, CE Loss)}
    \textbf{(a)} The exponent $\gamma$ gradually increases as the training progress and shows significant deviation from the lower bound 1 when the model is near the global minimum.
    \label{fig:metrics_cifar_cnn_cse}
\end{figure}

\begin{figure}[t]
    \centering
    \includegraphics[width=0.7\linewidth]{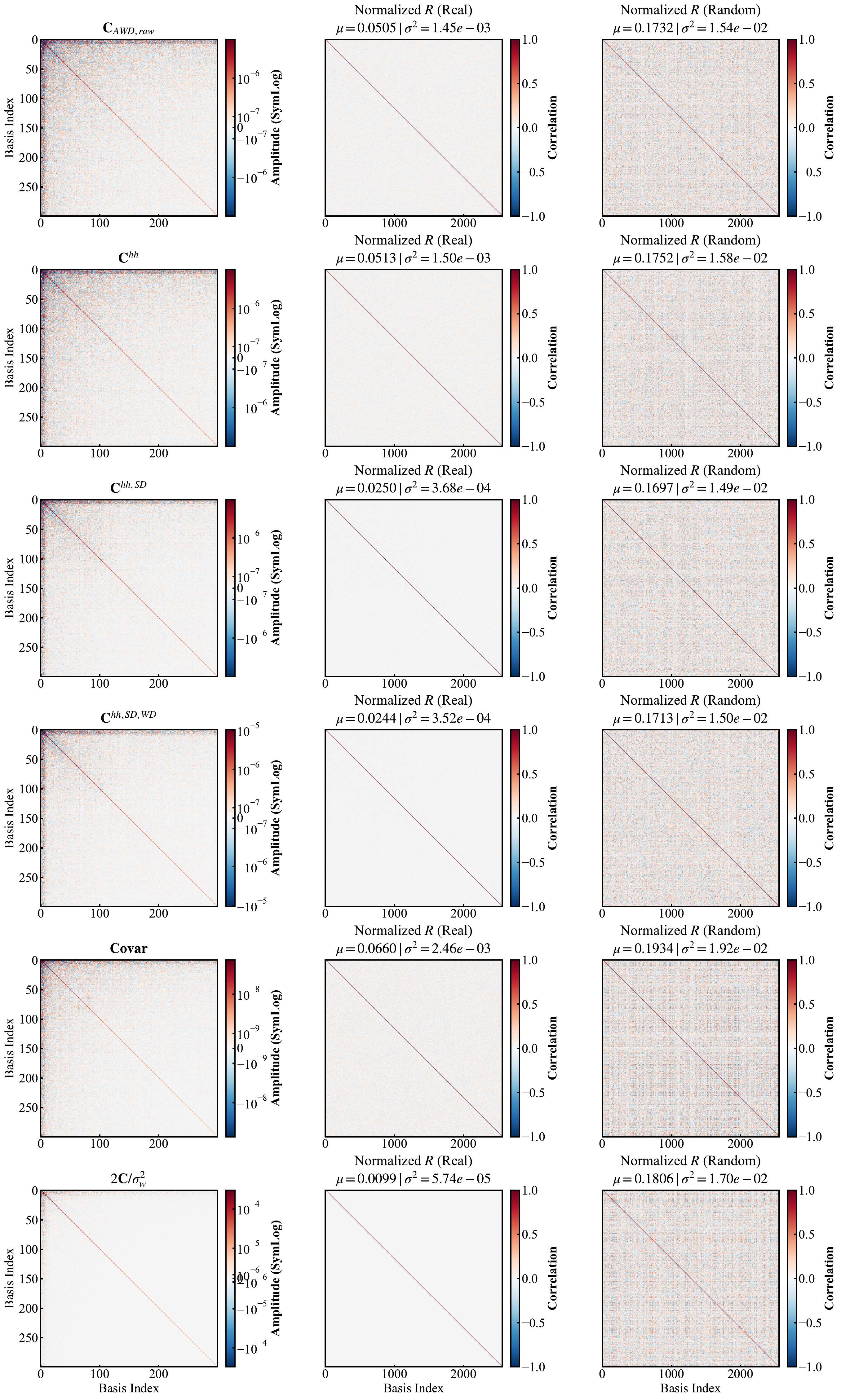}
    \caption{Visualization of Noise-Curvature Alignment (CNN on CIFAR-10 with CE loss, trained to convergence with $100\%$ training accuracy). The first column shows the Empirical Covariance Matrix $\mathbf{Covar}$, AWD-derived $\mathbf{C}_{AWD,raw}$, $\mathbf{C}^{hh,SD}$, $\mathbf{C}^{hh,SD,WD}$, $\mathbf{C}^{hh}$ and $2\mathbf{C}/\sigma_w^2$ in the Hessian eigenbasis which consistently show arrowhead structures. 
    The second column shows the corresponding normalized Correlation Matrix $\mathbf{R}$ for real data, which is strongly diagonal (red diagonal on bright background), indicating precise geometric alignment. The third column shows the randomized baseline $\mathbf{R}_{\text{rand}}$ (where $\mathbf{C}$ is randomly rotated) resembles ``static noise," showing that without alignment, energy leaks across all dimensions. $\mu$ and $\sigma^2$ represent the mean value and variance of the absolute values of off diagonal elements.}
    \label{fig:cmt_cifar_cnn_cse}
\end{figure}

\begin{figure}[t]
    \centering
    \includegraphics[width=0.7\linewidth]{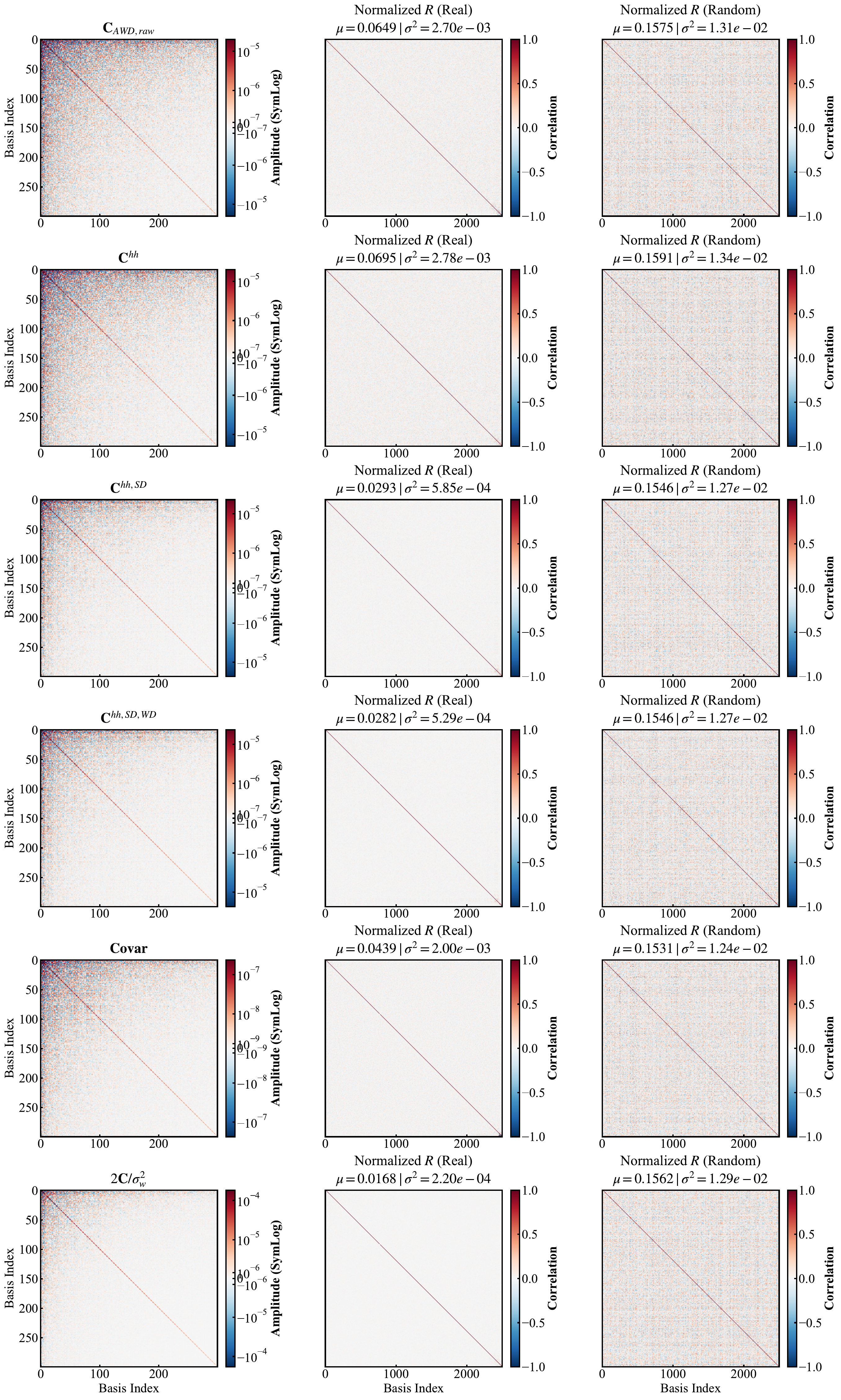}
    \caption{Visualization of Noise-Curvature Alignment (MLP on MNIST with CE loss, trained to convergence with $100\%$ training accuracy). }
    \label{fig:cmt_mnist_fc_cse}
\end{figure}

\begin{figure}[t]
    \centering
    \includegraphics[width=0.7\linewidth]{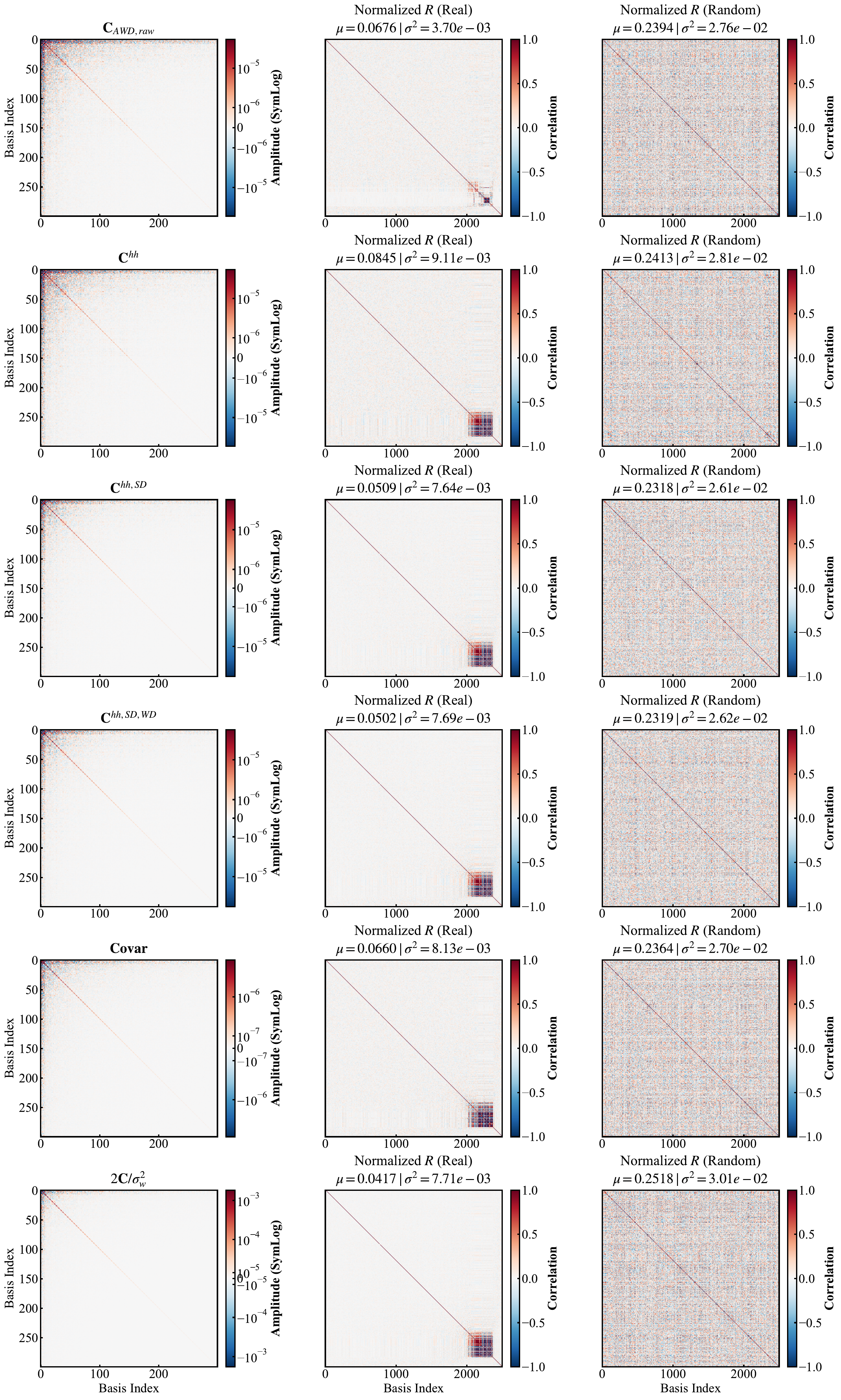}
    \caption{Visualization of Noise-Curvature Alignment (MLP on MNIST with MSE loss, trained to approach $\sim 95\%$ training accuracy). The dark block in the second column originates from the degenerated spectrum where the eigenvalues are too small for these flatter direction.}
    \label{fig:cmt_mnist_fc_mse}
\end{figure}

\begin{figure}[t]
    \centering
    \includegraphics[width=1\linewidth]{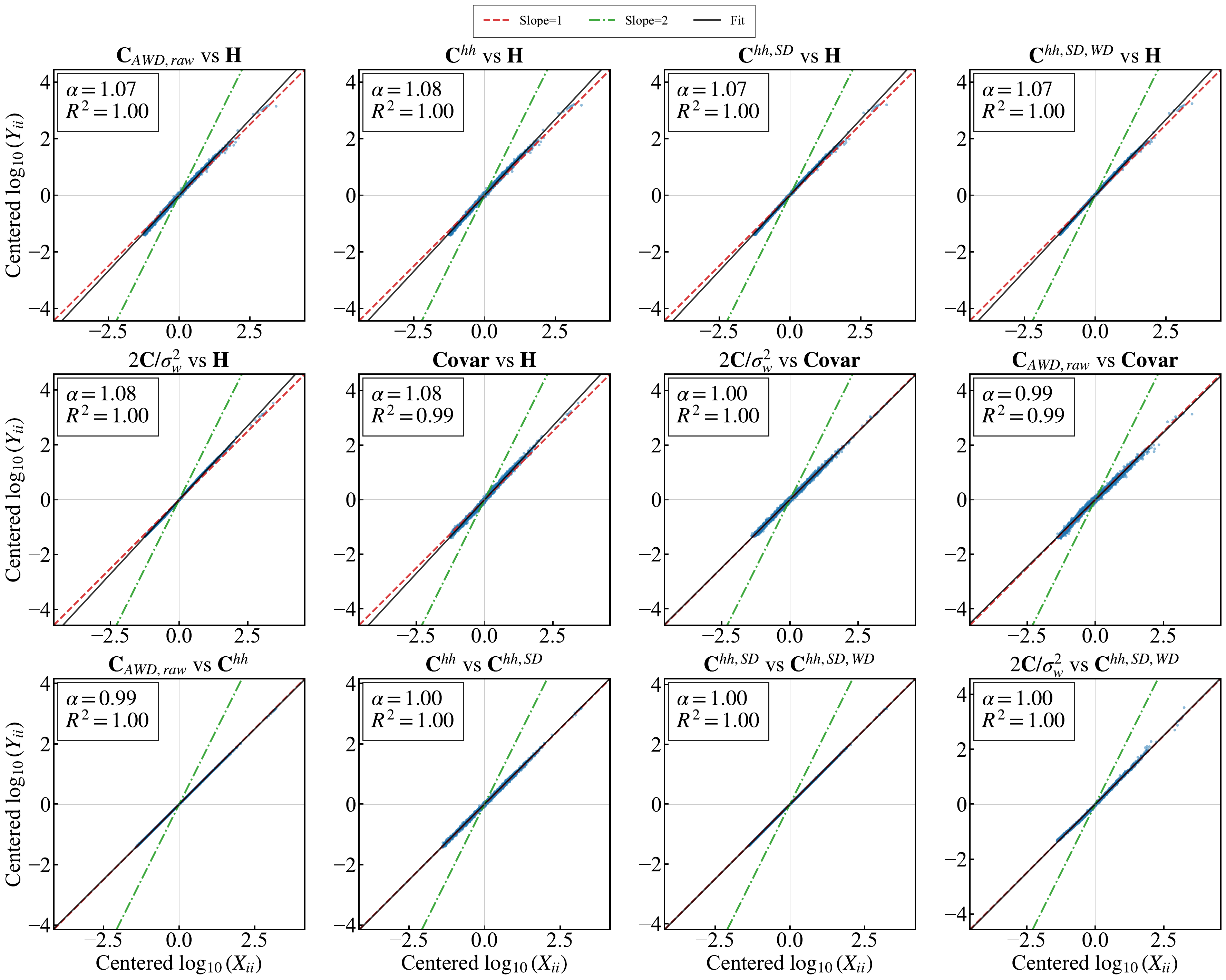}
    \caption{Log-log scaling analysis of diagonal elements (Top $\mathcal{N}=1500$ eigenvalues) for models trained to convergence $100\%$ training accuracy (CNN on CIFAR-10 with CE loss). To isolate the scaling exponent $\gamma$ (slope) from magnitude differences, data (dots) are mean-centered; solid lines indicate linear fits. Despite differences in magnitude, the centered log-log plots comparing the Empirical Covariance $\mathbf{Covar}$, Hessian $\mathbf{H}$, and AWD-derived Covariance $2\mathbf{C}/\sigma^2_w$ (including approximations $\mathbf{C}_{AWD,raw}, \mathbf{C}^{hh}, \mathbf{C}^{hh,SD}, \mathbf{C}^{hh,SD,WD}$) demonstrate the validity of our theory and the accuracy of the employed approximations.}
    \label{fig:lglg_cifar_cnn_cse}
\end{figure}

\begin{figure}[t]
    \centering
    \includegraphics[width=1\linewidth]{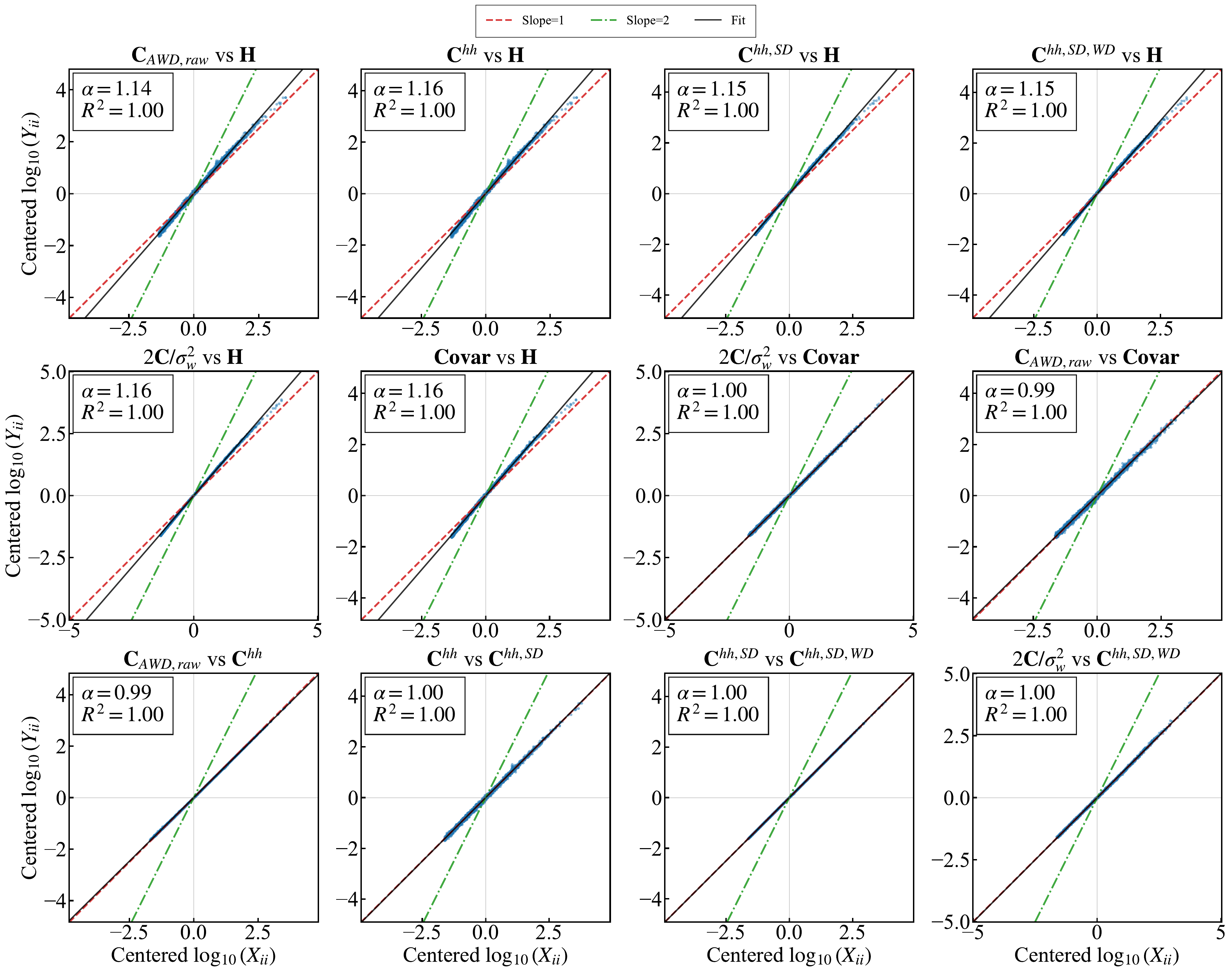}
    \caption{Log-log scaling analysis of diagonal elements (Top $\mathcal{N}=1500$ eigenvalues) for models trained to convergence $100\%$ training accuracy (MLP on MNIST with CE loss). }
    \label{fig:lglg_mnist_fc_cse}
\end{figure}

\begin{figure}[t]
    \centering
    \includegraphics[width=1\linewidth]{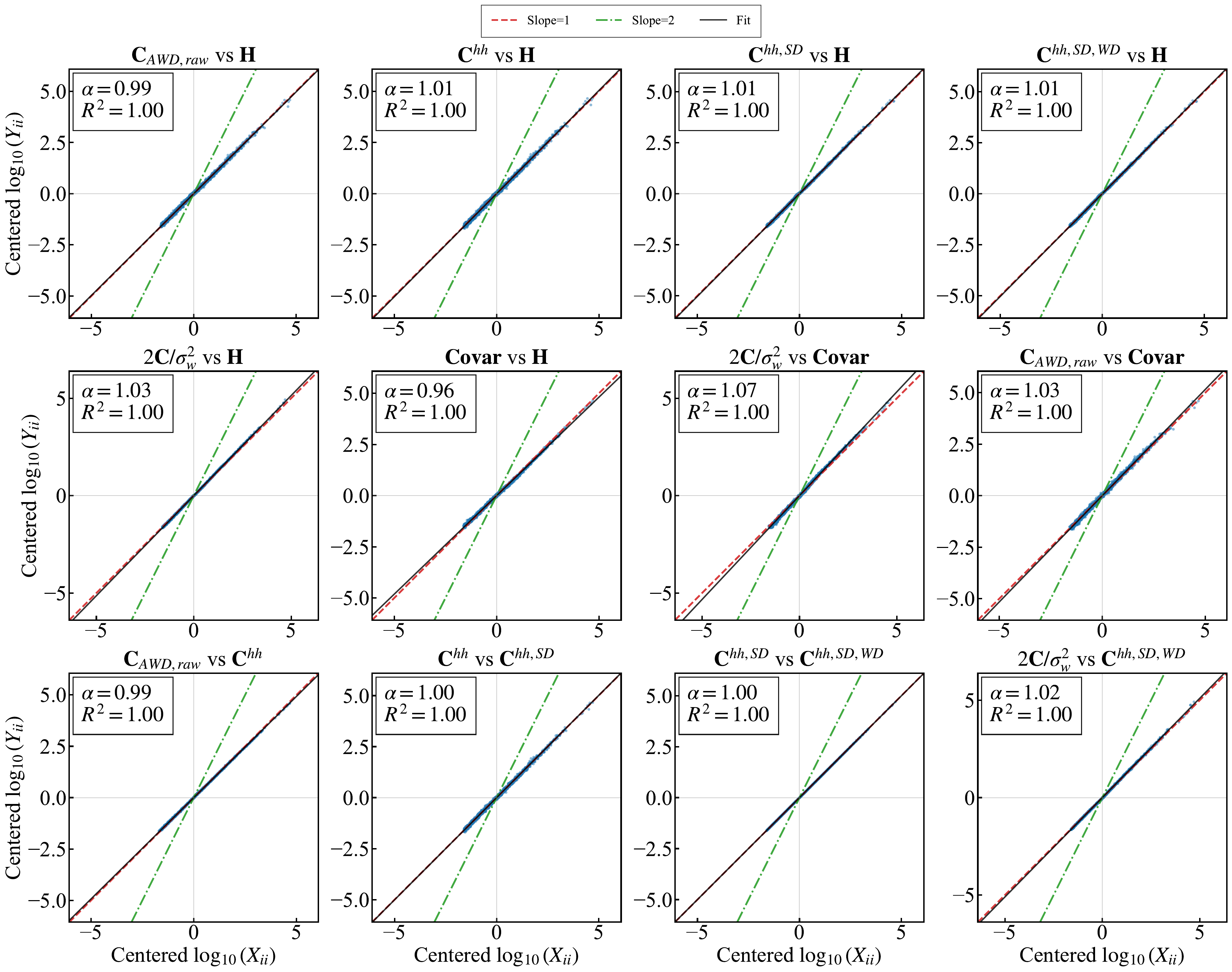}
    \caption{Log-log scaling analysis of diagonal elements (Top $\mathcal{N}=1500$ eigenvalues) for models trained to convergence $100\%$ training accuracy (MLP on MNIST with MSE loss). }
    \label{fig:lglg_mnist_fc_mse}
\end{figure}

\begin{figure}[t]
    \centering
    \includegraphics[width=1\linewidth]{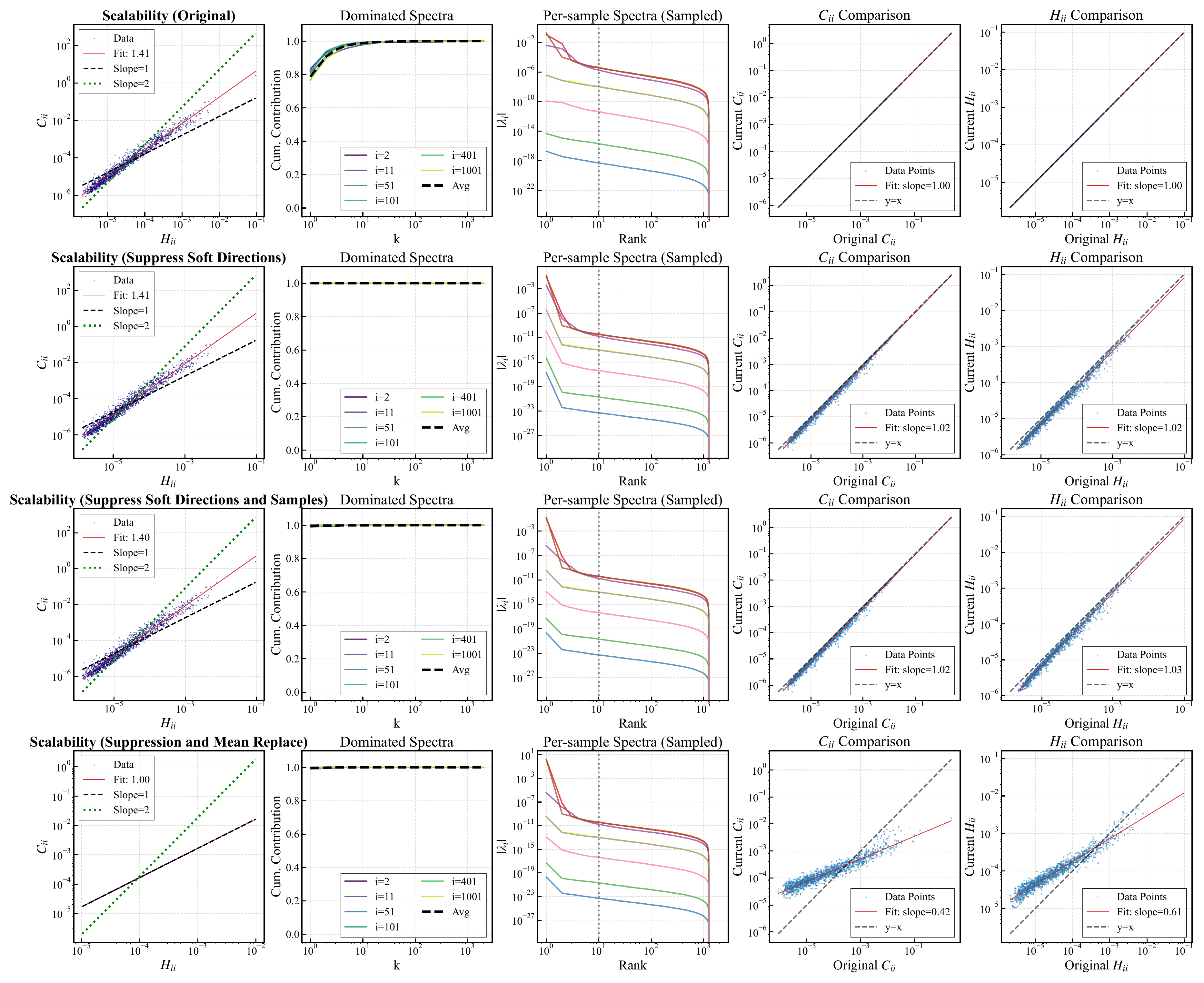}
    \caption{``Suppression experiment" for MLP on CIFAR-10 with CE loss ($\gamma \sim 1.4$ significantly larger than 1).
    The 1st column shows the log-log plot of diagonals between $2C_{ii}/\sigma_w^2 =  \mathbb{E}_{p} \left[\sum_{m} (\kappa_m^{(p)})^2 \left( \bm{u}_m^{(p)} \cdot \bm{v}_i \right)^2\right] $ and $H_{ii} = \mathbb{E}_{p} \left[\sum_{m} (\kappa_m^{(p)}) \left( \bm{u}_m^{(p)} \cdot \bm{v}_i \right)^2\right]$, while each row corresponds to different revised per-sample spectra. The 2nd column shows the cumulated contribution of per-sample eigenmode $\kappa^{(p)}_k$ defined by $\mathbb{E}_{p} \left[\sum^k_{m=1} (\kappa_m^{(p)})^2 \left( \bm{u}_m^{(p)} \cdot \bm{v}_i \right)^2\right]/\mathbb{E}_{p} \left[\sum_{m} (\kappa_m^{(p)})^2 \left( \bm{u}_m^{(p)} \cdot \bm{v}_i \right)^2\right]$ for different global hessian eigen-directions $i \in [1,N]$, where the bold black dashed line is the average among $i$. It is obvious that the first stiff mode contributes $> 80\%$. The 3rd column shows the per-sample hessian spectra before (the first row) and after (the rest rows) revision.
    The 4th and 5th columns show the log log plot of the diagonals between the raw and cooked $2C_{ii}/\sigma_w^2$ and $H_{ii}$ generated by original and revised per-sample hessian spectra separately. The 1st row shows the case where the per-sample hessian remains unchanged. The 2nd row shows the case keeping the first stiff eigenvalue $\kappa_1^{(p)}$ unchanged while suppressing the rest (scaling by $1e-5$). In this case, $2C_{ii}/\sigma_w^2$ and $H_{ii}$ hardly change. For the 3rd row, based on the manipulation in the 2nd row, we further suppress the contribution (scaling by $1e-3$) of non-stiff samples for those with $\kappa_1^{(p)}< \theta \ \text{max}_p \kappa_1^{(p)}$, here we set the threshold $\theta = 0.001$ that  keeps 36 out of 200 total samples. In this case, $2C_{ii}/\sigma_w^2$ and $H_{ii}$ hardly change. For the last row, we replace the remaining $\kappa_1^{(p)}$ to their mean value. $2C_{ii}/\sigma_w^2$ and $H_{ii}$ change significantly, yielding $\gamma=1$.}
    \label{fig:abl_cifar_mlp_cse}
\end{figure}

\begin{figure}[t]
    \centering
    \includegraphics[width=1\linewidth]{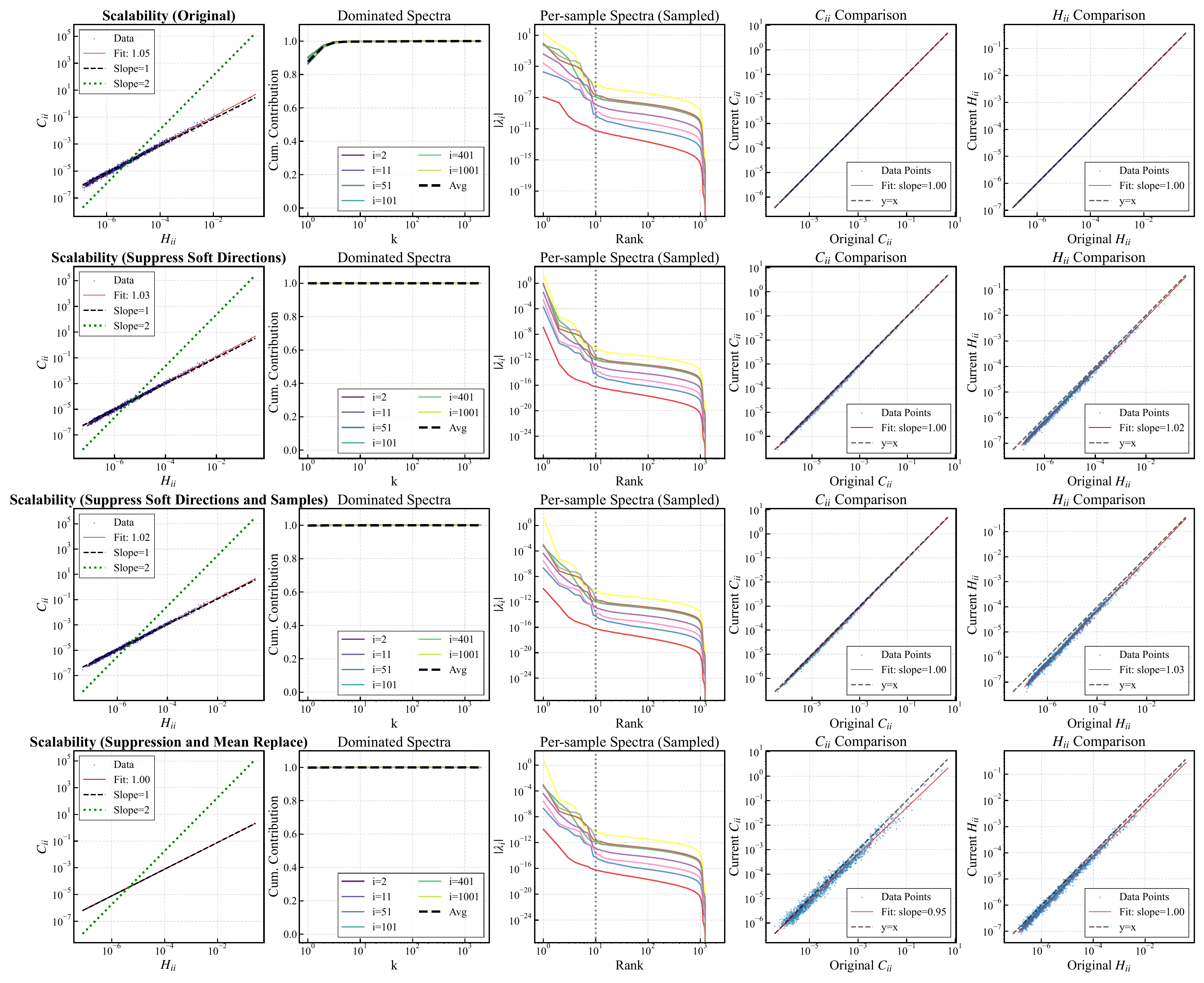}
    \caption{``Suppression experiment" for MLP on MNIST with MSE loss ($\gamma \sim 1.05$ near to 1).  Similar to Figure \ref{fig:abl_mnist_fc_cse}, here the threshold $\theta = 0.05$ and keeps 33 stiff samples out of 200. Compared to the raw data, cooked $2C_{ii}/\sigma_w^2$ and $H_{ii}$ slightly change but keep the linear scaling.}
    \label{fig:abl_mnist_fc_mse}
\end{figure}

\begin{figure}[t]
    \centering
    \includegraphics[width=1\linewidth]{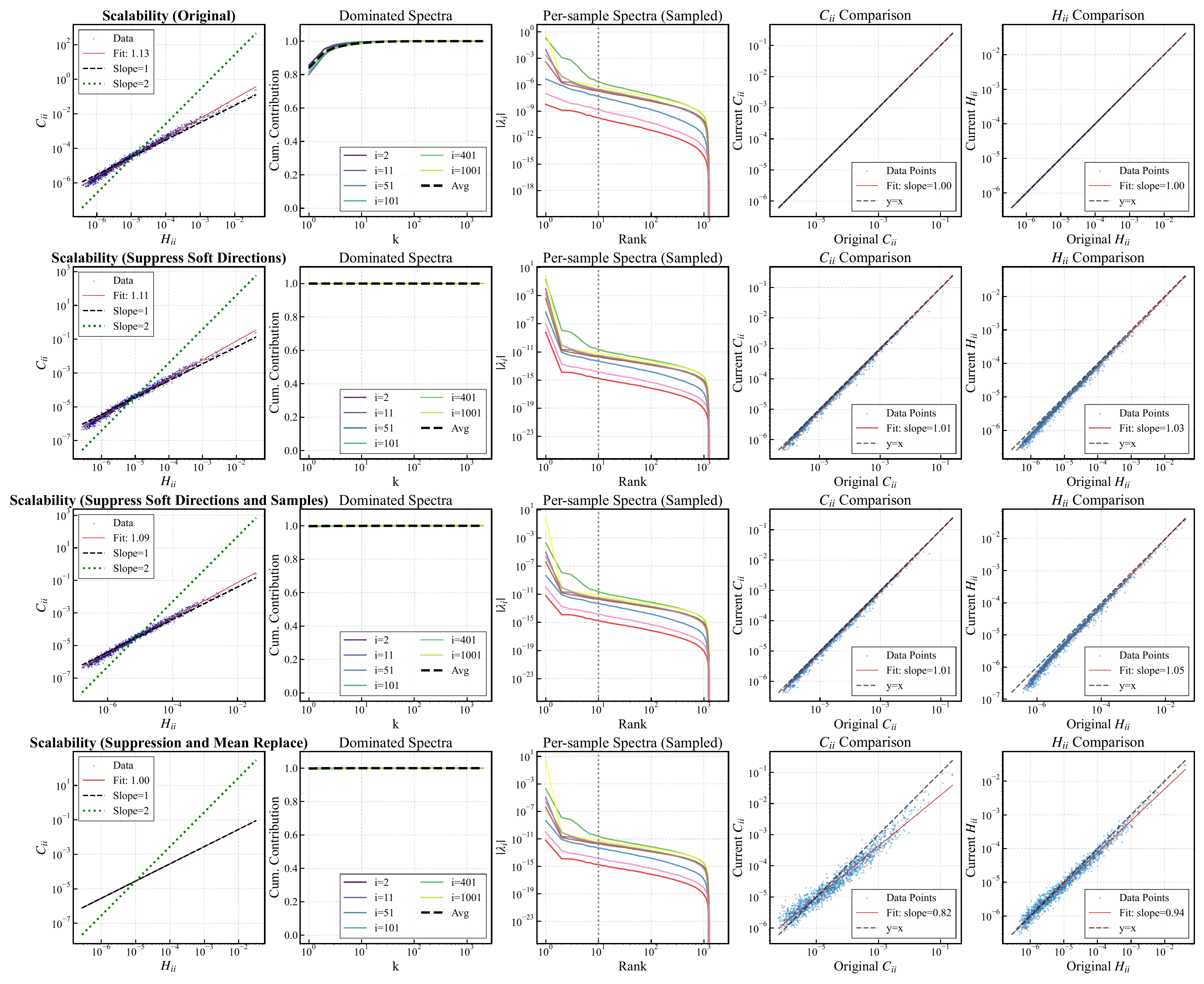}
    \caption{``Suppression experiment" for MLP on MNIST with CE loss ($\gamma \sim 1.1$ slightly larger than 1). The threshold $\theta = 0.05$ and keeps 20 stiff samples out of 200. Compared to the raw data, cooked $2C_{ii}/\sigma_w^2$ and $H_{ii}$ changes to some extend.}
    \label{fig:abl_mnist_fc_cse}
\end{figure}

\begin{figure}[t]
    \centering
    \includegraphics[width=1\linewidth]{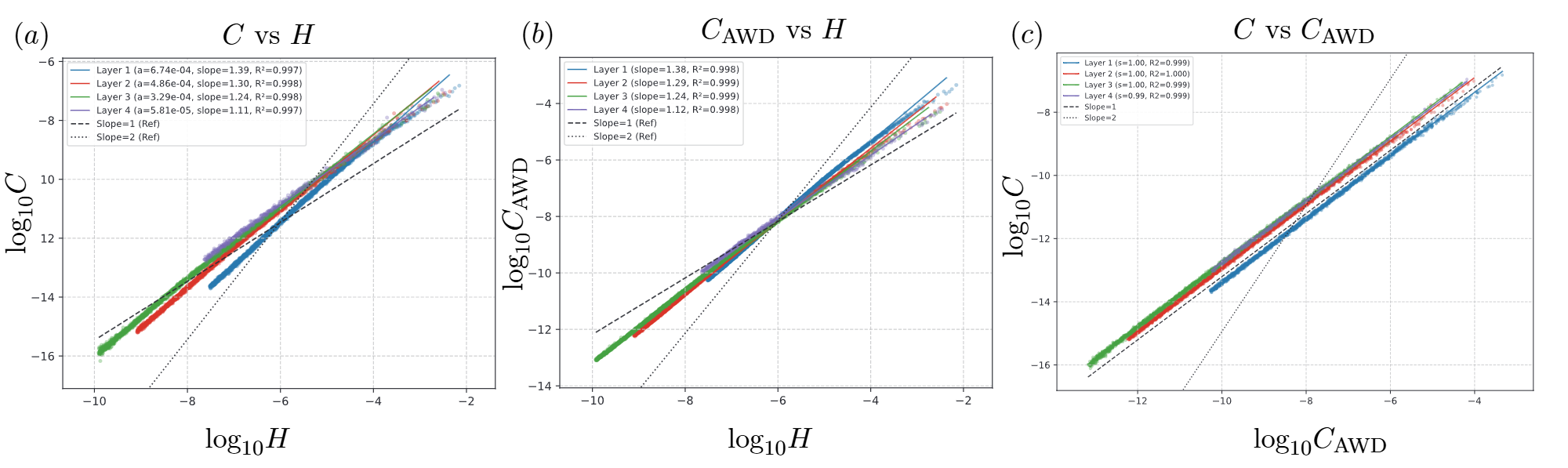}
    \caption{Log-log plots (in Hessian eigenbasis) of the diagonal elements for the multi-layer MNIST MLP. (a) directly numerically computed covariance $C$ against the Hessian $H$. Each layer exhibits a clean power-law relationship $C \propto H^\gamma$ with $R^2 > 0.99$. The magnitude ranges of the diagonal elements of $C$ and $H$ overlap strongly across layers. The fitted exponent $\gamma$ decreases from $\approx 1.39$ (Layer 1, blue) to $\approx 1.11$ (Layer 4, purple), but remains within $(1,2)$ across all layers, consistent with our theoretical prediction. (b) the AWD-based $C_{\text{AWD},ij}=E_p \left[\sum_m (\kappa_m^{(p)})^2 (u_m^{(p)} \cdot v_i)(u_m^{(p)} \cdot v_j)\right]$ (the 2nd moment of persample hessian) against the Hessian H which reproduces the empirical slopes faithfully. (c) directly numerically computed covariance $C$ against the AWD-based $C_{\text{AWD}}$, showing almost straight lines with slope = 1.}
    \label{fig:multilayer_per_layer}
\end{figure}

\begin{figure}[t]
    \centering
    \includegraphics[width=1\linewidth]{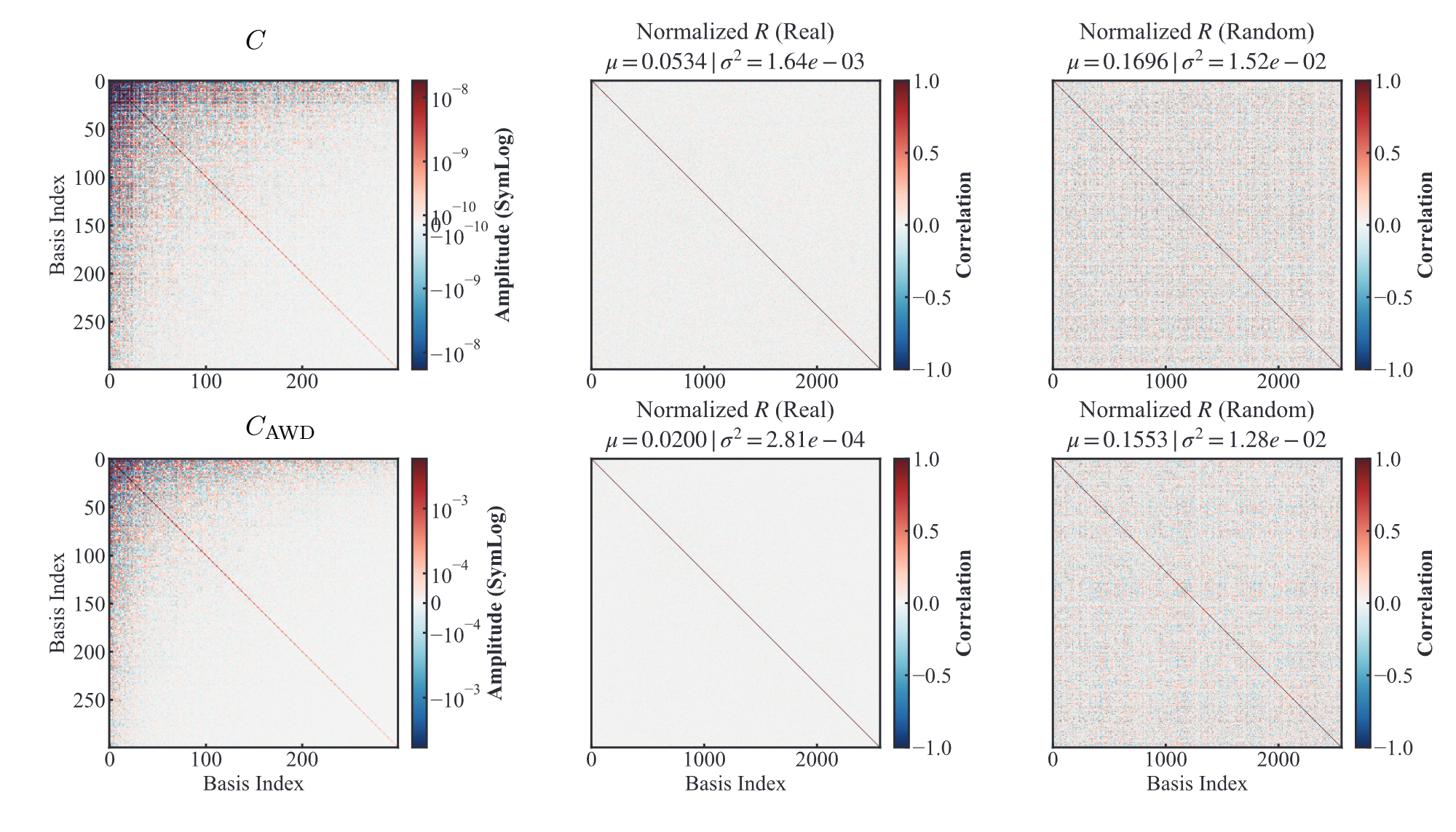}
    \caption{Approximate commutativity in the joint multi-layer parameter space. The empirical covariance and the AWD-derived covariance are shown in the global Hessian eigenbasis. Their scale-invariant correlation matrices have substantially smaller off-diagonal mean and variance than randomized baselines with the same spectra, indicating that the approximate $\mathbf C$--$\mathbf H$ commutativity persists beyond a single layer.}
    \label{fig:multilayer_commutativity}
\end{figure}

\begin{figure}[t]
    \centering
    \includegraphics[width=1\linewidth]{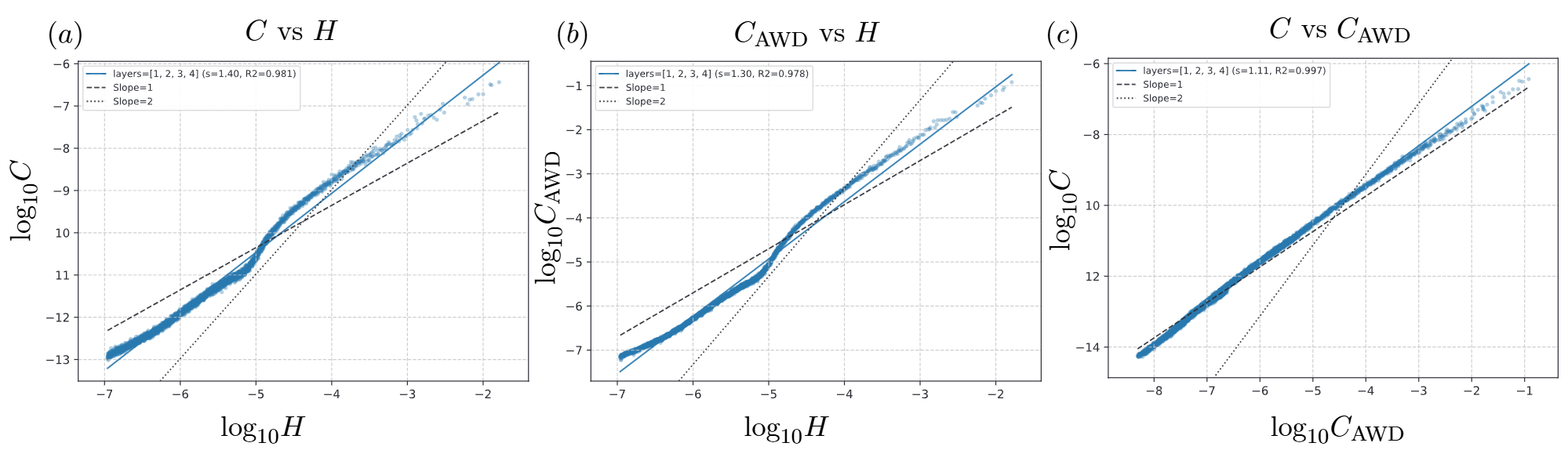}
    \caption{Joint multi-layer log-log analysis. Unlike the per-layer plots in \Cref{fig:multilayer_per_layer}, the joint $C_{ii}$--$H_{ii}$ relation does not form a single clean line in log-log coordinates. Nevertheless, the AWD-derived covariance closely follows the empirical covariance, and empirical $C_{ii}$ is approximately linear in $C_{\mathrm{AWD},ii}$, confirming that the second-moment description remains accurate even when a single joint exponent is not meaningful.}
    \label{fig:multilayer_loglog}
\end{figure}

\begin{figure}[t]
    \centering
    \includegraphics[width=1\linewidth]{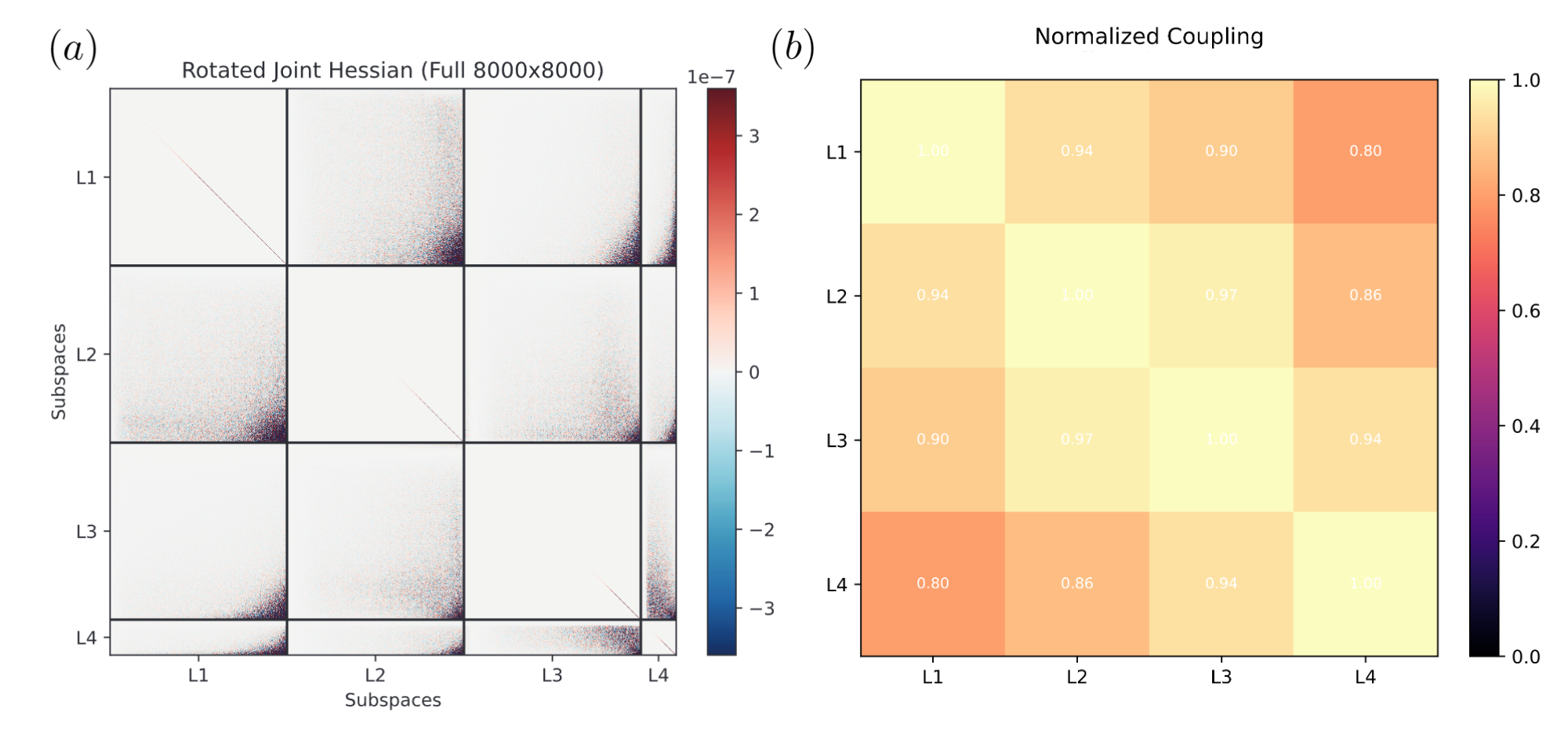}
    \caption{Layerwise decomposition of the joint Hessian $\mathbf{H}\in\mathbb{R}^{D\times D}$ ($D=\sum_{l}d_{l}$) for a 3-hidden-layer MLP with the output layer. Let $\hat{H}^{(l)}$ be the single-layer Hessian of layer $l$ with eigenbasis $V^{(l)}$, and define the block-diagonal rotation matrix $Q=\mathrm{blkdiag} \bigl(V^{(1)},V^{(2)},V^{(3)},V^{(4)}\bigr)$. (a) The rotated Hessian $H_{\text{rot}} = Q^{\top} H Q$. While each diagonal block is approximately diagonalised, the off-diagonal blocks remain substantial, indicating non-negligible inter-layer coupling. (b) the effect normalized coupling $|| H_{\text{rot},L_i,L_j}||\_F / (||H_{\text{rot},L_i,L_i}||\_F ||H_{\text{rot},L_j,L_j}||\_F)^{1/2}$ defined by normalized Frobenius Norm of each blocks $H_{\text{rot},L_i,L_j}$.}
    \label{fig:multilayer_coupling}
\end{figure}

\clearpage
\section*{NeurIPS Paper Checklist}

\begin{enumerate}

\item {\bf Claims}
    \item[] Question: Do the main claims made in the abstract and introduction accurately reflect the paper's contributions and scope?
    \item[] Answer: \answerYes{} 
    \item[] Justification: The abstract and introduction state the paper's scope: characterizing the SGD noise covariance--Hessian relation through AWD.

\item {\bf Limitations}
    \item[] Question: Does the paper discuss the limitations of the work performed by the authors?
    \item[] Answer: \answerYes{} 
    \item[] Justification: The paper states the local-perturbation and statistical covariance assumptions, distinguishes single-layer fits from full-space behavior, and identifies open mechanisms behind the CE--MSE distinction as future work. A key practical limitation is that computing Hessians, especially the second moment of per-sample Hessians, is computationally expensive; consequently, our experiments are restricted to relatively simple models such as MLPs and CNNs on MNIST and CIFAR-10, and the second-moment estimates are computed using sampled subsets of the data.

\item {\bf Theory assumptions and proofs}
    \item[] Question: For each theoretical result, does the paper provide the full set of assumptions and a complete (and correct) proof?
    \item[] Answer: \answerYes{} 
    \item[] Justification: The main text states the assumptions used in the AWD gradient approximation, covariance decomposition, and exponent bounds. Detailed derivations and proofs are provided in Appendix~\ref{app:derivations} and the subsequent proof sections.

    \item {\bf Experimental result reproducibility}
    \item[] Question: Does the paper fully disclose all the information needed to reproduce the main experimental results of the paper to the extent that it affects the main claims and/or conclusions of the paper (regardless of whether the code and data are provided or not)?
    \item[] Answer: \answerYes{} 
    \item[] Justification: Appendix~\ref{sec:expset} reports the datasets, architectures, analyzed parameter subspaces, optimizer, batch sizes, learning rates, epochs, random seeds, and eigenvalue-selection protocol needed to reproduce the main empirical results.

\item {\bf Open access to data and code}
    \item[] Question: Does the paper provide open access to the data and code, with sufficient instructions to faithfully reproduce the main experimental results, as described in supplemental material?
    \item[] Answer: \answerYes{} 
    \item[] Justification: We provide an anonymized code/supplemental package with instructions for reproducing the main experiments. The experiments use public benchmark datasets, MNIST and CIFAR-10, cited in the paper.

\item {\bf Experimental setting/details}
    \item[] Question: Does the paper specify all the training and test details (e.g., data splits, hyperparameters, how they were chosen, type of optimizer) necessary to understand the results?
    \item[] Answer: \answerYes{} 
    \item[] Justification: Appendix~\ref{sec:expset} gives the architecture, dataset, optimizer, learning-rate, batch-size, epoch, and analyzed-layer details. Table~\ref{tab:hyperparams} further lists hyperparameters for the experiments summarized in Table~\ref{tab:gamma_comparison}.

\item {\bf Experiment statistical significance}
    \item[] Question: Does the paper report error bars suitably and correctly defined or other appropriate information about the statistical significance of the experiments?
    \item[] Answer: \answerYes{} 
    \item[] Justification: Table~\ref{tab:gamma_comparison} reports mean $\pm$ standard deviation over four independent runs. Additional diagnostic figures support the conclusions across datasets, architectures, losses, and training settings.

\item {\bf Experiments compute resources}
    \item[] Question: For each experiment, does the paper provide sufficient information on the computer resources (type of compute workers, memory, time of execution) needed to reproduce the experiments?
    \item[] Answer: \answerNo{} 
    \item[] Justification: The experiments are small-scale diagnostic experiments on standard benchmark datasets and do not require specialized large-scale compute. We do not provide a separate compute-resource accounting.
    
\item {\bf Code of ethics}
    \item[] Question: Does the research conducted in the paper conform, in every respect, with the NeurIPS Code of Ethics \url{https://neurips.cc/public/EthicsGuidelines}?
    \item[] Answer: \answerYes{} 
    \item[] Justification: The research uses public benchmark datasets and studies theoretical and empirical properties of SGD noise. It does not involve human subjects, private data, surveillance, or high-risk deployment.

\item {\bf Broader impacts}
    \item[] Question: Does the paper discuss both potential positive societal impacts and negative societal impacts of the work performed?
    \item[] Answer: \answerNA{} 
    \item[] Justification: This is foundational research on optimization and generalization in neural networks, with no direct deployed system, application-specific model, or targeted societal use case. We do not identify a direct path to specific societal harm beyond the general downstream uses of machine-learning research.
    
\item {\bf Safeguards}
    \item[] Question: Does the paper describe safeguards that have been put in place for responsible release of data or models that have a high risk for misuse (e.g., pre-trained language models, image generators, or scraped datasets)?
    \item[] Answer: \answerNA{} 
    \item[] Justification: The paper does not release a high-risk pretrained model, scraped dataset, generative model, or other asset requiring misuse safeguards.

\item {\bf Licenses for existing assets}
    \item[] Question: Are the creators or original owners of assets (e.g., code, data, models), used in the paper, properly credited and are the license and terms of use explicitly mentioned and properly respected?
    \item[] Answer: \answerYes{} 
    \item[] Justification: The paper cites the original MNIST and CIFAR-10 dataset sources and uses them as public benchmark datasets. We do not redistribute restricted third-party assets.

\item {\bf New assets}
    \item[] Question: Are new assets introduced in the paper well documented and is the documentation provided alongside the assets?
    \item[] Answer: \answerYes{} 
    \item[] Justification: The anonymized code/supplemental materials are provided for reproducibility and documented with the experimental settings described in Appendix~\ref{sec:expset}. The paper does not introduce a new dataset or benchmark model.

\item {\bf Crowdsourcing and research with human subjects}
    \item[] Question: For crowdsourcing experiments and research with human subjects, does the paper include the full text of instructions given to participants and screenshots, if applicable, as well as details about compensation (if any)? 
    \item[] Answer: \answerNA{} 
    \item[] Justification: The paper does not involve crowdsourcing experiments or research with human subjects.

\item {\bf Institutional review board (IRB) approvals or equivalent for research with human subjects}
    \item[] Question: Does the paper describe potential risks incurred by study participants, whether such risks were disclosed to the subjects, and whether Institutional Review Board (IRB) approvals (or an equivalent approval/review based on the requirements of your country or institution) were obtained?
    \item[] Answer: \answerNA{} 
    \item[] Justification: The paper does not involve crowdsourcing or human-subject research, so IRB approval or equivalent review is not applicable.

\item {\bf Declaration of LLM usage}
    \item[] Question: Does the paper describe the usage of LLMs if it is an important, original, or non-standard component of the core methods in this research? Note that if the LLM is used only for writing, editing, or formatting purposes and does \emph{not} impact the core methodology, scientific rigor, or originality of the research, declaration is not required.
    \item[] Answer: \answerNA{} 
    \item[] Justification: LLMs were used only for writing, editing, or formatting assistance. They were not used as a core methodological, theoretical, or experimental component of the research.

\end{enumerate}

\end{document}